\documentclass[final,12pt]{colt2026} 

\title[On the Power of Adaptivity for $\varepsilon$-Best Arm Identification in Linear Bandits]{On the Power  of Adaptivity for $\varepsilon$-Best Arm Identification in Linear Bandits}\usepackage{times}

\coltauthor{%
 \Name{Arnab Maiti} \Email{arnabm2@uw.edu}\\
 \addr University of Washington
 \AND
 \Name{Yunbei Xu} \Email{yunbei@nus.edu.sg}\\
 \addr National University of Singapore%
 \AND
 \Name{Kevin Jamieson} \Email{jamieson@cs.washington.edu}\\
 \addr University of Washington
}

\usepackage{url}
\usepackage[margin=1in]{geometry}
\usepackage{amsfonts,amsopn,bm,color,mathtools,booktabs,multirow,appendix,caption}
\PassOptionsToPackage{ruled,vlined,linesnumbered}{algorithm2e}
\usepackage{natbib}
\usepackage{bbm}
\usepackage{enumitem}
\setlist[itemize]{noitemsep, topsep=0pt, leftmargin=11pt}

\def\E{\mathbb{E}}
\def\P{\mathbb{P}}
\def\R{\mathbb{R}}
\def\1{\bm{1}}

\newcommand{\mc}[1]{\mathcal{#1}}

\newcommand{\ip}[2]{\left\langle #1,\,#2 \right\rangle}

\allowdisplaybreaks
\begin{document}

\maketitle

\begin{abstract}

We study the minimax sample complexity of $\varepsilon$-best arm identification in linear bandits, a classical pure-exploration problem. Given a compact action set $\mathcal{X}$ that spans $\mathbb{R}^d$ and an unknown reward vector $\theta\in\mathbb{R}^d$, the goal is to output an arm $\widehat{x}\in\mathcal{X}$ such that $\langle \widehat{x},\theta\rangle \ge \max_{x\in\mathcal{X}} \langle x,\theta\rangle - \varepsilon$ with probability at least $1-\delta$, using as few samples as possible. Our aim is to better understand the power and limitations of adaptivity in this setting.

We begin with non-adaptive algorithms. We present a non-adaptive fixed-design method with sample complexity $\mathcal{O}\!\left(\frac{d\log(1/\delta)}{\varepsilon^2}+\frac{w(\mathcal{X})^2}{\varepsilon^2}\right)$, where $w(\mathcal{X})$ is a Gaussian width term dependent on $\mathcal{X}$, and we prove a matching lower bound $\Omega\!\left(\frac{d\log(1/\delta)}{\varepsilon^2}+\frac{w(\mathcal{X})^2}{\varepsilon^2}\right)$ for all non-adaptive fixed-design methods. Moreover, $w(\mathcal{X})\le \mathcal{O}(d)$ for general $\mathcal{X}$, which is tight for sets such as the unit $\ell_2$ ball, and $w(\mathcal{X})\le \mathcal{O}(\sqrt{d\log|\mathcal{X}|})$ when $\mathcal{X}$ is finite, which is tight for the canonical basis $\{e_1,\ldots,e_d\}$.

We then turn to adaptive sampling. For any finite action set $\mathcal{X}$, we prove the existence of an adaptive algorithm with sample complexity $\mathcal{O}\!\left(\frac{d\log(1/\delta)}{\varepsilon^2}+\frac{d\log(|\mathcal{X}|/d)}{\varepsilon^2}\right)$ via a generalization of Median Elimination, which is known to yield a $\log d$ improvement for the canonical basis. This raises a structural question: beyond the canonical basis, are there structured action sets for which adaptivity yields only logarithmic-factor improvements over the optimal non-adaptive rate? We answer in the affirmative for several natural action sets, namely the hypercube, the $\ell_2$ ball, $m$-sets, and multi-task multi-armed bandits.

Finally, we show that logarithmic improvements are not the whole story. To our knowledge, we provide the first construction of an action set $\mathcal{X}$ for which adaptivity yields a \emph{polynomial-factor improvement} over every non-adaptive algorithm. A key ingredient behind this separation is an $\ell_2$-norm estimation subroutine: we design an adaptive algorithm that uses $\mathcal{O}\!\left(\frac{d\log(1/\delta)}{\varepsilon^2}\right)$ samples from the unit $\ell_2$ ball in $\mathbb{R}^d$ and outputs an estimate $\widehat r$ satisfying $|\widehat r-\|\theta\|_2|\le \varepsilon$ with probability at least $1-\delta$, where $\theta$ is the unknown reward vector.

Taken together, these results illustrate when adaptivity can offer only modest savings and when it can enable genuine polynomial gains, sharpening our understanding of the role of adaptivity and geometry in pure exploration and experimental design.

\end{abstract}


\section{Introduction}
Experimental design is a classical and widely studied topic dating back to Fisher’s foundational work \citep{fisher1935design}, with applications spanning clinical trials, A/B testing, engineering, and scientific discovery. A major part of its appeal is that it allows one to commit in advance to a fixed set of measurements to collect, that is, a non-adaptive design; in many linear settings, there is a rich theory for constructing such designs \citep{pukelsheim2006optimal}.  Such non-adaptive designs also offer practical advantages over adaptive approaches, including computational simplicity and the ability to parallelize or batch data collection. Moreover, in certain linear settings, such non-adaptive approaches are known to be near-minimax optimal \citep{arias2012fundamental,even2002pac,fan2025universal}. Motivated by these advantages, we revisit a fundamental sequential decision problem in a linear setting, namely {pure exploration in} linear bandits, and ask how much adaptivity can help beyond the best non-adaptive design. In particular, we study the minimax complexity of non-adaptive designs and ask whether the structure of the action set confines adaptive sampling to logarithmic gains, or allows polynomial improvements over any non-adaptive design.

In this paper, we study the $\varepsilon$-best arm identification  problem {(i.e., pure exploration)} in linear bandits. We are given a compact arm set $\mc{X}\subset \R^d$ with $\mathrm{span}(\mc{X})=\R^d$ and an unknown reward vector $\theta\in\R^d$. At each round $t=1,2,\dots$, an algorithm selects an arm $x_t\in\mc{X}$ and observes $y_t=\langle x_t,\theta\rangle+\eta_t$, where $\eta_t\sim\mc{N}(0,1)$\footnote{All our algorithmic results extend to independent mean-zero subgaussian noise with bounded variance proxy. We restrict attention to Gaussian noise for simplicity of presentation.}.
 Given $\varepsilon\in(0,1)$ and $\delta\in(0,1)$, the objective is to design an algorithm, consisting of a sampling rule and a stopping time, such that upon stopping it outputs $\widehat{x}\in\mc{X}$ satisfying $\P\!\left(\max_{x\in\mc{X}}\langle x-\widehat{x},\theta\rangle\le \varepsilon\right)\ge 1-\delta$. We call an algorithm non-adaptive if its sampling rule does not depend on the observed rewards $y_t$, and adaptive otherwise. Our goal is to characterize the power and limitations of adaptivity in this problem setting.

When $\mathcal{X}=\{e_1,e_2,\ldots,e_d\}$, the problem reduces to the classical stochastic multi-armed bandit setting. A simple non-adaptive algorithm samples each arm $\mathcal{O}\!\left(\frac{\log(d/\delta)}{\varepsilon^2}\right)$ times and outputs the arm with the largest empirical mean. A union bound shows that, with probability at least $1-\delta$, the returned arm is $\varepsilon$-best. One way to extend this approach to linear bandits is via a fixed design. In optimal experimental design, a design distribution $\lambda$ is a probability distribution over $\mathcal{X}$. Using a finitely supported $\lambda$, one can construct a fixed design, that is, a deterministic sequence of arms, in various ways. Typically, the number of times an arm $x\in\mathcal{X}$ appears in the sequence is chosen proportional to $\lambda(x)$, though other constructions are also possible.
The sampling rule then pulls the arms in this predetermined order and observes the corresponding rewards. Using these observations, one typically forms an unbiased estimator $\widehat{\theta}$ and outputs the empirically best arm. We refer to such procedures as non-adaptive fixed-design algorithms.

One can view the non-adaptive multi-armed bandit algorithm discussed above as a fixed-design method induced by the uniform design distribution over the arms. This procedure is already near-optimal for multi-armed bandits: adaptive algorithms such as Median Elimination \citep{even2002pac} achieve the optimal minimax sample complexity $\mathcal{O}\!\left(\frac{d\log(1/\delta)}{\varepsilon^2}\right)$, improving over the optimal non-adaptive rate only by a logarithmic factor. This motivates the following question for the more general linear bandit setting:
\begin{quote}
\emph{What is the minimax sample complexity of non-adaptive fixed-design algorithms for $\varepsilon$-best arm identification in linear bandits? {Moreover, can adaptivity improve over the optimal non-adaptive rate only by logarithmic factors, as in multi-armed bandits, or are polynomial-factor improvements possible?}}
\end{quote}

\subsection{Notations}
Before we address the question raised above, we would like to mention a few notations. The unit $\ell_2$ ball refers to the set $\mathbb{B}_d := \{x \in \mathbb{R}^d : \|x\|_2 \le 1\}$.  If $\lambda$ is a distribution over the set $\mathcal{X}$, we say $\lambda \in \triangle_{\mathcal{X}}$ and will denote $A(\lambda; \mc{X}) := \mathbb{E}_{x\sim\lambda}\left[x x^\top\right]$. We define the gaussian width term $w(\mathcal{X})$ depending on $\mathcal{X}$ as 
\begin{align}
    w(\mathcal{X}):=\inf_{\lambda \in \triangle_{\mathcal{X}}}\mathbb{E}_{\eta\sim\mathcal{N}(0,I_d)}\left[\max_{x\in\mathcal{X}}\langle x,A(\lambda,\mc{X})^{-1/2}\eta \rangle\right]. \label{eqn:gauss_width}
\end{align}
For any matrix $W\in \mathbb{R}^{d\times d}$ and a vector $x\in\mathbb{R}^d$, we will denote $\|x\|_W^2 := x^\top W x$. 
For any positive scalar $x$ define $(x)_+ = \max\{1,x\}$ as the maximum of $x$ and one, \emph{not} zero.
Define $x_* = \arg\max_{x \in \mc{X}} \langle x ,\theta \rangle$.

In this paper, we focus on $(\varepsilon,\delta)$-PAC algorithms that, for any $\varepsilon\in(0,1]$ and $\delta\in(0,1)$, use $\frac{H_1\log(1/\delta)+H_2}{\varepsilon^2}$ samples, where $0<H_1\leq H_2$, and output $\hat{x}\in\mathcal{X}$ such that $\mathbb{P}(\langle x_*-\hat{x},\theta\rangle \le \varepsilon) \ge 1-\delta$.
\subsection{Our Contributions and Techniques}
We now address the question posed above by summarizing our main contributions and techniques. We first present a non-adaptive algorithm based on a fixed design with sample complexity $\mathcal{O}\!\left(\frac{d\log(1/\delta)}{\varepsilon^2}+\frac{w(\mathcal{X})^2}{\varepsilon^2}\right)$, where $w(\mathcal{X})$ is the Gaussian width term defined by \eqref{eqn:gauss_width} depending on $\mathcal{X}$. Our fixed design has length $\mathcal{O}\!\left(\frac{d\log(1/\delta)}{\varepsilon^2}+\frac{w(\mathcal{X})^2}{\varepsilon^2}\right)$ and is constructed using two design distributions, $\lambda_1$ and $\lambda_2$. Here, $\lambda_1$ minimizes $\max_{x\in\mathcal{X}}\|x\|^2_{A(\lambda;\mathcal{X})^{-1}}$, while $\lambda_2$ minimizes $\mathbb{E}_{\eta\sim\mathcal{N}(0,I_d)}\!\left[\max_{x\in\mathcal{X}}\big\langle x, A(\lambda;\mathcal{X})^{-1/2}\eta\big\rangle\right]$. Given the resulting samples, we form the unbiased least-squares estimator $\widehat{\theta}$ and output an arm $\widehat{x}\in\arg\max_{x\in\mathcal{X}}\langle x,\widehat{\theta}\rangle$. Using the Borell-TIS inequality, we show that $\widehat{x}$ is an $\varepsilon$-best arm with probability at least $1-\delta$.

We then prove a matching minimax lower bound of $\Omega\!\left(\frac{d\log(1/\delta)}{\varepsilon^2}+\frac{w(\mathcal{X})^2}{\varepsilon^2}\right)$ for any non-adaptive fixed-design algorithm. We first show a lower bound of $\Omega\!\left(\frac{d\log(1/\delta)}{\varepsilon^2}\right)$ for any algorithm (possibly adaptive). For the $\Omega\!\left(\frac{w(\mathcal{X})^2}{\varepsilon^2}\right)$ lower bound, our proof reduces $\varepsilon$-best arm identification to bounding the minimax simple regret after $T$ rounds. We then select a Gaussian prior over the reward vector $\theta$ tailored to the fixed design $(x_t)_{t\in[T]}$, more precisely to the design matrix $\sum_{t=1}^T x_t x_t^\top$, and use it to derive a lower bound on the simple regret. 

We can make our algorithm adaptive as follows. We partition the arm set into $d$ regions and attempt to identify a candidate good arm from each region. We then run Median Elimination on these $d$ candidate arms to obtain an $\varepsilon$-best arm. In certain cases, this yields a $\log d$ improvement, since the analysis can be localized to the region containing the best arm and the Borell--TIS inequality inequality is applied only over that region.

This observation also motivates us to ask whether there are structured action sets beyond the multi-armed bandit case for which adaptivity yields at most logarithmic factors improvement over our non-adaptive algorithm. Towards that we study several structured action sets of interest, namely the hypercube, the unit $\ell_2$-ball, $m$-sets, and multi-task multi-armed bandits (MAB), and characterize the corresponding Gaussian width terms $w(\mathcal{X})$. For each of these sets, we show, by establishing an adaptive lower bound, that adaptive algorithms can improve over our non-adaptive fixed-design sample complexity by at most logarithmic factors only. Our main technical contributions in this part are the lower bound arguments for $m$-sets and multi-task MAB. We begin by constructing a family of hard instances that is oblivious to the algorithm: each instance specifies a reward vector whose coordinates are chosen as $\varepsilon$ times coordinate-specific scaling factor. We then show that, unless the algorithm collects sufficiently many samples, it must incur simple regret $\Omega(\varepsilon)$. To analyze the simple regret, we localize the KL-divergence-based argument. Concretely, we restrict attention to a structured subset of instances in which a portion of the reward vector is fixed and any two instances differ in exactly two coordinates within the remaining portion. This reduction enables a KL-divergence analysis closely analogous to the classical multi-armed bandit setting. Finally, a global averaging argument lifts the localized bound to the full family, yielding simple regret $\Omega(\varepsilon)$ and, consequently, a minimax sample complexity lower bound.

The above result leads us to our main question: does adaptivity in linear bandits improve over non-adaptive approaches by only logarithmic factors? We answer this in the negative by constructing an action set $\mathcal{X}$ for which adaptivity yields a {\it polynomial-factor} improvement over all non-adaptive algorithms. We take $\mathcal{X}$ to be the union of $k=\mathrm{poly}(d)$ unit $\ell_2$-balls, each contained in a distinct $d$-dimensional subspace. A naive non-adaptive approach allocates on the order of $d^2/\varepsilon^2$ samples to each ball to form an unbiased estimator and identify an $\varepsilon$-best arm, leading to total sample complexity $kd^2/\varepsilon^2$. In contrast, an adaptive strategy can save a factor of $d$ by first identifying a ball that contains a near-optimal arm and then allocating additional samples to that ball to recover a near-optimal arm using just on the order of $\frac{kd + d^2}{\varepsilon^2}$ samples.

A key step in identifying such a “good” unit ball is an $\ell_2$-norm estimation subroutine, which aims to estimate the best value over the unit ball, namely the $\ell_2$-norm of the reward vector. Specifically, we design an adaptive algorithm that takes $\mathcal{O}\!\left(\frac{d\log(1/\delta)}{\varepsilon^2}\right)$ samples from the unit $\ell_2$-ball in $\mathbb{R}^d$, observes noisy rewards of the form $\langle x,\theta\rangle+\eta$ for an unknown vector $\theta\in\mathbb{R}^d$, and outputs an estimate $\widehat r$ such that, with probability at least $1-\delta$, $\widehat r\in[\|\theta\|_2-\varepsilon,\|\theta\|_2+\varepsilon]$. A key technical component of our algorithm is to sample multiple Rademacher unit vectors uniformly at random and take multiple observations along each sampled direction. We then construct an $\ell_2$-norm estimator by squaring the empirical mean reward along each direction and aggregating these squared estimates. This squaring step leads to a subexponential-tail analysis, which is the main non-trivial part of our proof required to establish the stated sample complexity. 

By constructing such a set $\mathcal{X}$, we highlight a key technical insight: accurately estimating the best value over a region $\mathcal{Z}\subset\mathcal{X}$, namely $\max_{x\in\mathcal{Z}}\langle x,\theta\rangle$, plays a crucial role in achieving improved performance. In the special case where $\mathcal{Z}$ is a unit $\ell_2$-ball, this task reduces to $\ell_2$-norm estimation. More generally, this observation suggests that region-wise value estimation is a fundamental ingredient in the design of near-optimal policies in related settings.

\subsection{Related Works}
Our work is closely related to a line of research on instance-dependent sample complexity for best arm identification \citep{soare2014best,karnin2016verification,xu2018fully,tao2018best,fiez2019sequential,jedra2020optimal,katz2020empirical,degenne2020gamification}. A key distinction between instance-dependent and minimax guarantees is the role of $\theta$. In the instance-dependent regime, one seeks sample complexity bounds of the form $\mathbf{H}_1\log(1/\delta)+\mathbf{H}_2$, where $\mathbf{H}_1$ and $\mathbf{H}_2$ depend on the arm set $\mathcal{X}$ and the reward vector $\theta$. In contrast, minimax characterizations aim for bounds of the same form with $\mathbf{H}_1$ and $\mathbf{H}_2$ depending on the arm set $\mathcal{X}$ and the accuracy parameter $\varepsilon$, but not on $\theta$, since the guarantee is worst-case over $\theta$. While several instance-dependent algorithmic ideas extend to the minimax setting, lower bounds are less transferable since they typically hinge on simple-regret arguments, and minimax complexity characterizations remain relatively sparse \citep{even2002pac,shamir2015complexity,chen2024assouad}.

One approach to designing algorithms in the instance-dependent regime is through experimental design. Early works such as \cite{soare2014best} and \cite{karnin2016verification} used G-optimal designs. Later, \cite{fiez2019sequential} proposed an algorithm based on sequential experimental design that is near-optimal for the first term $\mathbf{H}_1$, but can incur a large second term $\mathbf{H}_2$. This motivated \cite{katz2020empirical} to design an algorithm whose design distribution minimizes certain Gaussian width terms, thereby reducing $\mathbf{H}_2$. Such a Gaussian-width based approach to experimental design is also used in a related regret minimization setting by \cite{wagenmaker2021experimental}. For a detailed discussion of experimental design and its connection to linear bandits, we refer the reader to \cite{lattimore2020bandit}.

Our work is also related to the literature on norm estimation and value estimation. \cite{kong2020sublinear} studied optimal policy value estimation in contextual bandits. In a different direction, \cite{cai2011testing} studied estimating $\frac{1}{n}\sum_i |\theta_i|$ from an observation $Y\sim\mathcal{N}(\theta,I_n)$, and \cite{collier2020estimation} later generalized this to estimating $\frac{1}{n}\sum_i |\theta_i|^\gamma$ for $\gamma>0$. \cite{han2020estimation} studied $L_r$-norm estimation in Gaussian white noise models. More recently, \cite{cleanthous2025adaptive} studied adaptive estimation of the $L_2$-norm of a probability density on $\mathbb{R}^d$.


\section{Results on the Power of Adaptivity}\label{sec:non-adaptive}
In this section, we analyze the power of adaptivity for $\varepsilon$-best arm identification in linear bandits. In Section \ref{sec:upper}, we present a non-adaptive fixed design algorithm with sample complexity $\mathcal{O}\!\left(\frac{d\log(1/\delta)}{\varepsilon^2}+\frac{w(\mathcal{X})^2}{\varepsilon^2}\right)$, where $w(\mathcal{X})$ is a Gaussian width term depending on $\mathcal{X}$. We then prove a matching minimax lower bound of $\Omega\!\left(\frac{d\log(1/\delta)}{\varepsilon^2}+\frac{w(\mathcal{X})^2}{\varepsilon^2}\right)$ for non-adaptive fixed-design algorithms in Section \ref{sec:lower}. In Section \ref{sec:properties-gaussian}, we study the Gaussian width term $w(\mathcal{X})$ and characterize it for the structured sets of interest: the hypercube, the unit $\ell_2$-ball, $m$-sets, and multi-task multi-armed bandits. For each of these sets, we show in Section \ref{sec:log-adaptive} that adaptivity improves upon the minimax sample complexity of our non-adaptive algorithm by at most logarithmic factors only. Finally, in Section \ref{sec:poly-adaptive}, we construct a set $\mathcal{X}$ for which an adaptive algorithm yields a polynomial-factor improvement over all non-adaptive algorithms.

\subsection{Non-Adaptive Fixed Design Algorithm}\label{sec:upper}

In this section, we present a non-adaptive  fixed design algorithm that uses
$\mathcal{O}\!\left(\frac{d\log(1/\delta)}{\varepsilon^2}+\frac{w(\mathcal{X})^2}{\varepsilon^2}\right)$
samples and returns an $\varepsilon$-best arm with probability at least $1-\delta$.
Our analysis relies on several standard technical results, which are stated in
Appendix~\ref{appendix:technical-lemmas}.


Let $\lambda_1$ be a distribution over $\mathcal{X}$ that minimizes the quantity
$\mathbb{E}_{\eta\sim\mathcal{N}(0,I_d)}\!\left[\max_{x\in\mathcal{X}}\langle x, A(\lambda;\mathcal{X})^{-1/2}\eta\rangle\right]$.
Similarly, let $\lambda_2$ be a distribution over $\mathcal{X}$ that minimizes
$\max_{x\in\mathcal{X}} \|x\|^2_{A(\lambda;\mathcal{X})^{-1}}$.
Assume that both $\lambda_1$ and $\lambda_2$ have finite support (such minimizers always exist due to Carath\'eodory's theorem ).
Finally, define $\lambda_0$ to be the mixture distribution that samples from $\lambda_1$ with probability $1/2$ and from $\lambda_2$ with probability $1/2$.


Since $\tfrac{1}{2}A(\lambda_2;\mathcal{X}) \preceq A(\lambda_0;\mathcal{X})$, it follows that for every $x\in\mathcal{X}$,
$x^\top A(\lambda_0;\mathcal{X})^{-1}x \le 2\,x^\top A(\lambda_2;\mathcal{X})^{-1}x$.
Because $\lambda_2$ is $G$-optimal, we obtain $\|x\|^2_{A(\lambda_0;\mathcal{X})^{-1}} \le 2d$ for all $x\in\mathcal{X}$.

Similarly, since $\tfrac{1}{2}A(\lambda_1;\mathcal{X}) \preceq A(\lambda_0;\mathcal{X})$, the Sudakov--Fernique inequality yields
\[
\mathbb{E}_{\eta\sim\mathcal{N}(0,I_d)}\!\left[\max_{x\in\mathcal{X}}\langle x, A(\lambda_0;\mathcal{X})^{-1/2}\eta\rangle\right]
\le \sqrt{2}\,
\mathbb{E}_{\eta\sim\mathcal{N}(0,I_d)}\!\left[\max_{x\in\mathcal{X}}\langle x, A(\lambda_1;\mathcal{X})^{-1/2}\eta\rangle\right]
= \sqrt{2}\, w(\mathcal{X}).
\]

Now consider a fixed design $x_1,x_2,\ldots,x_T\in\mathcal{X}$ such that $\tau(A_T)\le 2\,\tau(A(\lambda_0;\mathcal{X}))$, where
$$A_T := \frac{1}{T}\sum_{i=1}^T x_i x_i^\top \text{ and }
\tau(A) := \mathbb{E}_{\eta\sim\mathcal{N}(0,I_d)}\!\left[\max_{x\in\mathcal{X}} x^\top A^{-1/2}\eta\right]^2 + 2\max_{x\in\mathcal{X}}\|x\|_{A^{-1}}^2 \log(2/\delta).$$
Such a fixed design exists for any $T\ge 180d$ \citep{katz2020empirical,allen2021near}.
By the bounds established above, it follows that $\tau(A_T)\le 4\,w(\mathcal{X})^2 + 8d\log(2/\delta)$.



Let $y_t=\langle x_t,\theta\rangle+\eta_t$ denote the noisy rewards under the fixed design, where $\eta_t\sim\mathcal{N}(0,1)$ are i.i.d.
Define the least-squares estimator
$\widehat{\theta}=\bigl(\sum_{t=1}^T x_t x_t^\top\bigr)^{-1}\sum_{t=1}^T x_t y_t$.
Note that $\widehat{\theta}$ is distributionally equivalent to $\theta+\bigl(\sum_{t=1}^T x_t x_t^\top\bigr)^{-1/2}\eta$, where $\eta\sim\mathcal{N}(0,I_d)$.
Consequently, applying the Borell--TIS inequality to the Gaussian process $V_x = x^\top(\widehat{\theta}-\theta)$, we obtain the following with probability at least $1-\delta$:
\begin{equation*}
\Big|\max_{x\in\mathcal{X}}\langle x,\widehat{\theta}-\theta\rangle\Big|
\le \frac{1}{\sqrt{T}}\cdot \mathbb{E}_{\eta\sim\mathcal{N}(0,I_d)}\!\left[\max_{x\in\mathcal{X}} x^\top A_T^{-1/2}\eta\right]
+ \frac{1}{\sqrt{T}}\cdot \sqrt{2\max_{x\in\mathcal{X}}\|x\|_{A_T^{-1}}^2\log(2/\delta)}
\le \sqrt{\frac{2\,\tau(A_T)}{T}}.
\end{equation*}

We select $\widehat{x}\in\arg\max_{x\in\mathcal{X}}\langle x,\widehat{\theta}\rangle$ and output it as our candidate $\varepsilon$-best arm.
If we set $T=360\bigl(\tfrac{d}{\varepsilon^2}\log(2/\delta)+\tfrac{w(\mathcal{X})^2}{\varepsilon^2}\bigr)$, then with probability at least $1-\delta$ we have
$\bigl|\max_{x\in\mathcal{X}}\langle x,\widehat{\theta}-\theta\rangle\bigr|\le \varepsilon/2$.
It follows that, with probability at least $1-\delta$,
$\langle \widehat{x},\theta\rangle \ge \max_{x\in\mathcal{X}}\langle x,\theta\rangle - \varepsilon$.

Hence, we have the following theorem.
\begin{theorem}\label{thm:adaptive-upper-bound-algo}
There exists a non-adaptive fixed-design algorithm that first selects a deterministic sequence of arms $x_1,x_2,\ldots,x_T\in\mathcal{X}$ with $T=360\!\left(\frac{d\log(1/\delta)}{\varepsilon^2}+\frac{w(\mathcal{X})^2}{\varepsilon^2}\right)$, then observes the corresponding noisy rewards $y_t=\langle x_t,\theta\rangle+\eta_t$, where $\eta_t\sim\mathcal{N}(0,1)$ are i.i.d., and outputs an arm $\widehat{x}\in\mathcal{X}$ such that, with probability at least $1-\delta$, $\max_{x\in\mathcal{X}} \langle x-\widehat{x},\theta\rangle \le \varepsilon$.
\end{theorem}

\subsection{Lower Bound for Non-Adaptive Fixed Design Algorithms}\label{sec:lower}
In this section, we prove a lower bound of $\Omega\!\left(\frac{d\log(1/\delta)}{\varepsilon^2}+\frac{w(\mathcal{X})^2}{\varepsilon^2}\right)$ on the sample complexity of any non-adaptive algorithm that commits to a fixed design $x_1,x_2,\ldots,x_T\in\mathcal{X}$, then observes the corresponding noisy rewards $y_1,y_2,\ldots,y_T$, and finally outputs an arm $\widehat{x}\in\mathcal{X}$ that is $\varepsilon$-best with probability at least $1-\delta$.

We begin by stating the following theorem, which follows from a hypothesis-testing reduction and standard KL-divergence arguments. We refer the reader to Appendix \ref{appendix:adaptive-lower-bound} for the complete proof.
\begin{theorem}\label{thm:adaptive-lower-bound}
Any algorithm (possibly adaptive) that outputs $\widehat x\in\mathcal{X}$ satisfying $\P\!\left(\max_{x\in\mc{X}}\langle x-\widehat{x},\theta\rangle\le \varepsilon\right)\ge 1-\delta$ for all $\theta\in\mathbb{R}^d$ must, for some worst-case $\theta$, use $\Omega\!\left(\frac{d\log(1/\delta)}{\varepsilon^2}\right)$ samples.
\end{theorem}
We next present our proof idea for establishing a lower bound of $\Omega\left(\frac{w(\mathcal{X})^2}{\varepsilon^2}\right)$. Towards that we look at a simple regret minimization problem. Consider a deterministic sequence
$(x_t)_{t=1}^T$ with $x_t\in \mathcal{X}$. Recall that we observe
\[
y_t \;=\; \ip{x_t}{\theta}+\eta_t,\qquad \eta_t\stackrel{\text{i.i.d.}}{\sim}\mathcal N(0,1),\quad t=1,\dots,T.
\]
Define the matrix
\[
A \;:=\; \sum_{t=1}^T x_t x_t^\top \;=\; X^\top X,
\]
where $X\in\R^{T\times d}$ has rows $x_t^\top$. Assume that $A$ is invertible (the other case is considered in Appendix \ref{appendix:singular-case-lower-bound}).
An algorithm $\mathcal{A}$ observes the full sequence $\mathcal{D}=\{(x_t,y_t)\}_{t=1}^T$ and outputs $\hat x_{\mathcal{A}}\in \mathcal{X}$.
The \emph{simple regret} is
\[
r(\hat x_{\mathcal{A}},\theta) \;:=\; \max_{x\in \mathcal{X}}\ip{x-\hat x_{\mathcal{A}}}{\theta}.
\]

Fix a prior $\theta \sim \mathcal N(0,\tau^2 A^{-1})$ independent of the noise
$\eta=(\eta_1,\dots,\eta_T)\sim\mathcal N(0,I_T)$ and (if the algorithm is randomized)
an internal random seed $U$, also independent of $(\theta,\eta)$.
Let $y=X\theta+\eta$, $\mathcal D=\{(x_t,y_t)\}_{t=1}^T$, and let the (possibly randomized)
non-interactive rule output $\hat x_{\mathcal{A}}=\hat x_{\mathcal{A}}(\mathcal D,U)\in\mathcal X$. We study the minimax expected regret
\[
\mathfrak R(A;\mathcal{X}) \;:=\; \inf_{\mathcal A}\E\big[r(\hat x_{\mathcal{A}},\theta)\big],
\]
where the infimum is over all (possibly randomized) non-interactive procedures $\mathcal A$ using $\mathcal{D}$. Throughout this proof, $\mathbb E[\cdot]$ denotes expectation over all randomness in $(\theta,\eta)$ and in the algorithm (if randomized), unless stated otherwise.

We define the gaussian width term $w(\mathcal{X};A):=\E_{g\sim\mathcal N(0,I_d)}\!\left[\sup_{x\in \mathcal{X}}\ip{x}{A^{-1/2}g}\right]$. 
If we establish that $\mathfrak{R}(A;\mathcal{X}) \ge c_0\, w(\mathcal{X};A)$ for some absolute constant $c_0$, then combining this with the inequality $w(\mathcal{X};A)\ge w(\mathcal{X})/\sqrt{T}$ and the straightforward calculations in Appendix~\ref{appendix:gaussian-lower-bound} yields the following theorem.
\begin{theorem}\label{thm:adaptive-lower-bound}
Consider any non-adaptive fixed-design algorithm that first selects a deterministic sequence of arms $x_1,x_2,\ldots,x_T\in\mathcal{X}$, then observes the corresponding noisy rewards $y_t=\langle x_t,\theta\rangle+\eta_t$, where $\eta_t\sim\mathcal{N}(0,1)$ are i.i.d., and outputs an arm $\widehat{x}\in\mathcal{X}$ such that $\max_{x\in\mathcal{X}} \langle x-\widehat{x},\theta\rangle \le \varepsilon$ with probability at least $1-\delta$.
Then $T \ge \Omega\!\left(\frac{w(\mathcal{X})^2}{\varepsilon^2}\right)$.
\end{theorem}

For the remainder of this section, we lower bound $\mathfrak R(A;\mathcal X)$.
Fix any (possibly randomized) non-interactive rule $\hat x_{\mathcal A}(\mathcal D,U)\in\mathcal X$.
By the definition of simple regret,
\[
\mathbb{E}[r(\hat x_{\mathcal A},\theta)]
=
\mathbb{E}\!\left[\max_{x\in\mathcal X}\langle x,\theta\rangle\right]
-
\mathbb{E}\!\left[\langle \hat x_{\mathcal A}(\mathcal D,U),\theta\rangle\right].
\]
Applying the tower rule and conditioning on $(\mathcal D,U)$, we have
\[
\mathbb{E}\!\left[\langle \hat x_{\mathcal A}(\mathcal D,U),\theta\rangle\right]
=
\mathbb{E}\!\left[
\mathbb{E}\!\left[\langle \hat x_{\mathcal A}(\mathcal D,U),\theta\rangle \mid \mathcal D,U\right]
\right].
\]
Since $\hat x_{\mathcal A}(\mathcal D,U)$ is measurable with respect to $(\mathcal D,U)$ and $U$ is independent of $\theta$, we can pull it outside the inner expectation, which gives
\[
\mathbb{E}\!\left[\langle \hat x_{\mathcal A}(\mathcal D,U),\theta\rangle\right]=
\mathbb{E}\!\left[
\big\langle \hat x_{\mathcal A}(\mathcal D,U),\, \mathbb{E}[\theta\mid \mathcal D]\big\rangle
\right].
\]
Under the Gaussian prior $\theta\sim\mathcal N(0,\tau^2A^{-1})$ and likelihood $y\mid\theta\sim\mathcal N(X\theta,I_T)$, the posterior mean is
$\mathbb{E}[\theta\mid\mathcal D]=\frac{\tau^2}{1+\tau^2}A^{-1}X^\top y$.
Substituting this expression yields
\[
\mathbb{E}[r(\hat x_{\mathcal A},\theta)]
=
\mathbb{E}\!\left[\max_{x\in\mathcal X}\langle x,\theta\rangle\right]
-
\frac{\tau^2}{1+\tau^2}\,
\mathbb{E}\!\left[
\big\langle \hat x_{\mathcal A}(\mathcal D,U),\, A^{-1}X^\top y\big\rangle
\right].
\]
Using $y=X\theta+\eta$ and $A=X^\top X$, we obtain
\[
\mathbb{E}[r(\hat x_{\mathcal A},\theta)]=
\mathbb{E}\!\left[\max_{x\in\mathcal X}\langle x,\theta\rangle\right]
-
\frac{\tau^2}{1+\tau^2}\,
\mathbb{E}\!\left[
\big\langle \hat x_{\mathcal A}(\mathcal D,U),\, \theta + A^{-1}X^\top\eta\big\rangle
\right].
\]
Next, for any vector $v$, we have
$\langle \hat x_{\mathcal A}(\mathcal D,U),v\rangle
\le \max_{x\in\mathcal X}\langle x,v\rangle$.
Applying this bound gives
\[
\mathbb{E}[r(\hat x_{\mathcal A},\theta)]\ge
\mathbb{E}\!\left[\max_{x\in\mathcal X}\langle x,\theta\rangle\right]
-
\frac{\tau^2}{1+\tau^2}\,
\mathbb{E}\!\left[
\max_{x\in\mathcal X}\langle x,\theta + A^{-1}X^\top\eta\rangle
\right].
\]
As $\max_{x\in\mathcal{X}}\langle x,\theta + A^{-1}X^\top\eta\rangle
\le
\max_{x\in\mathcal{X}}\langle x,\theta\rangle
+
\max_{x\in\mathcal{X}}\langle x,A^{-1}X^\top\eta\rangle,$ we get
\[
\mathbb{E}[r(\hat x_{\mathcal A},\theta)]\ge
\mathbb{E}\!\left[\max_{x\in\mathcal X}\langle x,\theta\rangle\right]
-
\frac{\tau^2}{1+\tau^2}
\left(
\mathbb{E}\!\left[\max_{x\in\mathcal X}\langle x,\theta\rangle\right]
+
\mathbb{E}\!\left[\max_{x\in\mathcal X}\langle x,A^{-1}X^\top\eta\rangle\right]
\right).
\]
Since $\theta\sim\mathcal N(0,\tau^2A^{-1})$, we have $\theta\stackrel{d}{=}\tau A^{-1/2}\xi$ for $\xi\sim\mathcal N(0,I_d)$, and thus $\mathbb{E}\!\left[\max_{x\in\mathcal X}\langle x,\theta\rangle\right]=\tau\,\mathbb{E}_{\xi\sim\mathcal N(0,I_d)}\!\left[\max_{x\in\mathcal X}\langle x,A^{-1/2}\xi\rangle\right]$.
Similarly, since $\eta\sim\mathcal N(0,I_T)$ and $A=X^\top X$, we have $A^{-1}X^\top\eta\stackrel{d}{=}\mathcal N(0,A^{-1})$, and hence $A^{-1}X^\top\eta\stackrel{d}{=}A^{-1/2}\xi$ for $\xi\sim\mathcal N(0,I_d)$.
Therefore, $\mathbb{E}\!\left[\max_{x\in\mathcal X}\langle x,A^{-1}X^\top\eta\rangle\right]=\mathbb{E}_{\xi\sim\mathcal N(0,I_d)}\!\left[\max_{x\in\mathcal X}\langle x,A^{-1/2}\xi\rangle\right]$.
Substituting these two identities yields
\[
\mathbb{E}[r(\hat x_{\mathcal A},\theta)]
\ge
\frac{\tau(1-\tau)}{1+\tau^2}\;
\mathbb{E}_{\xi\sim\mathcal N(0,I_d)}
\!\left[\max_{x\in\mathcal X}\langle x,A^{-1/2}\xi\rangle\right].
\]
Finally, the function $\frac{\tau(1-\tau)}{1+\tau^2}$ is maximized at $\tau=\sqrt{2}-1$, which yields a universal constant of approximately $0.207$.

\subsection{Properties of Gaussian Width}\label{sec:properties-gaussian}
In this section, we study basic properties of the term $w(\mathcal{X})$. First, we present the following proposition.

\begin{proposition}\label{prop:gaussian-width-basic-bounds}
For any $\mathcal{X}\subset \mathbb{R}^d$, we have $w(\mathcal{X}) \ge \Omega(\sqrt{d\log d})$ and $w(\mathcal{X}) \le O(d)$.
Moreover, if $\mathcal{X}$ is finite, then $w(\mathcal{X}) \le O(\sqrt{d\log|\mathcal{X}|})$.
\end{proposition}

These bounds are tight for suitable choices of $\mathcal{X}$.
For instance, in the classical multi-armed bandit setting, $w(\mathcal{X})=\Theta(\sqrt{d\log d})$ (as discussed in the introduction).
In contrast, for structured sets such as the hypercubes $\{-1,+1\}^d$ and $\{0,1\}^d$, as well as the unit $\ell_2$ ball in $\mathbb{R}^d$, one can verify that $w(\mathcal{X})=\Theta(d)$.
A more nuanced example is the class of $m$-sets, defined by $\mathcal{X}:=\{x\in\{0,1\}^d:\|x\|_1=m\}$, for which one can show $w(\mathcal{X})\ge \Omega(\sqrt{md})$; this matches the finite-set upper bound of $O(\sqrt{d \log |\mathcal{X}|})$ up to logarithmic factors.
We provide formal proofs of these claims and Proposition \ref{prop:gaussian-width-basic-bounds} in Appendix~\ref{appendix:gaussian-width-properties}.

These observations raise a natural question: does there exist a set $\mathcal{X}\subset \mathbb{R}^d$ of dimension $\Theta(d)$ for which $w(\mathcal{X})$ is polynomially smaller than what the upper bounds in Proposition~\ref{prop:gaussian-width-basic-bounds} might suggest?
We answer this question in the affirmative.
To this end, we introduce the multi-task multi-armed bandit problem, a well-known generalization of the classical multi-armed bandit setting that has been studied in prior work \citep{cesa2012combinatorial,cohen2017tight,maiti2025efficient,fan2025universal}.

Informally, in the multi-task multi-armed bandit (MAB) problem, we are given $m$ bandit problems, where the $i$-th problem contains $d_i \ge 2$ arms.
In each round, the learner selects one arm from every problem simultaneously and observes as feedback the sum of the corresponding rewards plus Gaussian noise.

We now formalize the multi-task MAB problem.
Let $d=\sum_{i=1}^m d_i$, and define $d_{1:i}=\sum_{j=1}^i d_j$ with the convention $d_{1:0}=0$.
The arm set $\mathcal{X}$ is defined as follows:
\begin{equation*}
    \mathcal X=\left\{x\in \{0,1\}^d: \forall j\in [m] \sum_{i=d_{1:j-1}+1}^{d_{1:j}} x_i=1\right\}.
\end{equation*}

For this set $\mathcal{X}$, we show that its dimension is $d-m+1=\Theta(d)$ and that its Gaussian width satisfies
$w(\mathcal{X}) \le O\!\left(\sum_{i=1}^m \sqrt{d_i\log d_i}\right)$; the proof appears in Appendix~\ref{appendix:gaussian-width-multi-task-MAB}.
Now suppose that $d_i=2$ for all $i\in[m-1]$ and $d_m=m^2$.
Then $\dim(\mathcal{X})=\Theta(m^2)$ and $\min\{d,\sqrt{d\log|\mathcal{X}|}\}=\Theta(m^{3/2})$.
On the other hand, the Gaussian width can be upper bounded as:
\[
w(\mathcal{X})\le O\!\left(\sum_{i=1}^m\sqrt{d_i\log d_i}\right)
= O\!\left(m+\sqrt{m^2\log m}\right)
= O\!\left(m\sqrt{\log m}\right),
\]
which is polynomially smaller than $\Theta(m^{3/2})$. We summarize the above discussion in the following theorem.
\begin{theorem}\label{thm:gaussian-width-separation}
There exists a finite set $\mathcal{X}\subset\mathbb{R}^d$ such that $\dim(\mathcal{X})=\Theta(d)$ and $w(\mathcal{X})\le O(\sqrt{d\log d})$, while $\sqrt{d\log|\mathcal{X}|}\ge \Omega(d^{3/4})$.
\end{theorem}
This example shows that $w(\mathcal{X})$ can be genuinely non-trivial, and in particular need not scale as $\Theta(d)$ nor as $\Theta(\sqrt{d\log|\mathcal{X}|})$ even when $\mathcal{X}$ is finite.

\subsection{Structured Sets $\mathcal{X}$ for Which Adaptivity Yields Only Logarithmic-Factor Improvements}\label{sec:log-adaptive}
For classical multi-armed bandits, Median Elimination, an adaptive algorithm, improves upon the non-adaptive approach described in the introduction by a factor of $\log d$. In the same spirit, we can make our algorithm from Section~\ref{sec:upper} adaptive as follows. We partition the arm set into $d$ regions and attempt to identify a candidate good arm from each region. We then run Median Elimination on these $d$ candidate arms to obtain an $\varepsilon$-best arm. The analysis of this adaptive approach can be localized to the region containing the best arm, and the Borell--TIS inequality is applied only over that region. This can lead to a $\log d$ improvement in certain regimes. In particular, for a finite set $\mathcal{X}\subset\mathbb{R}^d$, our adaptive algorithm satisfies the following guarantee.
\begin{theorem}\label{thm:median-eliminination-adaptive}
There exists an $(\varepsilon,\delta)$-PAC adaptive algorithm for $\varepsilon$-best arm identification with sample complexity
$O\!\left(\frac{d\log\!\bigl((|\mathcal{X}|/d)_+/\delta\bigr)}{\varepsilon^2}\right)$
for any finite set $\mathcal{X}\subset \mathbb{R}^d$.
\end{theorem}
Consequently, for those finite sets $\mathcal{X}\subset\mathbb{R}^d$ for which a lower bound of $\Omega\!\left(\frac{d\log(|\mathcal{X}|/\delta)}{\varepsilon^2}\right)$ is known to hold for $(\varepsilon,\delta)$-PAC non-adaptive algorithms, our adaptive approach improves the sample complexity to $O\!\left(\frac{d\log\!\bigl((|\mathcal{X}|/d)_+/\delta\bigr)}{\varepsilon^2}\right)$.
We refer the reader to Appendix~\ref{appendix:adaptive-version} for formal details. This raises a natural question of whether there are structured settings beyond multi-armed bandits in which adaptivity yields at most logarithmic improvements, or even no improvement at all. We answer this question in the affirmative for the structured sets studied in the previous section, namely the unit $\ell_2$-ball, the $\{-1,+1\}^d$ and $\{0,1\}^d$ hypercubes, $m$-sets, and multi-task MAB.


\begin{table}[t]
\centering
\caption{Sample complexity bounds for different structured action sets for which the adaptivity only yields logarithmic improvement.}
\begin{tabular}{lcc}
\toprule
Set 
& Non-adaptive upper bound 
& Adaptive lower bound \\
\midrule
unit $\ell_2$ ball, $\{-1,1\}^d$, $\{0,1\}^d$
& $O\!\left(\frac{d\log(1/\delta)}{\varepsilon^2} + \frac{d^2}{\varepsilon^2}\right)$
& $\Omega\!\left(\frac{d\log(1/\delta)}{\varepsilon^2} + \frac{d^2}{\varepsilon^2}\right)$ \\

$m$-sets $\quad(m\leq d/21)$
& $O\!\left(\frac{d\log(1/\delta)}{\varepsilon^2} + \frac{md \log(d/m)}{\varepsilon^2}\right)$
& $\Omega\!\left(\frac{d\log(1/\delta)}{\varepsilon^2} + \frac{md}{\varepsilon^2}\right)$ \\

Multi-task MAB
& $O\!\left(\frac{d\log(1/\delta)}{\varepsilon^2} + \frac{\bigl(\sum_{j=1}^m \sqrt{d_j \log d_j}\bigr)^2}{\varepsilon^2}\right)$
& $\Omega\!\left(\frac{d\log(1/\delta)}{\varepsilon^2} + \frac{\bigl(\sum_{j=1}^m \sqrt{d_j}\bigr)^2}{\varepsilon^2}\right)$ \\
\bottomrule
\label{table:1}
\end{tabular}
\end{table}

For these structured sets, the optimal minimax sample complexity takes the form $\mathbf{H}_1\log(1/\delta)+\mathbf{H}_2$. Combining the results of Sections~\ref{sec:upper} and~\ref{sec:lower}, we conclude that $\mathbf{H}_1=\Theta\!\left(\frac{d}{\varepsilon^2}\right)$. The more interesting term is $\mathbf{H}_2$, which can dominate when $\delta$ is an absolute constant. We now establish lower bounds on $\mathbf{H}_2$ for the various structured sets; these bounds are also summarized in Table~\ref{table:1}.

For the unit $\ell_2$ ball, the simple-regret lower bound of \cite{chen2024assouad}, which is based on a Gaussian-prior construction, together with the straightforward calculations in Appendix~\ref{appendix-ball-lower-bound}, implies that $\mathbf{H}_2 \ge \Omega\!\left(\frac{d^2}{\varepsilon^2}\right)$.

For the finite sets $\{-1,+1\}^d$ hypercube, $\{0,1\}^d$ hypercube, $m$-sets, and multi-task MAB, we prove lower bounds via a different approach, based on constructing a finite family of hard instances rather than using a continuous prior. To illustrate this approach, we provide high-level intuition for the hard-instance construction in the multi-task MAB setting. We defer the formal proofs of the lower bounds to Appendix~\ref{appendix-structured-sets-lower-bound}.

\textbf{Intuition}: It is well known that in the $K$-armed bandit problem, if the total number of samples satisfies $T \le \frac{cK}{\varepsilon^2}$ for a sufficiently small universal constant $c$, then no algorithm can identify an $\varepsilon$-best arm with constant success probability. A standard hard family consists of spiked instances indexed by $i_\star\in[K]$, where the unique optimal arm has mean $\mu_{i_\star}=10\varepsilon$ and all other arms have mean $0$, together with an alternative instance in which all arms have mean $0$. Under the uniform distribution over these hard instances, a KL-divergence based argument implies that when $T \le \frac{cK}{\varepsilon^2}$, any deterministic algorithm outputs $\widehat{i}$ with $\mathbb{E}[\mu_{\widehat{i}}]\le 7\varepsilon$, and Markov's inequality yields constant-probability failure; Yao's minimax principle extends the conclusion to randomized algorithms.

For the multi-task setting with $m$ MAB problems, where problem $j$ has $d_j$ arms, we take the cartesian-product hard family obtained by choosing one arm $i_j$ per problem $j$ and setting its mean to $10\varepsilon_j$ while all other arms have mean $0$, yielding $\prod_{j=1}^m d_j$ spiked instances, with the corresponding alternative instances formed by zeroing out a single problem $j$ while keeping the others fixed. Conditioning on the spiked choices in problems $k\neq j$, the $j$-th problem reduces to the standard $d_j$-armed problem, so if $T \le \frac{c\,d_j}{\varepsilon_j^2}$ then $\mathbb{E}[\mu_{\widehat{i}_j}] \le 7\varepsilon_j$. Summing over $j$ gives $\mathbb{E}\!\left[\sum_{j\in[m]}\mu_{\widehat{i}_j}\right]\le 7\sum_{j\in[m]}\varepsilon_j$, while choosing the spiked arm from each problem leads to a total value $10\sum_{j\in[m]}\varepsilon_j$, and Markov's inequality again yields constant-probability failure to achieve additive error $\sum_{j\in[m]}\varepsilon_j$, with extension to randomized algorithms by Yao's minimax principle.

Now set $\varepsilon_j:=\frac{\varepsilon\sqrt{d_j}}{\sum_{s\in[m]}\sqrt{d_s}}$. Then
$\frac{d_j}{\varepsilon_j^2}=\frac{\left(\sum_{s\in[m]}\sqrt{d_s}\right)^2}{\varepsilon^2}$ for all $j\in[m]$,
and
$\sum_{j\in[m]}\varepsilon_j=\varepsilon\cdot\frac{\sum_{j\in[m]}\sqrt{d_j}}{\sum_{s\in[m]}\sqrt{d_s}}=\varepsilon$.
Hence, if an algorithm takes at most $T\le \frac{c\left(\sum_{s\in[m]}\sqrt{d_s}\right)^2}{\varepsilon^2}$ samples for a sufficiently small constant $c$, then it fails to identify an $\varepsilon$-best arm with high constant probability.

\subsection{A Structured Set $\mathcal{X}$ for Which Adaptivity Yields a Polynomial-Factor Improvement}\label{sec:poly-adaptive}
In the previous section, we showed that for several well-known structured sets, adaptivity can yield only logarithmic-factor improvements. This raises an important question: can adaptivity lead to polynomial-factor improvements? We answer this question in the affirmative by constructing an action set $\mathcal{X}$ for which adaptivity yields a polynomial-factor improvement.

Consider positive integers $k,d$ with $d\leq k\leq d^2$. For each $i\in[k]$, define
\[
\mathcal{X}_i:=\left\{x\in\mathbb{R}^{kd}:\ \mathrm{supp}(x)\subseteq \{(i-1)d+1,\ldots,id\},\ \|x\|_2\le 1\right\},
\]
and let $\mathcal{X}:=\bigcup_{i=1}^k \mathcal{X}_i$. We first claim that, for the set $\mathcal{X}$, any non-adaptive algorithm has minimax sample complexity at least $\Omega\!\left(\frac{kd\log(1/\delta)}{\varepsilon^2}+\frac{kd^2}{\varepsilon^2}\right)$ for the $\varepsilon$-best arm identification problem. The term $\Omega\!\left(\frac{kd\log(1/\delta)}{\varepsilon^2}\right)$ follows directly from Theorem~\ref{thm:adaptive-lower-bound} together with the fact that $\mathrm{span}(\mathcal{X})$ has dimension $kd$. The $\Omega\!\left(\frac{kd^2}{\varepsilon^2}\right)$ term arises from the block structure of $\mathcal{X}$. Each set $\mathcal{X}_i$ is a unit $\ell_2$-ball in a $d$-dimensional coordinate subspace, and a non-adaptive algorithm does not know in advance which block contains the optimal arm. As a result, the sampling budget must be spread across all $k$ blocks, requiring on the order of $\Omega\!\left(\frac{d^2}{\varepsilon^2}\right)$ samples per block to ensure that, regardless of which $\mathcal{X}_i$ contains the best arm, the algorithm can identify an $\varepsilon$-best arm within that block upon termination. We prove this claim formally in Appendix~\ref{appendix-poly-sep-lower-bound}.

We now discuss how an adaptive approach can improve upon the above non-adaptive lower bound. A natural strategy is to first identify an index $i$ such that the best arm in $\mathcal{X}_i$ is $\varepsilon/2$-close to the best arm in $\mathcal{X}$, and then find an $\varepsilon/2$-best arm within $\mathcal{X}_i$. To identify such an index $i$, we can proceed as follows: for each $j\in[k]$, estimate the quantity $\max_{x\in\mathcal{X}_j}\langle x,\theta\rangle$ to within additive error $\varepsilon/4$, and output the index with the largest estimated value. After selecting this ``good'' set $\mathcal{X}_i$, we allocate additional samples to $\mathcal{X}_i$ in order to identify an $\varepsilon/2$-best arm in that set.

Let $\theta^{(i)} := (\theta_{(i-1)d+1},\theta_{(i-1)d+2},\ldots,\theta_{id})^\top$. Then $\max_{x\in\mathcal{X}_i}\langle x,\theta\rangle=\|\theta^{(i)}\|_2$. This motivates us to consider the following $\ell_2$-norm estimation problem: given an arm set $\mathbb{B}_d:=\{x\in\mathbb{R}^d:\|x\|_2\leq 1\}$, estimate the $\ell_2$-norm of the reward vector using as few samples as possible while observing the corresponding noisy rewards.

Assume that we can solve this $\ell_2$-norm estimation problem with an algorithm whose minimax sample complexity is $\mathcal{O}\!\left(\frac{d\log(1/\delta)}{\varepsilon^2}\right)$. Then, using $\mathcal{O}\!\left(\frac{kd\log(k/\delta)}{\varepsilon^2}\right)$ samples and a union bound, we can find an index $i\in[k]$ such that, with probability at least $1-\delta/2$, $\max_{x\in\mathcal{X}_i}\langle x,\theta\rangle \ge \max_{x\in\mathcal{X}}\langle x,\theta\rangle-\varepsilon/2$. Next, using our non-adaptive algorithm, we can find an $\varepsilon/2$-best arm in $\mathcal{X}_i$ with probability at least $1-\delta/2$ using an additional $\mathcal{O}\!\left(\frac{d\log(1/\delta)}{\varepsilon^2}+\frac{d^2}{\varepsilon^2}\right)$ samples. Since $k\ge d$, this yields an $\varepsilon$-best arm in $\mathcal{X}$ with probability at least $1-\delta$ using at most $\mathcal{O}\!\left(\frac{kd\log(1/\delta)}{\varepsilon^2}+\frac{kd\log k}{\varepsilon^2}\right)$ samples, yielding a polynomial improvement over all non-adaptive algorithms. We summarize the above discussion in the following theorem.
\begin{theorem}\label{thm:adaptivity-gap-kd}
Fix positive integers $k$ and $d$ satisfying $d\le k\le d^2$.
There exists a set $\mathcal{X}\subset \mathbb{R}^{kd}$ such that any non-adaptive $(\varepsilon,\delta)$-PAC algorithm for $\varepsilon$-best arm identification requires
$\Omega\!\left(\frac{kd\log(1/\delta)}{\varepsilon^2}+\frac{kd^2}{\varepsilon^2}\right)$
samples, while there exists an adaptive $(\varepsilon,\delta)$-PAC algorithm with sample complexity
$O\!\left(\frac{kd\log(1/\delta)}{\varepsilon^2}+\frac{kd\log k}{\varepsilon^2}\right)$,
yielding a polynomial improvement over all non-adaptive approaches.
\end{theorem}
The only remaining task is to design an algorithm that solves the $\ell_2$-norm estimation problem using at most $\mathcal{O}\!\left(\frac{d\log(1/\delta)}{\varepsilon^2}\right)$ samples, which we do in the next section.

\section{An Algorithm for $\ell_2$-Norm Estimation}\label{sec:norm-estimation}

In this section, we address the $\ell_2$-norm estimation problem, {a key ingredient underlying the polynomial gap between adaptive and non-adaptive algorithms in Theorem~\ref{thm:adaptivity-gap-kd}}. Recall that we are given the arm set $\mathbb{B}_d:=\{x\in\mathbb{R}^d:\|x\|_2\le 1\}$ and wish to estimate the $\ell_2$-norm of an unknown reward vector $\theta\in\mathbb{R}^d$ using as few samples $x_t\in\mathbb{B}_d$ as possible, while observing noisy rewards $y_t=\langle x_t,\theta\rangle+\eta_t$, where $\eta_t\sim\mathcal{N}(0,1)$. We solve this problem in three steps. First, we test whether $\lambda_0\varepsilon \le \|\theta\|_2 \le \lambda_1\sqrt{d}$ for some positive constants $\lambda_0$ and $\lambda_1$. If so, we obtain a constant-factor (within $2$) multiplicative estimate of $\|\theta\|_2$. Finally, we use this coarse estimate to refine our estimate of $\|\theta\|_2$ to within an additive error of $\varepsilon$. The final step is the most technically involved part of the analysis, and we outline its high-level ideas here. We formalize the guarantee of the $\ell_2$-norm estimation procedure in the following theorem, and defer the complete description of the procedure and its analysis to Appendix~\ref{appendix-l2-norm-estimation}.

\begin{theorem}\label{thm:l2-norm-estimation}
Let $\theta\in\mathbb{R}^d$ be an unknown reward vector. At each round, a learner selects an action $x_t\in\mathbb{B}_d$ and observes $y_t=\langle x_t,\theta\rangle+\eta_t$, where $\eta_t\sim\mathcal{N}(0,1)$ are i.i.d.
There exists an adaptive algorithm that uses $O\!\left(\frac{d\log(1/\delta)}{\varepsilon^2}\right)$ such observations and returns an estimate $\widehat{r}$ such that $|\widehat{r}-r|\le \varepsilon$ with probability at least $1-\delta$, where $r:=\|\theta\|_2$.
\end{theorem}

Begin by defining $r:=||\theta||_2$ and assuming that we are given a value $\varepsilon< r_0< 2\sqrt{d}$ such that $\frac{r}{2}< r_0\leq 2r$. We now describe an algorithm that takes $\mathcal{O}\left(\frac{d\log(1/\delta)}{\varepsilon^2}\right)$ samples from $\mathbb{B}_d$ and  outputs an estimate $\hat r$ such that with probability $1-\delta/2$, we have $|\hat r-r|\leq \varepsilon$.


\begin{algorithm2e}[H]\label{alg:ball-additive}
\caption{$\varepsilon$-additive error $\ell_2$-norm estimation algorithm}
\LinesNumbered
\KwIn{$\delta \in (0,1), \varepsilon >0, 
s=\frac{c_0d}{r_0^2},\; K=c_1r_0^2\cdot \varepsilon^{-2}\log(4/\delta)$ where $c_0,c_1$ are large absolute constants.}

\For{$k=1,\ldots,K$}{
Draw a Rademacher unit vector
$x^{(k)} = \frac{(\varepsilon_1, \dots, \varepsilon_d)}{\sqrt{d}}
\quad \text{where } \varepsilon_i \overset{iid}{\sim} \{ \pm 1 \}$.

Take $s$ samples along this direction and observe:
$
y_{k,\ell} = \mu_k + \eta_{k,\ell},  \text{ where } \mu_k=\langle x^{(k)},\theta\rangle, \;
\eta_{k,\ell} \overset{iid}{\sim} \mathcal{N}(0,1), \; \ell = 1, \dots, s
$.

Let $\bar{y}_k := \frac{1}{s} \sum_{\ell=1}^{s} y_{k,\ell}  \text{ and define }
Z_k := d \left( \bar{y}_k^2 - \frac{1}{s} \right)
$.
}

Define $\bar{Z} := \frac{1}{K} \sum_{k=1}^{K} Z_k$ and
output $\hat{r} :=
\begin{cases}
0 & \text{if } \bar{Z} <0 \\
\sqrt{\bar{Z}} & \text{otherwise}
\end{cases}$
\end{algorithm2e}



\noindent
We now begin with some high-level analysis. Recall $x^{(k)} = \frac{(\varepsilon_1, \ldots, \varepsilon_d)}{\sqrt{d}}$ where $\varepsilon_i\overset{iid}{\sim} \{ \pm 1 \}$ and $ \mu_k = \langle x^{(k)}, \theta \rangle$. Then we have the following:
\[
\mathbb{E}[\mu^2_k] = \theta^\top \mathbb{E}\left[  x^{(k)} (x^{(k)})^\top\right] \theta 
= \theta^\top \left( \frac{1}{d} I_d \right) \theta = \frac{r^2}{d}
.\]

\noindent
Now conditioning on $x^{(k)}$ (hence $\mu_k$), we have
\[
\bar{y}_k \sim \mathcal{N}\left(\mu_k, \frac{1}{s}\right)
\Rightarrow \mathbb{E}[\bar{y}_k^2 \mid \mu_k] = \mu_k^2 + \frac{1}{s}\Rightarrow \mathbb{E}[Z_k \mid \mu_k] = d \mu_k^2.
\]

\noindent
We will now aim to bound
\[
\bar{Z} - r^2 
= \underbrace{\left( \frac{1}{K} \sum_{k=1}^{K} d\mu_k^2 - r^2 \right)}_{\text{Term 1}}
+ \underbrace{\left( \frac{1}{K} \sum_{k=1}^{K} \left( Z_k - d\mu_k^2 \right) \right)}_{\text{Term 2}}.
\]

\noindent
\textbf{Bounding Term 1:}
Let $X_k := d \mu_k^2 - r^2$.  
With the help of Hanson–Wright inequality, we can show that $X_k$ is sub-exponential and show the following for some absolute constant $c$:
\begin{equation}\label{main-eq:bounding-term-1}
    \Pr\left( \left| \frac{1}{K} \sum_{k=1}^{K} X_k \right| > t \right) 
\leq 2 \exp\left( -c\cdot K \min\left( \frac{t^2}{r^4}, \frac{t}{r^2} \right) \right).
\end{equation}

\noindent
Choosing $t = \frac{r \varepsilon}{4}$, and $K = c_1 r_0^2 \varepsilon^{-2} \log \frac{4}{\delta}$ for some large constant $c_1$,  
we get:
\[
\left| \frac{1}{K} \sum_{k=1}^{K} d\mu_k^2 - r^2 \right| 
\leq \frac{r \varepsilon}{4} \quad \text{with probability } \geq 1 - \frac{\delta}{4}.
\]

\noindent
\textbf{Bounding Term 2:}  Let $W_k := Z_k - d\mu_k^2 = d\left( \bar{y}_k^2 - \frac{1}{s} - \mu_k^2 \right)$. Conditioning on $\mu_k$, we have:
\[
\sqrt{d} \cdot \bar{y}_k \mid \mu_k \sim \mathcal{N}(\sqrt{d}\cdot\mu_k, \sigma^2), \quad \text{with } \sigma^2 = \frac{d}{s}.
\]
Hence $W_k \mid \mu_k$ is sub-exponential and therefore conditioning on $\mu_{1:K} := \mu_1, \ldots, \mu_K$, we get:
\[
\Pr\left( \left| \frac{1}{K} \sum_{k=1}^{K} W_k \right| > t \,\middle|\, \mu_{1:K} \right) 
\leq 2 \exp\left( -cK \min\left( \frac{t^2}{\bar{V}}, \frac{t}{b} \right) \right).
\]
where $
\bar{V} := 8 \left( \frac{d^2}{s^2} + \frac{d^2}{s} \cdot \frac{1}{K} \sum_{k=1}^{K} \mu_k^2 \right)
$ and $b=4d/s$. Using Eq. \eqref{main-eq:bounding-term-1}, we have $
\bar{V} \leq \tau := 8 \left( \frac{d^2}{s^2} + \frac{3dr^2}{2s} \right)$ with probability at least $1 - \frac{\delta}{8}$. Hence, substituting the values of $\tau$ and $b$, choosing $t=\frac{r\varepsilon}{4}$ and $K=c_1 r_0^2 \varepsilon^{-2}\log\!\frac{4}{\delta}$ for a sufficiently large constant $c_1$, and using $r/2<r_0\le 2r$, we obtain:
\begin{align*}
    \Pr\left( \left| \frac{1}{K} \sum_{k=1}^{K} W_k \right| > t \right)\leq \frac{\delta}{8} + \Pr\left( \left| \frac{1}{K} \sum_{k=1}^{K} W_k \right| > t \,\middle|\, \bar{V} \leq \tau \right)\leq \frac{\delta}{8}+ 2 \exp\left( -cK \min\left( \frac{t^2}{\tau}, \frac{t}{b} \right) \right)\leq \frac{\delta}{4}.
\end{align*}
\noindent
\textbf{Final guarantee:} Combining the above results, we have, with probability at least $1 - \frac{\delta}{2}$, $\left| \bar{Z} - r^2 \right| \leq \frac{r \varepsilon}{2}$. Let us now assume that $\left| \bar{Z} - r^2 \right| \leq \frac{r \varepsilon}{2}$. As $\bar{Z}\geq 0$ we have $\hat{r}=\sqrt{\bar{Z}}$, which implies the following:
    \[
    |\hat{r} - r| = \left| \sqrt{\bar{Z}} - r \right| 
    = \left| \bar{Z} - r^2 \right| \bigg/ \left( \sqrt{\bar{Z}} + r \right)
    \leq \frac{\frac{r \varepsilon}{2}}{r}
    = \frac{\varepsilon}{2} < \varepsilon.
    \]

\section{Conclusion and Future Works}
In this paper, we studied the power of adaptivity for pure exploration in linear bandits. We established matching upper and lower bounds on the minimax sample complexity of non-adaptive fixed-design algorithms.We then identified structured action sets for which adaptivity yields at most a logarithmic improvement over our non-adaptive approach, and we constructed a structured action set for which adaptive sampling attains a polynomial improvement over any non-adaptive approach. To obtain the polynomial separation, we develop an $\ell_2$-norm estimation procedure that uses $\mathcal{O}\!\left(\frac{d\log(1/\delta)}{\varepsilon^2}\right)$ samples. As a consequence, we also obtain a polynomial separation between the sample complexity of identifying an $\varepsilon$-best arm and that of estimating the optimal value $\max_{x\in\mathcal{X}}\langle x,\theta\rangle$ when $\mathcal{X}$ is the unit $\ell_2$ ball.

Our work raises several open questions. Is there an intrinsic and interpretable property of $\mathcal{X}$ that determines whether adaptivity can yield at most logarithmic improvements or instead enables polynomial improvements over non-adaptive approaches? Is the minimax sample complexity of estimating $\max_{x\in\mathcal{X}}\langle x,\theta\rangle$ always $\mathcal{O}\!\left(\frac{d\log(1/\delta)}{\varepsilon^2}\right)$ for all action sets $\mathcal{X}\subset \mathbb{R}^d$? Can our value-estimation-based adaptive approach be extended to other structured sets to obtain polynomial improvements, and can our adaptive lower-bound techniques be generalized beyond the classes considered in this paper? 

{Finally, we hope these results help connect various lines of work: for experimental design, they suggest that pure exploration in linear bandits \emph{can} exhibit polynomial advantages from adaptivity under suitable action-set geometry; for reinforcement learning through a PAC-learning lens, they suggest that the structure of the feature mapping in linear function approximation (see, e.g., \cite{bradtke1996linear,jin2020provably}) may fundamentally determine when adaptive exploration across episodes is beneficial.
}

\section*{Acknowledgements}
KJ and AM were supported in part by NSF 2141511, 2023239, and a Singapore AI Visiting Professorship award.  YX was supported by NUS A-0010008-00-00. 


\bibliography{refs}

@inproceedings{even2002pac,
  title={PAC bounds for multi-armed bandit and Markov decision processes},
  author={Even-Dar, Eyal and Mannor, Shie and Mansour, Yishay},
  booktitle={International Conference on Computational Learning Theory},
  pages={255--270},
  year={2002},
  organization={Springer}
}

@article{cesa2012combinatorial,
  title={Combinatorial bandits},
  author={Cesa-Bianchi, Nicolo and Lugosi, G{\'a}bor},
  journal={Journal of Computer and System Sciences},
  volume={78},
  number={5},
  pages={1404--1422},
  year={2012},
  publisher={Elsevier}
}

@inproceedings{cohen2017tight,
  title={Tight bounds for bandit combinatorial optimization},
  author={Cohen, Alon and Hazan, Tamir and Koren, Tomer},
  booktitle={Conference on Learning Theory},
  pages={629--642},
  year={2017},
  organization={PMLR}
}

@article{maiti2025efficient,
  title={Efficient near-optimal algorithm for online shortest paths in directed acyclic graphs with bandit feedback against adaptive adversaries},
  author={Maiti, Arnab and Fan, Zhiyuan and Jamieson, Kevin and Ratliff, Lillian J and Farina, Gabriele},
  journal={arXiv preprint arXiv:2504.00461},
  year={2025}
}

@article{fan2025universal,
  title={On the Universal Near Optimality of Hedge in Combinatorial Settings},
  author={Fan, Zhiyuan and Maiti, Arnab and Jamieson, Kevin and Ratliff, Lillian J and Farina, Gabriele},
  journal={arXiv preprint arXiv:2510.17099},
  year={2025}
}

@article{soare2014best,
  title={Best-arm identification in linear bandits},
  author={Soare, Marta and Lazaric, Alessandro and Munos, R{\'e}mi},
  journal={Advances in neural information processing systems},
  volume={27},
  year={2014}
}

@article{jedra2020optimal,
  title={Optimal best-arm identification in linear bandits},
  author={Jedra, Yassir and Proutiere, Alexandre},
  journal={Advances in Neural Information Processing Systems},
  volume={33},
  pages={10007--10017},
  year={2020}
}

@article{fiez2019sequential,
  title={Sequential experimental design for transductive linear bandits},
  author={Fiez, Tanner and Jain, Lalit and Jamieson, Kevin G and Ratliff, Lillian},
  journal={Advances in neural information processing systems},
  volume={32},
  year={2019}
}

@article{katz2020empirical,
  title={An empirical process approach to the union bound: Practical algorithms for combinatorial and linear bandits},
  author={Katz-Samuels, Julian and Jain, Lalit and Jamieson, Kevin G and others},
  journal={Advances in Neural Information Processing Systems},
  volume={33},
  pages={10371--10382},
  year={2020}
}

@inproceedings{tao2018best,
  title={Best arm identification in linear bandits with linear dimension dependency},
  author={Tao, Chao and Blanco, Sa{\'u}l and Zhou, Yuan},
  booktitle={International Conference on Machine Learning},
  pages={4877--4886},
  year={2018},
  organization={PMLR}
}

@inproceedings{degenne2020gamification,
  title={Gamification of pure exploration for linear bandits},
  author={Degenne, R{\'e}my and M{\'e}nard, Pierre and Shang, Xuedong and Valko, Michal},
  booktitle={International Conference on Machine Learning},
  pages={2432--2442},
  year={2020},
  organization={PMLR}
}

@inproceedings{xu2018fully,
  title={A fully adaptive algorithm for pure exploration in linear bandits},
  author={Xu, Liyuan and Honda, Junya and Sugiyama, Masashi},
  booktitle={International Conference on Artificial Intelligence and Statistics},
  pages={843--851},
  year={2018},
  organization={PMLR}
}

@article{karnin2016verification,
  title={Verification based solution for structured mab problems},
  author={Karnin, Zohar S},
  journal={Advances in Neural Information Processing Systems},
  volume={29},
  year={2016}
}

@inproceedings{wagenmaker2021experimental,
  title={Experimental design for regret minimization in linear bandits},
  author={Wagenmaker, Andrew and Katz-Samuels, Julian and Jamieson, Kevin},
  booktitle={International Conference on Artificial Intelligence and Statistics},
  pages={3088--3096},
  year={2021},
  organization={PMLR}
}

@book{lattimore2020bandit,
  title={Bandit algorithms},
  author={Lattimore, Tor and Szepesv{\'a}ri, Csaba},
  year={2020},
  publisher={Cambridge University Press}
}

@inproceedings{kong2020sublinear,
  title={Sublinear optimal policy value estimation in contextual bandits},
  author={Kong, Weihao and Brunskill, Emma and Valiant, Gregory},
  booktitle={International conference on artificial intelligence and statistics},
  pages={4377--4387},
  year={2020},
  organization={PMLR}
}

@article{cai2011testing,
  title={Testing Composite Hypotheses, Hermite Polynomials and Optimal Estimation of a Nonsmooth Functional},
  author={Cai, T Tony and Low, Mark G},
  journal={The Annals of Statistics},
  volume={39},
  number={244},
  pages={1012--1041},
  year={2011}
}

@article{collier2020estimation,
  title={ON ESTIMATION OF NONSMOOTH FUNCTIONALS OF SPARSE NORMAL MEANS},
  author={Collier, Olivier and Comminges, La{\"e}titia and Tsybakov, Alexandre B},
  journal={Bernoulli},
  year={2020}
}

@article{han2020estimation,
  title={On estimation of L r-norms in Gaussian white noise models},
  author={Han, Yanjun and Jiao, Jiantao and Mukherjee, Rajarshi},
  journal={Probability Theory and Related Fields},
  volume={177},
  number={3},
  pages={1243--1294},
  year={2020},
  publisher={Springer}
}

@article{cleanthous2025adaptive,
  title={Adaptive estimation of the L 2-norm of a probability density and related topics II. Upper bounds via the oracle approach},
  author={Cleanthous, G and Georgiadis, AG and Lepski, OV},
  journal={The Annals of Statistics},
  volume={53},
  number={3},
  pages={1275--1297},
  year={2025},
  publisher={Institute of Mathematical Statistics}
}

@article{chen2024assouad,
  title={Assouad, Fano, and Le Cam with Interaction: A Unifying Lower Bound Framework and Characterization for Bandit Learnability},
  author={Chen, Fan and Foster, Dylan J and Han, Yanjun and Qian, Jian and Rakhlin, Alexander and Xu, Yunbei},
  journal={Advances in Neural Information Processing Systems},
  volume={37},
  pages={75585--75641},
  year={2024}
}

@inproceedings{shamir2015complexity,
  title={On the complexity of bandit linear optimization},
  author={Shamir, Ohad},
  booktitle={Conference on Learning Theory},
  pages={1523--1551},
  year={2015},
  organization={PMLR}
}

@book{pukelsheim2006optimal,
  title={Optimal design of experiments},
  author={Pukelsheim, Friedrich},
  year={2006},
  publisher={SIAM}
}

@article{arias2012fundamental,
  title={On the fundamental limits of adaptive sensing},
  author={Arias-Castro, Ery and Candes, Emmanuel J and Davenport, Mark A},
  journal={IEEE Transactions on Information Theory},
  volume={59},
  number={1},
  pages={472--481},
  year={2012},
  publisher={IEEE}
}

@book{fisher1935design,
  title={The Design of Experiments},
  author={Fisher, Ronald A.},
  year={1935},
  publisher={}
}

@book{vershynin2018high,
  title={High-dimensional probability: An introduction with applications in data science},
  author={Vershynin, Roman},
  volume={47},
  year={2018},
  publisher={Cambridge university press}
}

@misc{gubner2021gamma,
  title={The gamma function and stirling’s formula},
  author={Gubner, John A},
  year={2021}
}

@article{ball1997elementary,
  title={An elementary introduction to modern convex geometry},
  author={Ball, Keith and others},
  journal={Flavors of geometry},
  volume={31},
  number={1-58},
  pages={26},
  year={1997}
}

@article{allen2021near,
  title={Near-optimal discrete optimization for experimental design: A regret minimization approach},
  author={Allen-Zhu, Zeyuan and Li, Yuanzhi and Singh, Aarti and Wang, Yining},
  journal={Mathematical Programming},
  volume={186},
  number={1},
  pages={439--478},
  year={2021},
  publisher={Springer}
}

@article{bradtke1996linear,
  title={Linear least-squares algorithms for temporal difference learning},
  author={Bradtke, Steven J and Barto, Andrew G},
  journal={Machine learning},
  volume={22},
  number={1},
  pages={33--57},
  year={1996},
  publisher={Springer}
}

@inproceedings{jin2020provably,
  title={Provably efficient reinforcement learning with linear function approximation},
  author={Jin, Chi and Yang, Zhuoran and Wang, Zhaoran and Jordan, Michael I},
  booktitle={Conference on learning theory},
  pages={2137--2143},
  year={2020},
  organization={PMLR}
}
\newpage
\tableofcontents
\appendix

\section{Technical Lemmas}
\begin{lemma}\label{lem:gaussian-sub-exp}
    If $X \sim \mathcal{N}(\mu, \sigma^2)$ and $Y := X^2 - \mathbb{E}[X^2]$,  
then for all $|\lambda| \leq \frac{1}{4\sigma^2}$, we have:
\[
\log \mathbb{E}(e^{\lambda Y}) \leq 4 \left( \sigma^4 + \mu^2 \sigma^2 \right) \lambda^2
\]
\end{lemma}
\textbf{Proof:} 

Let $X \sim \mathcal{N}(\mu, \sigma^2)$. For $|\lambda| \leq \frac{1}{4\sigma^2}$, we compute:

\[
\mathbb{E}\left[e^{\lambda X^2}\right] 
= \frac{1}{\sqrt{2\pi \sigma^2}} \int_{-\infty}^{\infty} \exp\left( \lambda x^2 - \frac{(x - \mu)^2}{2\sigma^2} \right) dx
\]

Now, complete the square in the exponent:
\[
\lambda x^2 - \frac{(x - \mu)^2}{2\sigma^2} 
= -\frac{1}{2\sigma^2} \left[ (1 - 2\sigma^2 \lambda)x^2 - 2\mu x + \mu^2 \right]
\]

\[
= -\frac{1 - 2\sigma^2 \lambda}{2\sigma^2} 
\left( x-\frac{\mu}{1 - 2\sigma^2 \lambda} \right)^2 
+ \frac{\mu^2 \lambda}{1 - 2\sigma^2 \lambda}
\]

Hence:
\[
\mathbb{E}\left[e^{\lambda X^2}\right] 
= \frac{e^{\mu^2 \lambda / (1 - 2\sigma^2 \lambda)}}{\sqrt{2\pi \sigma^2}} 
\int_{-\infty}^{\infty} 
\exp\left( - \frac{(1 - 2\sigma^2 \lambda)}{2\sigma^2} 
\left( x-\frac{ \mu}{1 - 2\sigma^2 \lambda} \right)^2 \right) dx
\]

This integral is Gaussian:
\[
= \frac{e^{\mu^2 \lambda / (1 - 2\sigma^2 \lambda)}}{\sqrt{2\pi \sigma^2}} 
\cdot \sqrt{ \frac{2\pi \sigma^2}{1 - 2\sigma^2 \lambda} }
\]

\[
= (1 - 2\sigma^2 \lambda)^{-1/2} \exp\left( \frac{\mu^2 \lambda}{1 - 2\sigma^2 \lambda} \right)
\]

As
\[
\lambda \mathbb{E}[X^2] = \lambda(\mu^2 + \sigma^2)
\]

\[
\log \mathbb{E}[e^{\lambda Y}] 
= -\frac{1}{2} \log(1 - 2\sigma^2 \lambda) 
+ \frac{\mu^2 \lambda}{1 - 2\sigma^2 \lambda} 
- \lambda(\mu^2 + \sigma^2)
\]

Let $\nu := 2\sigma^2 \lambda$.  
As $\lambda \leq \frac{1}{4\sigma^2}$, we have $|\nu| \leq \frac{1}{2}$.

\[
\Rightarrow \log \mathbb{E}[e^{\lambda Y}] 
= \left( -\frac{1}{2} \log(1 - \nu) - \frac{\nu}{2} \right) 
+ \mu^2 \lambda \left( \frac{1}{1 - \nu} - 1 \right)
\]

For any $\nu\in[-\tfrac{1}{2}, \tfrac{1}{2}]$, we have:

\begin{enumerate}
    \item 
    \[
    -\frac{1}{2} \log(1 - \nu) - \frac{\nu}{2} \leq \nu^2
    \]

    \item 
    \[
    \left| \frac{1}{1 - \nu} - 1 \right| 
    = \left| \frac{\nu}{1 - \nu} \right| 
    \leq \frac{|\nu|}{1 - |\nu|} 
    \leq 2|\nu|
    \]
\end{enumerate}

Hence we have:
\[
\log \mathbb{E}[e^{\lambda Y}] 
\leq \nu^2 + \mu^2 |\lambda| |2\nu| 
= 4\sigma^4 \lambda^2 + 4\mu^2 \sigma^2 \lambda^2
\]

\begin{lemma}\label{lem:sub-exp-additive}
If $U_1, \ldots, U_K$ are independent mean-zero, sub-exponential random variables with parameters $(V_1, b),\ldots,(V_K,b)$ respectively, then we have the following:
\[
 \Pr\left( \left| \frac{1}{K} \sum_{s=1}^{K} U_s \right| > t \right) 
\leq 2 \exp\left( -cK \min\left\{ \frac{t^2}{\bar{V}}, \frac{t}{b} \right\} \right)
\]
where $\bar{V} = \frac{1}{K} \sum_{k=1}^{K} V_k$ and $c>0$ is some absolute constant.
\end{lemma}
\begin{proof}
Let $U$ be a mean-zero sub-exponential random variables with parameters $(V,b)$. First, we show the following:
    \[
\Pr(|U| > t) \leq 2 \exp\left( -c \min\left\{ \frac{t^2}{V}, \frac{t}{b} \right\} \right)
\]

For any $t \geq 0$ and any $\lambda \in \left[0, \frac{1}{b} \right]$,
\begin{align*}
\Pr[U \geq t] &= \Pr\left[ e^{\lambda U} \geq e^{\lambda t} \right] \\
&\leq e^{-\lambda t} \mathbb{E}[e^{\lambda U}] \\
&\leq \exp\left( -\lambda t + \frac{\lambda^2 V}{2} \right)
\end{align*}

Now we minimize over $\lambda \in \left[0, \frac{1}{b}\right]$.  
Observe that the unconstrained minimizer is: 
\[
\lambda^* = \frac{t}{V}
\]

If $t \leq \frac{V}{b}$, then $\lambda^* \leq \frac{1}{b}$, so we get:
\[
\Pr(U \geq t) \leq \exp\left( -\frac{t^2}{2V} \right)
\]

If $t > \frac{V}{b}$, the minimum occurs at $\lambda = \frac{1}{b}$:
\[
\Pr(U \geq t) \leq \exp\left( -\frac{t}{b} + \frac{V}{2b^2} \right)
\leq \exp\left( -\frac{1}{2} \cdot \frac{t}{b} \right)
\]

If $U_1, \ldots, U_K$ are independent mean-zero, sub-exponential random variables with parameters 
\[
(V_s, b) \quad \text{(same } b \text{ for all)}, 
\]
then for any $|\lambda| \leq \frac{K}{b}$, we have:

\[
\mathbb{E} \left[ \exp\left( \frac{\lambda}{K} \sum_{s=1}^K U_s \right) \right] 
= \prod_{s=1}^K \mathbb{E} \left[ e^{\frac{\lambda}{K} U_s} \right]
\leq \exp\left( \frac{\lambda^2}{2K^2} \sum_{s=1}^K V_s \right)
\]
 \[
\Rightarrow \bar{U} := \frac{1}{K} \sum_{s=1}^K U_s 
\quad \text{is sub-exponential with parameters } 
\left( \frac{1}{K^2} \sum_{s=1}^K V_s, \, \frac{b}{K} \right)
\]

\[
\Rightarrow \Pr\left( \left| \frac{1}{K} \sum_{s=1}^{K} U_s \right| > t \right) 
\leq 2 \exp\left( -cK \min\left\{ \frac{t^2}{\bar{V}}, \frac{t}{b} \right\} \right)
\]

where 
\[
\bar{V} = \frac{1}{K} \sum_{k=1}^{K} V_k
\]
\end{proof}

\begin{lemma}[Hanson–Wright inequality]\label{lem:hanson-wright}
Let $X = (X_1, \ldots, X_n) \in \mathbb{R}^n$ be a random vector with independent components $X_i$ satisfying:
\[
\mathbb{E}[X_i] = 0, \qquad \|X_i\|_{\psi_2} \leq K
\]

Let $A$ be an $n \times n$ matrix. Then for every $t \geq 0$,
\[
\Pr\left( \left| X^\top A X - \mathbb{E}[X^\top A X] \right| > t \right) 
\leq 2 \exp\left( -c \min\left( \frac{t^2}{K^4 \|A\|_F^2}, \, \frac{t}{K^2 \|A\|} \right) \right)
\]

where
\[
\|X\|_{\psi_2} = \sup_{p \geq 1} p^{-1/2} \left( \mathbb{E}|X|^p \right)^{1/p}
\]
\end{lemma}

\begin{lemma}[Chain Rule]\label{kl-chain-rule}
        Let $f(x_1,x_2,\ldots,x_n)$ and $g(x_1,x_2,\ldots,x_n)$ be two joint PDFs for a tuple of random variables $(X_i)_{i\in[n]}$. Let the sample space be $\Omega= \mathbb{R}^{n}$. Then we have the following:
        \begin{equation*}
            KL(f,g)=\int\limits_{\omega\in \Omega}f(\omega)\left(KL(f(X_1),g(X_1))+\sum_{i=2}^n KL(f(X_i|X_{-i}=\omega_{-i}),g(X_i|X_{-i}=\omega_{-i}))\right) d\omega\;
        \end{equation*}
        where $X_{-i}=(X_1,\ldots,X_{i-1})$, $\omega_{-i}=(\omega_1,\ldots,\omega_{i-1})$.
    \end{lemma}
\begin{proof}
    \begin{align*}
            KL(f,g)&=\int\limits_{\omega\in \Omega}f(\omega)\log\left(\frac{f(\omega)}{g(\omega)}\right)\;d\omega\\
            &=\int\limits_{\omega\in \Omega}f(\omega)\log\left(\frac{f(\omega_1)\prod_{i=2}^nf(\omega_i|\omega_{-i})}{g(\omega_1)\prod_{i=2}^ng(\omega_i|\omega_{-i})}\right)\;d\omega\\
            &=\int\limits_{\omega\in \Omega}f(\omega)\left(\log\left(\frac{f(\omega_1)}{g(\omega_1)}\right)+\sum_{i=2}^n\log\left(\frac{f(\omega_i|\omega_{-i})}{g(\omega_i|\omega_{-i})}\right)\right)\;d\omega\\
            &=\int\limits_{\omega\in \Omega}f(\omega)\log\left(\frac{f(\omega_1)}{g(\omega_1)}\right)\;d\omega+\sum_{i=2}^n\int\limits_{\omega\in \Omega}f(\omega)\log\left(\frac{f(\omega_i|\omega_{-i})}{g(\omega_i|\omega_{-i})}\right)\;d\omega\\
            &=\int\limits_{\omega_1\in \mathbb{R}}f(\omega_1)\log\left(\frac{f(\omega_1)}{g(\omega_1)}\right)\;d\omega_1+\sum_{i=2}^n\int\limits_{\omega\in \mathbb{R}^{i}}f(\omega)\log\left(\frac{f(\omega_i|\omega_{-i})}{g(\omega_i|\omega_{-i})}\right)\;d\omega\\
            &=KL(f(X_1),g(X_1))+\sum_{i=2}^n\int\limits_{\omega_{-i}\in \mathbb{R}^{i-1}}f(\omega_{-i})\int\limits_{\omega_{i}\in \mathbb{R}}f(\omega_i|\omega_{-i})\log\left(\frac{f(\omega_i|\omega_{-i})}{g(\omega_i|\omega_{-i})}\right)\;d\omega_id\omega_{-i}\\
            &=KL(f(X_1),g(X_1))+\sum_{i=2}^n\int\limits_{\omega_{-i}\in \mathbb{R}^{i-1}}f(\omega_{-i})KL(f(X_i|X_{-i}=\omega_{-i}),g(X_i|X_{-i}=\omega_{-i}))\;d\omega_{-i}\\
            &=\int_{\omega\in \Omega}f(\omega)KL(f(X_1),g(X_1))\;d\omega+\sum_{i=2}^n\int\limits_{\omega\in \Omega}f(\omega)KL(f(X_i|X_{-i}=\omega_{-i}),g(X_i|X_{-i}=\omega_{-i}))\;d\omega\\
            &=\int\limits_{\omega\in \Omega}f(\omega)\left(KL(f(X_1),g(X_1))+\sum_{i=2}^n KL(f(X_i|X_{-i}=\omega_{-i}),g(X_i|X_{-i}=\omega_{-i}))\right)\;d\omega
    \end{align*}
\end{proof}

\section{Non-Adaptive Fixed Design Bounds}
\subsection{Technical Lemmas for Non-Adaptive Fixed-Design Algorithm}\label{appendix:technical-lemmas}
In this section, we state various technical lemmas and explain how they are used in Section \ref{sec:upper}.
\begin{lemma}[Carath\'eodory's theorem]\label{thm:caratheodory}
Let $S \subset \mathbb{R}^p$ and let $x \in \mathrm{conv}(S)$, where $\mathrm{conv}(S)$ denotes the convex hull of $S$.
Then there exist points $s_0,\dots,s_p \in S$ and coefficients $\alpha_0,\dots,\alpha_p \ge 0$ with
$\sum_{i=0}^p \alpha_i = 1$ such that
\[
x = \sum_{i=0}^p \alpha_i s_i.
\]
In particular, every point in the convex hull of $S$ can be expressed as a convex combination of at most $p+1$ points of $S$.
\end{lemma}
Since, $\mathbb{E}_{\eta\sim\mathcal{N}(0,I_d)}[\max_{x\in\mathcal{X}}\langle x, A(\lambda ;\mathcal{X})^{-1/2}\eta\rangle]$ depends on $\lambda$ through $A(\lambda ;\mathcal{X})=\mathbb{E}_{x\sim\lambda}[xx^\top]$, applying Carath\'eodory's theorem on $S=\{xx^\top:x\in\mathcal{X}\}$ yields a finite supported minimizer $\lambda_1$.

\begin{lemma}[G-optimal design]
    Consider a compact set $\mathcal{X}\subset \mathbb{R}^d$ such that it spans $\mathbb{R}^d$. There exists a distribution $\lambda_*$ over $\mathcal{X}$ such that $\max_{x\in\mathcal{X}}\|x\|^2_{A(\lambda_*;\mathcal{X})^{-1}}=d$ and $|\mathrm{supp}(\lambda_*)|\leq d(d+1)/2$.
\end{lemma}
Applying the above lemma yields a finite supported minimizer $\lambda_2$.

\begin{lemma}[Sudakov--Fernique inequality]\label{thm:sudakov-fernique}
Let $(X_t)_{t\in T}$ and $(Y_t)_{t\in T}$ be two mean-zero Gaussian processes. Assume that for all $t,s\in T$, we have
\[
\mathbb{E}\left[\bigl( X_t - X_s \bigr)^2\right] \le \mathbb{E}\left[\bigl( Y_t - Y_s \bigr)^2\right].
\]
Then
\[
\mathbb{E}\left[\sup_{t\in T} X_t \right]\le \mathbb{E}\left[\sup_{t\in T} Y_t\right].
\]
\end{lemma}

Define $X_x=\langle x, A(\lambda_0;\mathcal X)^{-1/2}\eta\rangle$ and $Y_x=\sqrt2\,\langle x, A(\lambda_1;\mathcal X)^{-1/2}\eta\rangle$ with $\eta\sim\mathcal N(0,I_d)$.  
Then for any $s,t\in\mathcal X$, $X_s-X_t=\langle s-t, A(\lambda_0;\mathcal X)^{-1/2}\eta\rangle$, so by $\E[\eta\eta^\top]=I_d$ we have $\E[(X_s-X_t)^2]=(s-t)^\top A(\lambda_0;\mathcal X)^{-1}(s-t)$ (and similarly $\E[(Y_s-Y_t)^2]=2(s-t)^\top A(\lambda_1;\mathcal X)^{-1}(s-t)$).  
Thus the increment condition of Sudakov--Fernique holds, yielding $\E\sup_{x\in\mathcal X}X_x\le \E\sup_{x\in\mathcal X}Y_x$.

\begin{lemma}[Linear image of a standard Gaussian]\label{lem:linear-gaussian}
Let $A\in\mathbb{R}^{m\times n}$ be a fixed matrix and let $\eta\sim\mathcal{N}(0,I_n)$. Then $A\eta$ is distributionally equivalent to a mean-zero Gaussian vector in $\mathbb{R}^m$ with covariance matrix $AA^\top$. That is $A\eta\stackrel{d}{=}\xi$ where $\xi\sim \mathcal{N}(0,AA^\top)$.
\end{lemma}
We used the above lemma to conclude that $\widehat{\theta}$ is distributionally equivalent to $\theta+(\sum_{t=1}^T x_tx_t^\top)^{-1/2}\eta$ where $\eta\sim\mathcal{N}(0,I_d)$.

\begin{lemma}[Rounding lemma \citep{katz2020empirical,allen2021near}]\label{lem:rounding-single}
Let $\mathcal{Z}\subset \mathbb{R}^d$ be finite with $m:=|\mathcal{Z}|$, and let $N\in\mathbb{N}$. Define
\[
S_N \;:=\;\Bigl\{\kappa\in \mathbb{N}^{m}:\ \sum_{x\in\mathcal{Z}} \kappa_x \le N\Bigr\},
\qquad
C_N \;:=\;\Bigl\{\pi\in [0,N]^{m}:\ \sum_{x\in\mathcal{Z}} \pi_x \le N\Bigr\}.
\]
For any weight vector $v\in \mathbb{R}_+^{m}$, define the (design) matrix
\[
A(v)\;:=\;\sum_{x\in\mathcal{Z}} v_x\, x x^\top \;\in\; \mathbb{S}_d^+.
\]
Let $F:\mathbb{S}_d^+\to \mathbb{R}$ satisfy:
(i) if $A\preceq B$ then $F(A)\ge F(B)$, and
(ii) for all $A\in\mathbb{S}_d^+$ and $t\in(0,1)$, $F(tA)=t^{-1}F(A)$.
Fix $\epsilon\in(0,1/6]$. If $m\ge N\ge 5d/\epsilon^2$, then for every $\pi\in C_N$ there exists an algorithm running in
$\widetilde{O}(m d^2)$ time that outputs $\kappa\in S_N$ such that
\[
F\!\bigl(A(\kappa)\bigr)\;\le\; (1+6\epsilon)\,F\!\bigl(A(\pi)\bigr).
\]
Moreover, for any set $V\subset \mathbb{R}^d$, the functions $F_V,G_V:\mathbb{S}_d^+\to\mathbb{R}$ defined by
\[
F_V(A)\;:=\;\mathbb{E}_{\eta\sim \mathcal{N}(0,I)}\!\left[\max_{v\in V} v^\top A^{-1/2}\eta\right],
\qquad
G_V(A)\;:=\;\max_{v\in V} v^\top A v
\]
satisfy conditions (i) and (ii).
\end{lemma}
For any $T\geq 180d$, we used the above lemma to show the existence of a fixed design $x_1,x_2,\ldots,x_T\in \mathcal{X}$ such that $\tau(A_T)\leq 2\tau(A(\lambda_0;\mathcal{X}))$ where $A_T:=\frac{1}{T}\sum_{i=1}^T x_ix_i^\top$ and $\tau(A):=\mathbb{E}_{\eta\sim\mathcal{N}(0,I_d)}\left[\max_{x\in\mathcal{X}}x^\top A^{-1/2}\eta\right]^2+2\max_{x\in\mathcal{X}}\|x\|_{A^{-1}}^2\log(2/\delta).$ In the lemma, we set $\mathcal{Z}$ as the finite support of $\lambda_0$, $N=T$, and $\pi$ as $\pi_x=T\cdot \lambda_0(x)$.

\begin{lemma}[Borell-TIS inequality]
Let $S \subset \mathbb{R}^d$ be bounded. Let $(V_s)_{s\in S}$ be a Gaussian process such that
\[
\mathbb{E}[V_s] = 0 \quad \text{for all } s \in S.
\]
Define
\[
\sigma^2 = \sup_{s\in S} \mathbb{E}[V_s^2].
\]
Then, for all $u > 0$,
\[
\mathbb{P}\!\left( \left| \sup_{s\in S} V_s - \mathbb{E}\!\sup_{s\in S} V_s \right| \ge u \right)
\le
2 \exp\!\left( -\frac{u^2}{2\sigma^2} \right).
\]
\end{lemma}
We applied Borell-TIS inequality on the gaussian process $V_x=x^\top(\widehat{\theta}-\theta)$. We then used the facts that $V_x$ is distributionally equivalent to $\langle x,(\sum_{t=1}x_tx_t^\top)^{-1/2}\eta\rangle$ where $\eta\sim \mathcal{N}(0,I_d)$ and $ \sup_{x\in \mathcal{X}} \mathbb{E}[V_x^2]=\max_{x\in\mathcal{X}}x^\top(\sum_{t=1}x_tx_t^\top)^{-1}x .$
\subsection{Adaptive Lower Bound}\label{appendix:adaptive-lower-bound}

\begin{theorem}\label{thm:main}
Assume $\mathcal{X}\subset \mathbb{R}^d$ is compact, $\mathrm{span}(\mathcal X)=\mathbb{R}^d$, and $d\geq 2$.
Consider the Gaussian linear observation model
\[
y_t=\langle x_t,\theta\rangle + \eta_t,\qquad \eta_t\sim \mathcal{N}(0,1)\ \text{i.i.d.},
\]
where the (possibly adaptive) actions satisfy $x_t\in\mathcal{X}$.
Any algorithm that outputs $\widehat x\in\mathcal{X}$ satisfying, for all $\theta\in\mathbb{R}^d$,
\[
\langle \widehat x,\theta\rangle \ge \max_{x\in \mathcal X}\langle x,\theta\rangle-\varepsilon
\quad\text{with probability at least } 1-\delta,
\]
must use
\[
n=\Omega\!\left(\frac{d\log(1/\delta)}{\varepsilon^2}\right)
\]
samples for some worst-case $\theta$ and $\delta\in(0,1/16)$.
\end{theorem}

\begin{proof}
Let
\[
B := \mathrm{conv}(\mathcal{X}\cup(-\mathcal{X})).
\]
Let $E$ be the L\"owner ellipsoid of $B$, the minimum-volume ellipsoid containing $B$.
Since $B$ is centrally symmetric and full-dimensional, $E$ is an origin-centered ellipsoid and there exists an invertible linear map $A$ such that
\[
A(E)=\mathbb{B}_2^d
\qquad\text{and hence}\qquad
A(B)\subseteq \mathbb{B}_2^d.
\]
Define
\[
K := A(B),\qquad \mathcal{X}' := A(\mathcal{X}),
\]
so $\mathcal{X}'\subseteq K\subseteq \mathbb{B}_2^d$.

We work in the transformed instance $(\mathcal{X}',\theta')$ where $\theta' := (A^{-1})^\top \theta$.
Indeed, $\langle Ax,\theta'\rangle=\langle x,\theta\rangle$, so any lower bound for $\mathcal{X}'$ implies the same $\mathcal{X}$.
Hence it suffices to prove the statement for $\mathcal{X}'$, and we now drop primes and write $\mathcal{X}\subseteq \mathbb{B}_2^d$ with
\[
K = \mathrm{conv}(\mathcal{X}\cup(-\mathcal{X}))\subseteq \mathbb{B}_2^d,
\]
such that $\mathbb{B}_2^d$ is the L\"owner ellipsoid of $K$.

Since $\mathbb{B}_2^d$ is the minimum-volume ellipsoid containing $K$,
By applying John’s theorem on the polar $K^\circ$ (see Theorem 3.1 from \cite{ball1997elementary}), there exist points $u_1,\dots,u_m \in K\cap \mathbb{S}^{d-1}$ and weights $c_1,\dots,c_m>0$ such that
\begin{equation}\label{eq:john}
\sum_{j=1}^m c_j\, u_j u_j^\top = I_d,
\qquad\text{and hence}\qquad
\sum_{j=1}^m c_j = \mathrm{tr}(I_d)=d.
\end{equation}

We now note that these $u_j$ actually belong to $\mathcal{X}\cup(-\mathcal{X})$.
Fix $j$.
Because $u_j\in \mathbb{S}^{d-1}$ and $K\subseteq \mathbb{B}_2^d$, we have
\[
\max_{x\in K}\langle x,u_j\rangle \le \max_{x\in K}\|x\|_2\|u_j\|_2= 1.
\]
On the other hand, $u_j\in K$ implies $\max_{x\in K}\langle x,u_j\rangle \ge \langle u_j,u_j\rangle =1$.
Hence
\[
\max_{x\in K}\langle x,u_j\rangle = 1.
\]
Since $K=\mathrm{conv}(\mathcal{X}\cup(-\mathcal{X}))$ and the objective is linear, the maximum over $K$ equals the maximum over $\mathcal{X}\cup(-\mathcal{X})$.
Thus there exists $v\in \mathcal{X}\cup(-\mathcal{X})$ such that $\langle v,u_j\rangle=1$.
By Cauchy-Schwarz, $\langle v,u_j\rangle \le \|v\|_2\|u_j\|_2 \le 1$, so equality forces $v=u_j$.
Therefore $u_j\in \mathcal{X}\cup(-\mathcal{X})$.

Define
\[
v_j :=
\begin{cases}
u_j, & u_j\in \mathcal{X},\\
-\,u_j, & u_j\in -\mathcal{X},
\end{cases}
\]
so that $v_j\in \mathcal{X}\cap \mathbb{S}^{d-1}$ for every $j$.
Since $v_j v_j^\top = u_j u_j^\top$, the same weights $c_1,\ldots,c_m$ satisfy
\begin{equation}\label{eq:johnX}
\sum_{j=1}^m c_j\, v_j v_j^\top = I_d,
\qquad\text{and}\qquad
\sum_{j=1}^m c_j = d.
\end{equation}

Let $\mathcal{D}$ be the distribution on $[m]$ such that $\mathbb{P}_{J\sim \mathcal{D}}(J=j)=c_j/d$.
Then \eqref{eq:johnX} implies
\begin{equation}\label{eq:cov}
\mathbb{E}[v_J v_J^\top] = \frac{1}{d}I_d.
\end{equation}

We first prove the following lemma.
\begin{lemma}\label{lem:baseline}
Any algorithm that outputs $\widehat x\in\mathcal{X}\subseteq \mathbb{B}_2^d$ satisfying, for all $\theta\in\mathbb{R}^d$,
\[
\langle \widehat x,\theta\rangle \ge \max_{x\in \mathcal X}\langle x,\theta\rangle-\varepsilon
\quad\text{with probability at least } 1-\delta,
\]
must use
\[
n \ \ge\ \frac{2-\sqrt{2}}{32\,\varepsilon^2}\,\log\!\Big(\frac{1}{4\delta}\Big)
\ =\ \Omega\!\left(\frac{\log(1/\delta)}{\varepsilon^2}\right)
\]
samples for some worst-case $\theta$ and $\delta\in(0,1/16)$.
\end{lemma}
\begin{proof}
Recall that $J\sim\mathcal{D}$ is a random index with $\mathbb{P}_{J\sim \mathcal{D}}(J=j)=c_j/d$.
Hence, we have:
\[
\mathbb E[v_J v_J^\top] = \frac{1}{d}I_d.
\]
Fix $v_1$. Then, we have:
\[
\mathbb E\big[(v_1^\top v_J)^2\big]
= v_1^\top \mathbb E[v_J v_J^\top] v_1
= \frac{1}{d}\|v_1\|_2^2
= \frac{1}{d}.
\]
Hence there exists $k\in[m]$ such that $(v_1^\top v_k)^2\le 1/d$, and thus
$v_1^\top v_k \le 1/\sqrt{d}\le 1/\sqrt{2}$ (since $d\ge 2$).
Therefore
\begin{equation}\label{eq:sep_const}
\|v_1-v_k\|_2^2
=2-2v_1^\top v_k
\ge 2-\sqrt{2}.
\end{equation}
Set $a:=v_1$, $b:=v_k$, and denote $\rho := \|a-b\|_2$. Then \eqref{eq:sep_const} gives $\rho^2\ge 2-\sqrt{2}$.

\noindent
Let $u := (a-b)/\|a-b\|_2 \in \mathbb S^{d-1}$ so that $\langle a,u\rangle-\langle b,u\rangle=\rho$.
Define
\[
\theta_+ := \frac{4\varepsilon}{\rho}\,u,
\qquad
\theta_- := -\frac{4\varepsilon}{\rho}\,u.
\]

Consider an algorithm that uses $n$ samples and outputs $\widehat x$ under each of the hypothesis above while satisfying the conditions in the theorem statement.
Let $p_{\max}:=\max_{x\in\mathcal X}\langle x,u\rangle$ and $p_{\min}:=\min_{x\in\mathcal X}\langle x,u\rangle$.
Note that $p_{\max}\ge \langle a,u\rangle$ and $p_{\min}\le \langle b,u\rangle$.

\noindent
Under $\theta_+$, the optimal value is $\max_{x\in\mathcal X}\langle x,\theta_+\rangle=(4\varepsilon/\rho)\,p_{\max}$.
If $\widehat x$ is $\varepsilon$-optimal for $\theta_+$, then
\[
\frac{4\varepsilon}{\rho}\langle \widehat x,u\rangle
=\langle \widehat x,\theta_+\rangle
\ge \frac{4\varepsilon}{\rho}p_{\max}-\varepsilon
\quad\Rightarrow\quad
\langle \widehat x,u\rangle \ge p_{\max}-\frac{\rho}{4}
\ge \langle a,u\rangle-\frac{\rho}{4}.
\]
Similarly under $\theta_-$, the optimal value is $(4\varepsilon/\rho)(-p_{\min})$ and $\varepsilon$-optimality of $\widehat x$ implies
\[
-\frac{4\varepsilon}{\rho}\langle \widehat x,u\rangle
=\langle \widehat x,\theta_-\rangle
\ge \frac{4\varepsilon}{\rho}(-p_{\min})-\varepsilon
\quad\Rightarrow\quad
\langle \widehat x,u\rangle \le p_{\min}+\frac{\rho}{4}
\le \langle b,u\rangle+\frac{\rho}{4}.
\]
Let
\[
s := \frac{\langle a,u\rangle+\langle b,u\rangle}{2}.
\]
Since $\langle a,u\rangle-\langle b,u\rangle=\rho$, the two implications based on $\varepsilon$-optimality of $\widehat x$ become:
\[
\text{Under }\theta_+:\ \langle \widehat x,u\rangle \ge s+\frac{\rho}{4},
\qquad
\text{Under }\theta_-:\ \langle \widehat x,u\rangle \le s-\frac{\rho}{4},
\]
each with probability at least $1-\delta$.
Therefore the decision rule
\[
\widehat H=
\begin{cases}
\theta_+, & \text{if }\langle \widehat x,u\rangle \ge s,\\
\theta_-, & \text{otherwise},
\end{cases}
\]
distinguishes $\theta_+$ from $\theta_-$ with error at most $\delta$ under each hypothesis.

\noindent
Let $\mathbb P_+$ and $\mathbb P_-$ be the probability laws under $\theta_+$ and $\theta_-$.
By Bretagnolle--Huber inequality, we get
\[
\mathbb P_+(\widehat H=\theta_-)+\mathbb P_-(\widehat H=\theta_+)
\ge \tfrac12 \exp\!\big(-D_{\mathrm{KL}}(\mathbb P_+\|\mathbb P_-)\big).
\]
Since each error is at most $\delta$, the left-hand side is at most $2\delta$, hence
\[
D_{\mathrm{KL}}(\mathbb P_+\|\mathbb P_-)\ge \log\!\Big(\frac{1}{4\delta}\Big).
\]
For Gaussian noise and adaptive actions $x_t\in\mathcal X\subseteq \mathbb B_2^d$,
\begin{align*}
D_{\mathrm{KL}}(\mathbb P_+\|\mathbb P_-)
&= \frac12\sum_{t=1}^n \mathbb E_+\!\big[\langle x_t,\theta_+-\theta_-\rangle^2\big]
= \frac12\sum_{t=1}^n \mathbb E_+\!\Big[\Big\langle x_t,\frac{8\varepsilon}{\rho}u\Big\rangle^2\Big] \\
&\le \frac12\sum_{t=1}^n \Big(\frac{8\varepsilon}{\rho}\Big)^2
= \frac{32n\varepsilon^2}{\rho^2},
\end{align*}
using $|\langle x_t,u\rangle|\le \|x_t\|_2\|u\|_2\le 1$.
Combining the last two inequalities, we have
\[
n \ge \frac{\rho^2}{32\varepsilon^2}\log\!\Big(\frac{1}{4\delta}\Big).
\]
Finally, $\rho^2\ge 2-\sqrt2$ by \eqref{eq:sep_const} yields
\[
n \ge \frac{2-\sqrt2}{32\varepsilon^2}\log\!\Big(\frac{1}{4\delta}\Big).
\]
\end{proof}

Define the alternatives
\[
\theta^{(j)} := 3\varepsilon\, v_j,\qquad j=1,\dots,m,
\]
and consider the hypothesis test
\[
H_0:\ \theta = 0,
\qquad
H_1:\ \theta=\theta^{(J)} \text{ where } J \sim \mathcal{D}.
\]
Let $\mathbb{P}_0$ be the law under $H_0$, $\mathbb{P}_1$ the law under $H_1$, and $\mathbb{P}_{\theta^{(j)}}$ the law under $\theta=\theta^{(j)}$.
For Gaussian noise, convexity of KL and the chain rule yield, for any (possibly adaptive) choice of actions $x_t\in\mathcal{X}$,
\[
D_{\mathrm{KL}}(\mathbb{P}_0\|\mathbb{P}_1)
\le \sum_{j=1}^m \frac{c_j}{d}\,D_{\mathrm{KL}}(\mathbb{P}_0\|\mathbb{P}_{\theta^{(j)}})
= \frac12\sum_{t=1}^N \mathbb{E}\!\Big[\big\langle x_t,\theta^{(J)}\big\rangle^2\Big],
\]
where $N$ is the total number of samples used by the algorithm solving the hypothesis testing problem, and the expectation is under $\mathbb{P}_0$ and the distribution $\mathcal{D}$.
Using \eqref{eq:cov} and $\theta^{(J)}=3\varepsilon v_J$,
\begin{align*}
\mathbb{E}\!\Big[\big\langle x_t,\theta^{(J)}\big\rangle^2\Big]
&= 9\varepsilon^2\, \mathbb{E}\big[\langle x_t,v_J\rangle^2\big]
= 9\varepsilon^2\, x_t^\top \mathbb{E}[v_J v_J^\top] x_t \\
&= 9\varepsilon^2\, x_t^\top\Big(\frac{1}{d}I_d\Big)x_t
= \frac{9\varepsilon^2}{d}\|x_t\|_2^2
\le \frac{9\varepsilon^2}{d}.
\end{align*}
Therefore, we have
\begin{equation}\label{eq:KLmix}
D_{\mathrm{KL}}(\mathbb{P}_0\|\mathbb{P}_1)\le \frac{9N\varepsilon^2}{2d}.
\end{equation}
By Bretagnolle--Huber inequality, we have
\[
\mathbb{P}_0(\widehat{H}=H_1)+\mathbb{P}_1(\widehat{H}=H_0)
\ge \tfrac12\exp\!\big(-D_{\mathrm{KL}}(\mathbb{P}_0\|\mathbb{P}_1)\big)
\ge \tfrac12 \exp\!\left(-\frac{9N\varepsilon^2}{2d}\right).
\]
In particular, if $N \le \frac{2d}{9\varepsilon^2}\log\!\big(\frac{1}{10\delta}\big)$,
then the sum of the two errors exceeds $4\delta$, so at least one error exceeds $2\delta$.
Hence any algorithm with error at most $2\delta$ under both $H_0$ and $H_1$ requires
\begin{equation}\label{eq:testlb}
N=\Omega\!\left(\frac{d\log(1/\delta)}{\varepsilon^2}\right).
\end{equation}

Consider an algorithm $\mathsf{Alg}$ that satisfies the conditions of our theorem statement. We now solve the above hypothesis test using $\mathsf{Alg}$.

Run $\mathsf{Alg}$ for $n$ rounds to obtain $\widehat{x}\in\mathcal{X}$.
Then take
\[
n_{\mathrm{est}} :=  \frac{8\log(2/\delta)}{\varepsilon^2}
\]
additional samples using the constant action $x_t=\widehat{x}$ and let $\widehat{v}$ be the empirical mean of these additional observations. Output $\widehat{H}=H_1$ if $\widehat{v}>\varepsilon$ and $\widehat{H}=H_0$ otherwise.

Gaussian concentration implies
\[
|\widehat{v}-\langle \widehat{x},\theta\rangle|\le \varepsilon/2
\quad\text{with probability at least } 1-\delta.
\]

Under $H_0$, $\langle \widehat{x},\theta\rangle=0$, so $\widehat{v}\le \varepsilon/2$ with probability at least $1-\delta$, hence $\widehat{H}=H_0$ with probability at least $1-\delta$.

Under $H_1$, let us condition on $J=j$.
Since $v_j\in\mathcal{X}\cap\mathbb{S}^{d-1}$,
\[
\max_{x\in\mathcal{X}}\langle x,\theta^{(j)}\rangle \ge \langle v_j,3\varepsilon v_j\rangle = 3\varepsilon.
\]
Therefore $\mathsf{Alg}$ outputs $ \widehat{x}$ such that with probability at least $1-\delta$,
\[
\langle \widehat{x},\theta^{(j)}\rangle \ge 3\varepsilon-\varepsilon = 2\varepsilon.
\]
Therefore, with probability at least $1-2\delta$, we have $\widehat{v}\ge 3\varepsilon/2>\varepsilon$ and therefore $\widehat{H}=H_1$.

Thus we solve the hypothesis testing problem with probability at least $1-2\delta$ under both $H_0$ and $H_1$ using
\[
N = n+n_{\mathrm{est}}
\]
samples.

Finally, Lemma \ref{lem:baseline} implies $n\ge c\,\log(2/\delta)/\varepsilon^2$ for an absolute constant $c>0$.
Hence $n_{\mathrm{est}}\le (8/c)\,n$ and so $N\le (1+8/c)\,n $.
Combining with \eqref{eq:testlb} yields
\[
n=\Omega\!\left(\frac{d\log(1/\delta)}{\varepsilon^2}\right).
\]
\end{proof}

\subsection{Gaussian Width Lower Bound}\label{appendix:gaussian-lower-bound}
Fix $\varepsilon\in (0,1]$, $\delta\in(0,1)$ and an $(\varepsilon,\delta)$-PAC algorithm. Let the algorithm run for $T:=\frac{H_1\log(1/\delta)+H_2}{\varepsilon^2}$ rounds and output $\hat x\in \mathcal{X}$.
Define the simple regret for any $\theta$ as
\[
R_T(\theta)\;:=\;\max_{x\in \mathcal{X}}\langle x,\theta\rangle-\langle \hat x,\theta\rangle \;\ge 0.
\]
Also define the quantity
\[
Z(\theta)\;:=\;\max_{x,y\in \mathcal{X}}\langle x-y,\theta\rangle,
\]
Since $\hat x\in \mathcal{X}$, we have $R_T(\theta)\le Z(\theta)$ for every $\theta$.

As the algorithm is $(\varepsilon,\delta)$-PAC, we have $\mathbb P(R_T(\theta)>\varepsilon)\le \delta$. Hence, taking expectation over $\theta\sim\mathcal N(0,\Sigma)$ where $\Sigma=\tau^2 A^{-1}$, we have
\[
\mathbb{E}_{\theta\sim\mathcal{N}(0,\Sigma)}[R_T(\theta)]\leq \varepsilon+\delta\cdot\mathbb{E}_{\theta\sim\mathcal{N}(0,\Sigma)}[Z(\theta)]=\varepsilon+2\tau\cdot \delta\cdot w(\mathcal{X};A)
\]

Recall that for any $T$ and any algorithm, we have:
\[
\mathbb E_{\theta}[R_T(\theta)]\geq \frac{\tau(1-\tau)}{1+\tau^2} w(\mathcal{X};\mathcal{A})
\]
Hence, using the above two inequalities and the facts that $ w(\mathcal{X};\mathcal{A})\geq w(\mathcal{X})/\sqrt{T}$ and $H_1\leq H_2$, we have the following by setting $\delta=0.1$
\[
\frac{0.12 \cdot w(\mathcal{X})}{\sqrt{T}}\leq  \varepsilon \quad \Rightarrow \quad H_2\geq \frac{0.0144}{\log(10)+1} \cdot w(\mathcal{X})^2
\]
\subsection{Singular Case of the Gaussian Width Lower Bound}\label{appendix:singular-case-lower-bound}
In this section, we prove the following result.
\begin{theorem}
    Assume $\mathcal X\subset\mathbb R^d$ is finite with $\mathrm{span}(\mathcal X)=\mathbb R^d$.
Consider $\delta\in(0,1/2)$ and a fixed design $x_1,\dots,x_T $ such that $\sum_{t=1}^T x_tx_t^\top$ is singular. Then for every (possibly randomized) non-adaptive procedure $\mathcal A$ that outputs $\hat x\in\mathcal{X}$ based on this fixed design, there exists a $\theta\in \mathbb{R}^d$ such that
\[
\mathbb{P}(\max_{x\in\mathcal{X}}\langle x-\hat x,\theta\rangle> \varepsilon)>\delta.
\]
\end{theorem}
\begin{proof}
    Since $A:=\sum_{t=1}^T x_tx_t^\top$ is singular, pick $v\neq 0$ such that $Av=0$. Then we have:
\[
0 \;=\; v^\top Av \;=\; \mathbb \sum_{t=1}^T (v^\top x_t)^2.
\]

Therefore, for all $t\in[T]$, we have:
\[
\langle x_t,v\rangle=0
\]

\medskip
Consider two reward vectors
\[
\theta_{+1}=\alpha v,\qquad \theta_{-1}=-\alpha v,
\]
with $\alpha>0$. Then for every $t$ and $s\in \{-1,+1\}$, we have
\[
y_t=\langle x_t,\theta_s\rangle+\eta_t
=s\cdot\alpha\langle x_t,v\rangle+\eta_t
=\eta_t
\]
Hence the distribution of the algorithm's output $\hat x$ is identical under
$\theta_{+1}$ and $\theta_{-1}$.

Define
\[
a:=\max_{x\in \mathcal X}\langle x,v\rangle,\qquad
b:=\min_{x\in \mathcal X}\langle x,v\rangle.
\]
We claim $a>b$. For the sake of contradiction, assume $a=b$. In this case $\langle x,v\rangle$ is constant over $\mathcal X$. However, $\langle x_t,v\rangle$ and $x_t\in \mathcal{X}$ which implies that $\langle x,v\rangle=0$ for all $x\in\mathcal X$. This implies
$v\perp \mathrm{span}(\mathcal X)$, contradicting $\mathrm{span}(\mathcal X)=\mathbb R^d$.
Thus
\[
\Delta(v)\;:=\;a-b\;>\;0.
\]

Let us define regret as $r(\hat x,\theta)=\max_{x\in\mathcal X}\langle x-\hat x,\theta\rangle$. We now compute regret under each $\theta_s$ for $s\in\{-1,+1\}$. 

\begin{itemize}
\item For $\theta_{+1}=\alpha v$,
\[
r(\hat x,\theta_{+1})
=\max_{x\in\mathcal X}\langle x-\hat x,\alpha v\rangle
=\alpha\Big(a-\langle \hat x,v\rangle\Big),
\]

\item For $\theta_{-1}=-\alpha v$,
\[
r(\hat x,\theta_{-1})
=\max_{x\in\mathcal X}\langle x-\hat x,-\alpha v\rangle
=\alpha\Big(\langle \hat x,v\rangle-b\Big),
\]
\end{itemize}

Adding the two bounds gives
\[
r(\hat x,\theta_{+1})+r(\hat x,\theta_{-1})
\;=\;\alpha(a-b)=\alpha\,\Delta(v).
\]
Therefore, if $\alpha=\frac{4\varepsilon}{\Delta(v)}$, we have
\[
\max\Big\{r(\hat x,\theta_{+1}),r(\hat x,\theta_{-1})\Big\}
\;\ge\;\frac{\alpha\,\Delta(v)}{2}=2\varepsilon.
\]
This implies that $\max_{s\in\{-1,+1\}}\mathbb{P}(\max_{x\in\mathcal{X}}\langle x-\hat x,\theta_s\rangle> \varepsilon)\geq \frac{1}{2}>\delta.$
\end{proof}

\section{Adaptive Version of Our Non-Adaptive Algorithm.}\label{appendix:adaptive-version}
In this section, we make our non-adaptive algorithm from Section \ref{sec:upper} adaptive.
Let $\mathcal{R}=(R_1,R_2,\ldots,R_d)$ be a partition of $\mathcal{X}$ into $d$ regions. Let us define 
\[
w(\lambda,\mathcal{X},\mathcal{R})=\max_{i\in[d]}\mathbb{E}_{\eta\sim\mathcal{N}(0,I_d)}\left[\max_{x\in R_i}\langle x, A(\lambda ;\mathcal{X})^{-1/2}\eta\rangle\right]
\]
Let $i_*\in[d]$ be the index such that $x_*\in R_{i_*}$.

Let $\lambda_1$ be a distribution over $\mathcal{X}$ that minimizes the expression $\mathbb{E}_{\eta\sim\mathcal{N}(0,I_d)}[\max_{x\in\mathcal{X}}\langle x, A(\lambda ;\mathcal{X})^{-1/2}\eta\rangle]$. Similarly let $\lambda_2$ be a distribution over $\mathcal{X}$ that minimizes the expression $\max_{x\in\mathcal{X}}\|x\|^2_{A(\lambda;\mathcal{X})^{-1}}$. Assume that $\lambda_1$ and $\lambda_2$ are finite supported (such minimizers always exist).  Let $\lambda_0$ be a distribution over $\mathcal{X}$ such that we sample from $\lambda_1$ with probability $1/2$ and we sample from $\lambda_2$ with probability $1/2$.

As $\frac{1}{2}A(\lambda_2;\mathcal{X})\prec A(\lambda_0;\mathcal{X})$, for all $x\in \mathcal{X}$, we have $x^\top A(\lambda_0;\mathcal{X})^{-1} x\leq 2x^\top A(\lambda_2;\mathcal{X})^{-1} x $. As $\lambda_2$ is a G-optimal design, we have $\|x\|^2_{A(\lambda_0;\mathcal{X})^{-1}}\leq 2d$. 

As $\frac{1}{2}A(\lambda_1;\mathcal{X})\prec A(\lambda_0;\mathcal{X})$, due to Sudakov-Fernique inequality, we have
\[
w(\lambda_0,\mathcal{X},\mathcal{R})\leq \sqrt{2}\cdot w(\lambda_1,\mathcal{X},\mathcal{R})\leq \sqrt{2}\cdot w(\mathcal{X}).
\]

Now consider a fixed design $x_1,x_2,\ldots,x_T\in \mathcal{X}$ such that $\tau(A_T,\mathcal{R})\leq 2\tau(A(\lambda_0;\mathcal{X}),\mathcal{R})$ where $A_T:=\frac{1}{T}\sum_{i=1}^T x_ix_i^\top$ and $\tau(A,\mathcal{R}):=\max_{i\in[d]}\mathbb{E}_{\eta\sim\mathcal{N}(0,I_d)}\left[\max_{x\in R_i}x^\top A^{-1/2}\eta\right]^2+2\max_{x\in\mathcal{X}}\|x\|_{A^{-1}}^2\log(4/\delta).$ Such a fixed design exists for any $T\geq 180d$ due to Lemma \ref{lem:rounding-single}. Due to the calculations above, we have $\tau(A_T)\leq 4 w(\lambda_1,\mathcal{X},\mathcal{R})^2+8d\log(4/\delta)$.

Let $y_t=\langle x_t,\theta\rangle+\eta_t$ denote the noisy rewards for our fixed design where $\eta_t\sim\mathcal{N}(0,1)$.
Let $\widehat{\theta}=(\sum_{t=1}^T x_tx_t^\top)^{-1}\sum_{t=1}^T x_ty_t$. Now we observe that $\widehat{\theta}$ is distributionally equivalent to $\theta+(\sum_{t=1}^T x_tx_t^\top)^{-1/2}\eta$ where $\eta\sim\mathcal{N}(0,I_d)$. This enables us to apply Borell-TIS inequality on the gaussian process $V_x=x^\top(\widehat{\theta}-\theta)$ over $R_{i_*}$ and obtain the following with probability $1-\delta/2$:
\begin{equation*}
    \Big|\max_{x\in R_{i_*}}\langle x,\widehat{\theta}-\theta\rangle\Big|\leq \frac{1}{\sqrt{T}}\cdot\mathbb{E}_{\eta\sim\mathcal{N}(0,I_d)}\left[\max_{x\in R_{i_*}}x^\top A_T^{-1/2}\eta\right]+\frac{1}{\sqrt{T}}\cdot\sqrt{2\max_{x\in\mathcal{X}}\|x\|_{A_T^{-1}}^2\log(4/\delta)}\leq \sqrt{\frac{2\tau(A_T)}{T}}
\end{equation*}

Let us compute $x^{(1)},x^{(2)},\ldots,x^{(d)}$ such that $x^{(i)}\in\arg\max_{x\in R_i}\langle x,\widehat{\theta}\rangle$. If we choose $T=1440(\frac{d}{\varepsilon^2}\log(4/\delta)+w(\lambda_1,\mathcal{X},\mathcal{R})^2/\varepsilon^2)$, then $|\max_{x\in R_{i_*}}\langle x,\widehat{\theta}-\theta\rangle|\leq \varepsilon/4$ with probability $1-\delta/2$. This implies that $\langle x^{(i_*)},\theta\rangle\geq \langle x_*,\theta\rangle-\varepsilon/2$ with probability $1-\delta/2$. 

If we run the Median Elimination algorithm from \cite{even2002pac} by treating each $x^{(i)}$ as an arm of a stochastic MAB instance with mean $\langle x^{(i)},\theta\rangle$, we get an index $\hat i$ after additional $\left(\frac{d\log(1/\delta)}{\varepsilon^2}\right)$ samples such that with probability $1-\delta/2$, we have $\max_{i\in[d]}\langle x^{(i)}-x^{(\hat i)},\theta\rangle\leq \varepsilon/2$. We output $x^{(\hat i)}$ as the candidate best arm for $\mathcal{X}$ and due to union bound, we have $\langle x_*-x^{(\hat i)},\theta\rangle\leq \varepsilon$.

In certain cases, if the partition is chosen appropriately, our adaptive algorithm can improve the sample complexity by logarithmic factors. For instance, consider a finite set $\mathcal{X}\subset R^d$. For simplicity of presentation, let us assume that $|\mathcal{X}|$ is multiple of $d$. Let $\mathcal{R}=(R_1,\ldots,R_d)$ be a partition of $\mathcal{X}$ such that for all $i\in[d]$, we have $|R_i|=|\mathcal{X}|/d$. Using Proposition \ref{prop:partition-gaussian-width}, we can show that $w(\lambda_1,\mathcal{X},\mathcal{R})\leq O(\sqrt{d\log(|\mathcal{X}|/d)})$. This implies that our adaptive algorithm has an improved sample complexity of $O\left(\frac{d\log((|\mathcal{X}|/d)/\delta )}{\varepsilon^2}\right)$.

\section{$\ell_2$ Norm Estimation Algorithm}\label{appendix-l2-norm-estimation}
Let $\mathcal{X}$ denote the unit $\ell_2$ ball in $\mathbb{R}^d$, and let $\theta\in\mathbb{R}^d$ be the reward vector. Our goal is to estimate $r:=\|\theta\|_2$ using only samples obtained by querying points in $\mathcal{X}$. In Appendix~\ref{sec:ball-additive-estimate}, we analyze Algorithm~\ref{alg:ball-additive}. In Appendix~\ref{sec:ball-mult-estimate}, we study how to estimate $\|\theta\|_2$ up to a constant multiplicative factor. In Appendix~\ref{sec:ball-large-regime}, we consider the regime where $\|\theta\|_2\geq \sqrt{d}$. Finally, in Appendix~\ref{sec:ball-meta-algo}, we analyze a meta-algorithm that handles all regimes and provides a universal guarantee.

\subsection{Estimating with the Help of a Multiplicative Estimate $r_0$}\label{sec:ball-additive-estimate}
Begin by defining $r:=||\theta||_2$. We now assume that we are given a value $r_0$ such that $\varepsilon<r_0<2\sqrt{d}$ and $\frac{r}{2}< r_0\leq 2r$. In the next section, we show how to satisfy these assumptions. Now we analyse Algorithm \ref{alg:ball-additive}.

In this section we work with sub-exponential random variables. A mean zero $U$ is sub-exponential with parameters $(V, b)$ if
\[
\mathbb{E}[e^{\lambda U}] \leq \exp\left( \frac{\lambda^2 V}{2} \right) \quad \text{for } |\lambda| \leq \frac{1}{b}
\]
If $X \sim \mathcal{N}(\mu, \sigma^2)$ and $Y := X^2 - \mathbb{E}[X^2]$,  
then $Y$ is a sub-exponential with parameters $(V,b)=(8\left( \sigma^4 + \mu^2 \sigma^2 \right),4\sigma^2)$ due to Lemma \ref{lem:gaussian-sub-exp}.

If $U_1, \ldots, U_K$ are independent mean-zero, sub-exponential random variables with parameters $(V_1, b),\ldots,(V_K,b)$ respectively, then we get the following due to Lemma \ref{lem:sub-exp-additive}:
\[
 \Pr\left( \left| \frac{1}{K} \sum_{s=1}^{K} U_s \right| > t \right) 
\leq 2 \exp\left( -cK \min\left\{ \frac{t^2}{\bar{V}}, \frac{t}{b} \right\} \right)
\]
where $\bar{V} = \frac{1}{K} \sum_{k=1}^{K} V_k$ and $c>0$ is some absolute constant.

We now begin with some high-level analysis. Recall $x^{(k)} = \frac{(\varepsilon_1, \ldots, \varepsilon_d)}{\sqrt{d}}$ where $\varepsilon_i\overset{iid}{\sim} \{ \pm 1 \}$ and $ \mu_k = \langle x^{(k)}, \theta \rangle$. Then we have the following:
\[
\mathbb{E}[\mu^2_k] = \theta^\top \mathbb{E}\left[ x^{(k)} (x^{(k)})^\top \right] \theta 
= \theta^\top \left( \frac{1}{d} I_d \right) \theta = \frac{r^2}{d}
\]

Now conditioning on $x^{(k)}$ (hence $\mu_k$), we have
\[
\bar{y}_k \sim \mathcal{N}\left(\mu_k, \frac{1}{s}\right)
\Rightarrow \mathbb{E}[\bar{y}_k^2 \mid \mu_k] = \mu_k^2 + \frac{1}{s}\Rightarrow \mathbb{E}[Z_k \mid \mu_k] = d \mu_k^2
\]

We will now aim to bound
\[
\bar{Z} - r^2 
= \underbrace{\left( \frac{1}{K} \sum_{k=1}^{K} d\mu_k^2 - r^2 \right)}_{\text{Term 1}}
+ \underbrace{\left( \frac{1}{K} \sum_{k=1}^{K} \left( Z_k - d\mu_k^2 \right) \right)}_{\text{Term 2}}
\]

\textbf{Bounding Term 1:}
Note that if $r=0$, then  $\text{Term I}=d \mu_k^2 - r^2=0$. Hence, the non-trivial case is when $r>0$ which we focus on below.
Let $X_k := d \mu_k^2 - r^2$.  
We now show that $X_k$ is sub-exponential with parameters
\[
(V, b) = \left( C r^4, \, C r^2 \right)
\]
for some absolute constant $C$ with the help of Hanson–Wright inequality (see Lemma \ref{lem:hanson-wright}).

Recall that
\[
\mu_k = \left\langle \frac{1}{\sqrt{d}} \varepsilon^{(k)}, \theta \right\rangle,
\qquad \varepsilon_i^{(k)} \overset{iid}{\sim} \{ \pm 1 \}
\]

Now observe that:
\[
d \mu_k^2 = \left( \varepsilon^{(k)} \right)^\top A \varepsilon^{(k)}, \quad \text{where } A := \theta \theta^\top .
\]

As $\mathbb{E}[d \mu_k^2] = r^2$, we apply the Hanson–Wright inequality (we apply Lemma \ref{lem:hanson-wright} with $K = 1$, as Rademacher variables have unit $\psi_2$-norm) to get:
\[
\Pr\left( |X_k| > t \right) 
\leq 2 \exp\left( -c \min\left( \frac{t^2}{\|A\|_F^2}, \frac{t}{\|A\|} \right) \right)
\]

As $A = \theta \theta^\top$, we have
\[
\|A\|_F^2 = \sum_i \theta_i^4 + 2 \sum_{i < j} \theta_i^2 \theta_j^2 
= \left( \sum_i \theta_i^2 \right)^2 = \|\theta\|_2^4 = r^4 
\quad 
\]

Note that $A$ is a rank-1 matrix, and
\[
A \theta = \theta \theta^\top \theta = \|\theta\|_2^2 \theta
\]

Hence, $\theta$ is an eigenvector of $A$ with eigenvalue $\|\theta\|_2^2$.

As the operator norm is equal to the largest eigenvalue, we have
\[
\|A\| = \|\theta\|_2^2 = r^2
\]

Hence, we have:
\[
\Pr\left( |X_k| > t \right) \leq 2 \exp\left( -c \min\left( \frac{t^2}{r^4}, \frac{t}{r^2} \right) \right)
\]

Using the above inequality and applying Lemma \ref{lem:tail-to-sub-exp}, we get that $X_k$ is sub-exponential with parameters $(C r^4, C r^2)$ for some constant $C$.

Hence, we have:
\begin{equation}\label{eq:bounding-term-1}
    \Pr\left( \left| \frac{1}{K} \sum_{k=1}^{K} X_k \right| > t \right) 
\leq 2 \exp\left( -\frac{c}{C}\cdot K \min\left( \frac{t^2}{r^4}, \frac{t}{r^2} \right) \right)
\end{equation}

Choosing $t = \frac{r \varepsilon}{4}$, and $K = c_1 r_0^2 \varepsilon^{-2} \log \frac{4}{\delta}$ for some large constant $c_1$,  
we get:
\[
\left| \frac{1}{K} \sum_{k=1}^{K} d\mu_k^2 - r^2 \right| 
\leq \frac{r \varepsilon}{4} \quad \text{with probability } \geq 1 - \frac{\delta}{4}
\]

\textbf{Bounding Term 2:}  Let $W_k := Z_k - d\mu_k^2 = d\left( \bar{y}_k^2 - \frac{1}{s} - \mu_k^2 \right)$. Conditioning on $\mu_k$, we have:
\[
\sqrt{d} \cdot \bar{y}_k \mid \mu_k \sim \mathcal{N}(\sqrt{d}\cdot\mu_k, \sigma^2), \quad \text{with } \sigma^2 = \frac{d}{s}
\]

Hence $W_k \mid \mu_k$ is sub-exponential with parameters:
\[
V(\mu_k) = 8 \left( \frac{d^2}{s^2} + \mu_k^2 \cdot \frac{d^2}{s} \right), \qquad b = \frac{4d}{s}
\]

Conditioning on $\mu_{1:K} := \mu_1, \ldots, \mu_K$, we get:
\[
\Pr\left( \left| \frac{1}{K} \sum_{k=1}^{K} W_k \right| > t \,\middle|\, \mu_{1:K} \right) 
\leq 2 \exp\left( -cK \min\left( \frac{t^2}{\bar{V}}, \frac{t}{b} \right) \right)
\]

where $
\bar{V} := \frac{1}{K} \sum_{k=1}^{K} V(\mu_k)=8 \left( \frac{d^2}{s^2} + \frac{d^2}{s} \cdot \frac{1}{K} \sum_{k=1}^{K} \mu_k^2 \right).
$ We now aim to upper bound
$
\bar{V}.
$

Recall that $X_k = d \mu_k^2 - r^2$ is sub-exponential with parameters $(Cr^4, Cr^2)$.

Hence, plugging $t=r^2/2$ into Eq. \eqref{eq:bounding-term-1} yields the following, for a sufficiently large constant $c_1$:
\[
\Pr\left( \left| \frac{1}{K} \sum_{k=1}^{K} d \mu_k^2 - r^2 \right| > \frac{r^2}{2} \right) 
\leq \frac{\delta}{8}
\]


\bigskip

Hence, with probability at least $1 - \frac{\delta}{8}$, we have $
\bar{V} \leq \tau := 8 \left( \frac{d^2}{s^2} + \frac{3dr^2}{2s} \right).
$

\begin{align*}
    \Pr\left( \left| \frac{1}{K} \sum_{k=1}^{K} W_k \right| > t \right)&= \Pr\left( \left| \frac{1}{K} \sum_{k=1}^{K} W_k \right| > t ,\, \bar{V} > \tau \right) +\Pr\left( \left| \frac{1}{K} \sum_{k=1}^{K} W_k \right| > t ,\, \bar{V} \leq \tau \right)  \\
    &\leq \Pr(\bar{V} > \tau) + \Pr\left( \left| \frac{1}{K} \sum_{k=1}^{K} W_k \right| > t \,\middle|\, \bar{V} \leq \tau \right)\cdot \Pr(\bar{V} \leq \tau)\\
    &\leq \frac{\delta}{8} + \underbrace{\Pr\left( \left| \frac{1}{K} \sum_{k=1}^{K} W_k \right| > t \,\middle|\, \bar{V} \leq \tau \right)}_{\text{Term III}}
\end{align*}

Now we have:
\begin{equation}\label{appendix-eq:bounding-term-3}
    \text{Term III}=\Pr\left( \left| \frac{1}{K} \sum_{k=1}^{K} W_k \right| > t \,\middle|\, \bar{V} \leq \tau \right) 
\leq 2 \exp\left( -cK \min\left( \frac{t^2}{\tau}, \frac{t}{b} \right) \right)
\end{equation}

As $s = \frac{c_0 d}{r_0^2}$, we have the following for some absolute constant $C'$:
\[
b = \left( \frac{4d}{s} \right) = \left( \frac{4}{c_0} \cdot {r_0^2} \right)
\quad \text{and} \quad 
\tau = 8 \left( \frac{r_0^4}{c_0^2} + \frac{3r^2 r_0^2}{2c_0} \right)
\leq C' r_0^4 
\]

Also, with 
\[
K = c_1 r_0^2 \varepsilon^{-2} \log \frac{4}{\delta}
\quad \text{,} \quad 
t = \frac{r \varepsilon}{4} \quad \text{ and }\quad\frac{r}{2}\leq r_0\leq 2r
\]
we get that the Term III is at most $\frac{\delta}{8}$ after plugging the above parameters into Eq. \eqref{appendix-eq:bounding-term-3}.

Hence, with probability at least $1 - \delta/2$, we have 
$
\left| \bar{Z} - r^2 \right| \leq \frac{r \varepsilon}{2}. 
$

Let us now assume that $\left| \bar{Z} - r^2 \right| \leq \frac{r \varepsilon}{2}$. As $\bar{Z}\geq 0$ we have $\hat{r}=\sqrt{\bar{Z}}$, which implies the following:
    \[
    |\hat{r} - r| = \left| \sqrt{\bar{Z}} - r \right| 
    = \left| \bar{Z} - r^2 \right| \bigg/ \left( \sqrt{\bar{Z}} + r \right)
    \leq \frac{\frac{r \varepsilon}{2}}{r}
    = \frac{\varepsilon}{2} < \varepsilon
    \]

\subsection{Estimating $r$ up to a Constant Factor}\label{sec:ball-mult-estimate}

\begin{algorithm2e}[H]
\caption{Adaptive multi-scale test for $r=\|\theta\|_2$}
\label{alg:multiscale-test}
\LinesNumbered
\DontPrintSemicolon
\SetKwInOut{Input}{Input}
\SetKwInOut{Output}{Output}
\SetKwFunction{Test}{Test}
\SetKwFunction{Stat}{Statistic}

\Input{$\varepsilon\in(0,1]$, $\delta\in(0,1/3)$, dimension $d$, large absolute constants $c_0,c_1>0$}
\Output{An estimate $r_0$}

\BlankLine
\SetKwProg{Fn}{Function}{}{end}
\Fn{\Stat{$t_j,\delta_j$}}{
    $s_j \leftarrow c_0 \dfrac{d}{t_j^2}$\;
    $K_j \leftarrow c_1 \log\!\left(\dfrac{1}{\delta_j}\right)$\;
    \For{$k\leftarrow 1$ \KwTo $K_j$}{
        Draw a Rademacher unit vector
$x^{(k)} = \frac{(\varepsilon_1, \dots, \varepsilon_d)}{\sqrt{d}}
\quad \text{where } \varepsilon_i \overset{iid}{\sim} \{ \pm 1 \}$\;
        \For{$\ell\leftarrow 1$ \KwTo $s_j$}{
            Observe $y_{k,\ell} = \langle x^{(k)},\theta\rangle + \eta_{k,\ell}$ where $\eta_{k,\ell}\sim\mathcal{N}(0,1)$ i.i.d.\;
        }
        $\bar y_k \leftarrow \dfrac{1}{s_j}\sum_{\ell=1}^{s_j} y_{k,\ell}$\;
        $Z_k \leftarrow d\left(\bar y_k^2 - \dfrac{1}{s}\right)$\;
    }
    $U(t_j) \leftarrow \dfrac{1}{K_j}\sum_{k=1}^{K_j} Z_k$\;
    \KwRet{$U(t_j)$}\;
}

\BlankLine
\Fn{\Test{$t_j,\delta_j$}}{
    $U(t_j) \leftarrow$ \Stat{$t_j,\delta_j$}\;
    \eIf{$U(t_j) \ge \dfrac{3}{2}t_j^2$}{
        \KwRet{$H_1$} \tcp*{$H_1$ is the hypothesis that $r \ge 2t_j$}
    }{
        \KwRet{$H_0$} \tcp*{$H_0$ is the hypothesis that $r \le t_j$}
    }
}

\BlankLine
$j \leftarrow 0$\;
\While{$2^j\cdot\varepsilon<2\sqrt{d}$}{
    $t_j \leftarrow 2^j\cdot\varepsilon$\;
    $\delta_j \leftarrow \dfrac{\delta}{2^{j+2}}$\;
    outcome $\leftarrow$ \Test{$t_j,\delta_j$}\;
    \If{outcome $= H_0$}{
        \textbf{break}\;
    }
    $j \leftarrow j+1$\;
}
$r_0 \leftarrow t_j$\;
\KwRet{$r_0$}\;
\end{algorithm2e}

We aim to output $r_0$ such that, with high probability, it satisfies meaningful properties that depend on the value of $r$. We discuss these properties in detail toward the end of this section.

Now we analyse the algorithm. Fix $j\geq 0$. Recall that under $\texttt{Test}(t_j,\delta_j)$, we return the hypothesis  $H_1 \, (r \geq 2t_j)$ if $U(t_j) \geq \frac{3}{2} t_j^2$ otherwise we return the hypothesis $H_0\, (r \leq t_j)$. 

We now compute the probability of error under the hypothesis $H_0\, (r \leq t_j)$. Error under the hypothesis $H_0$ ($r \leq t_j$) means:
\[
U(t_j) \geq \frac{3}{2} t_j^2 
\quad \text{whereas} \quad 
\mathbb{E}[U(t_j)] = r^2 \leq t_j^2
\]

This requires an upward deviation of at least:
\[
U(t_j) - \mathbb{E}[U(t_j)] \geq \Delta_0 := \frac{3}{2} t_j^2 - r^2 \geq \frac{1}{2} t_j^2
\]

Recall:
\[
U(t_j) - r^2 = \left( \frac{1}{K_j} \sum_k d\mu_k^2 - r^2 \right) + \frac{1}{K_j} \sum_k (Z_k - d\mu_k^2)
\]

Let:
\[
a_0 := \frac{1}{8} t_j^2, \quad \nu_0 := \frac{1}{8} t_j^2
\]

Then the error under $H_0$ satisfies:
\[
\left\{ \text{error under } H_0 \right\} 
\subseteq \left\{ \left| \frac{1}{K_j} \sum_k d\mu_k^2 - r^2 \right| >a_0 \right\}
\cup \left\{ \left| \frac{1}{K_j} \sum_k (Z_k - d\mu_k^2) \right| > \nu_0 \right\}
\]

\bigskip

Recall that $X_k := d \mu_k^2 - r^2$ is sub-exponential with parameters
\[
(V, b) = \left( C r^4, \, C r^2 \right)
\]
for some absolute constant $C$. Now observe that 

Hence, we have the following for some absolute constant $c$:
\[
\Pr\left( \left| \frac{1}{K_j} \sum_{k=1}^{K_j} X_k \right| > a_0 \right) 
\leq 2 \exp\left( -cK_j \min\left( \frac{a_0^2}{r^4}, \frac{a_0}{r^2} \right) \right)
\]

Note that:
\[
\min\left\{ \frac{a_0^2}{r^4}, \frac{a_0}{r^2} \right\} =\min\left\{ \frac{t_j^4}{64r^4}, \frac{t_j^2}{8r^2} \right\} \geq \frac{1}{64}
\]

As $K_j = c_1 \log \left( \frac{1}{\delta_j} \right)$, then for a large universal constant $c_1$, we have the following:
\[
\left| \frac{1}{K_j} \sum_{k=1}^{K_j} d\mu_k^2 - r^2 \right| 
\leq a_0 \quad \text{with probability } \geq 1 - \frac{\delta_j}{4}
\]

Let $W_k := Z_k - d\mu_k^2$

Conditioning on $\mu_k$, we have:
\[
\sqrt{d} \cdot \bar{y}_k \mid \mu_k \sim \mathcal{N}(\sqrt{d}\cdot\mu_k, \sigma^2), \quad \text{with } \sigma^2 = \frac{d}{s_j}
\]

Recall $W_k \mid \mu_k$ is sub-exponential with parameters:
\[
V(\mu_k) = 8 \left( \frac{d^2}{s_j^2} + \mu_k^2 \cdot \frac{d^2}{s_j} \right), \qquad b = \frac{4d}{s_j}
\]

and therefore we had the following for some absolute constant $c$:
\[
\Pr\left( \left| \frac{1}{K_j} \sum_{k=1}^{K_j} W_k \right| > \nu_0 \right)\leq \Pr(\bar{V}>\tau) + \underbrace{2 \exp\left( -cK_j \min\left( \frac{\nu_0^2}{\tau}, \frac{\nu_0}{b} \right) \right)}_{\text{Term I}}
\]

where 
\[
\bar{V} := \frac{1}{K_j} \sum_{k=1}^{K_j} V(\mu_k),\quad \nu_0 = \frac{t_j^2}{8}, \quad b = \frac{4t_j^2}{c_0}, \quad \tau := 8 \left( \frac{d^2}{s_j^2} + \frac{3dr^2}{2s_j}\right) \leq C't_j^4
\]
for some absolute constant $C'$.

Recall that $X_k = d \mu_k^2 - r^2$ is sub-exponential with parameters $(Cr^4, Cr^2)$. As $K_j = c_1 \log \left( \frac{1}{\delta_j} \right)$, we have for some absolute constant $c$ and for a large universal constant $c_1$:
\[
\Pr\left( \left| \frac{1}{K_j} \sum_{k=1}^{K_j} d \mu_k^2 - r^2 \right| > \frac{r^2}{2} \right) 
\leq 2 \exp(-cK_j) \leq \frac{\delta_j}{4}
\]
Hence, we have $\Pr(\bar{V}>\tau)\leq\frac{\delta_j}{4}$.

As $ \frac{\nu_0}{b}=\frac{c_0}{32}, 
\frac{\nu_0^2}{\tau} \geq\frac{1}{64C'}, $ and $ K_j = c_1 \log \left( \frac{1}{\delta_j} \right)$, then for a large universal constant $c_1$, we have the following:
\[
\Pr\left( \left| \frac{1}{K_j} \sum_{k=1}^{K_j} W_k \right| > \nu_0 \right)\leq \Pr(\bar{V}>\tau) +\text{Term I}\leq \frac{\delta_j}{4}+\frac{\delta_j}{4}=\frac{\delta_j}{2}.
\]

Hence, we have the following:
\[
\Pr(\text{error under } H_0) \leq \Pr\left( \left| \frac{1}{K_j} \sum_{k=1}^{K_j} X_k \right| > a_0 \right)+\Pr\left( \left| \frac{1}{K_j} \sum_{k=1}^{K_j} W_k \right| > \nu_0 \right)\leq \frac{3}{4} \delta_j
\]

We next compute the probability of error under the hypothesis $H_1$ ($r \geq 2t_j$). Error under the hypothesis $H_1$ ($r \geq 2t_j$) means:

\[
U(t_j) \leq \frac{3}{2} t_j^2 
\quad \text{whereas} \quad 
\mathbb{E}[U(t_j)] = r^2 \geq 4t_j^2
\]

This requires a downward deviation of at least:
\[
\mathbb{E}[U(t_j)] - U(t_j) \geq \Delta_1 := r^2 - \frac{3}{2} t_j^2 \geq \frac{5}{8} r^2
\]

Recall:
\[
U(t_j) - r^2 = \left( \frac{1}{K_j} \sum_k d\mu_k^2 - r^2 \right) + \frac{1}{K_j} \sum_k (Z_k - d\mu_k^2)
\]

Let:
\[
a_1 := \frac{1}{8} r^2, \quad \nu_1 := \frac{1}{8} r^2
\]

Then the error under $H_1$ satisfies:
\[
\left\{ \text{error under } H_1 \right\} 
\subseteq \left\{ \left| \frac{1}{K_j} \sum_k d\mu_k^2 - r^2 \right| >a_1 \right\}
\cup \left\{ \left| \frac{1}{K_j} \sum_k (Z_k - d\mu_k^2) \right| > \nu_1 \right\}
\]

Recall that $X_k := d \mu_k^2 - r^2$ is sub-exponential with parameters
\[
(V, b) = \left( C r^4, \, C r^2 \right)
\]
for some absolute constant $C$. Now observe that 

Hence, we have the following for some absolute constant $c$:
\[
\Pr\left( \left| \frac{1}{K_j} \sum_{k=1}^{K_j} X_k \right| > a_1 \right) 
\leq 2 \exp\left( -cK_j \min\left( \frac{a_1^2}{r^4}, \frac{a_1}{r^2} \right) \right)
\]

Note that:
\[
\min\left\{ \frac{a_1^2}{r^4}, \frac{a_1}{r^2} \right\} =\min\left\{ \frac{r^4}{64r^4}, \frac{r^2}{8r^2} \right\} = \frac{1}{64}
\]

As $K_j = c_1 \log \left( \frac{1}{\delta_j} \right)$, then for a large universal constant $c_1$, we have the following:
\[
\left| \frac{1}{K_j} \sum_{k=1}^{K_j} d\mu_k^2 - r^2 \right| 
\leq a_1 \quad \text{with probability } \geq 1 - \frac{\delta_j}{4}
\]

Let $W_k := Z_k - d\mu_k^2$

Conditioning on $\mu_k$, we have:
\[
\sqrt{d} \cdot \bar{y}_k \mid \mu_k \sim \mathcal{N}(\sqrt{d}\cdot\mu_k, \sigma^2), \quad \text{with } \sigma^2 = \frac{d}{s_j}
\]

Recall $W_k \mid \mu_k$ is sub-exponential with parameters:
\[
V(\mu_k) = 8 \left( \frac{d^2}{s_j^2} + \mu_k^2 \cdot \frac{d^2}{s_j} \right), \qquad b = \frac{4d}{s_j}
\]

and therefore we had the following for some absolute constant $c$:
\[
\Pr\left( \left| \frac{1}{K_j} \sum_{k=1}^{K_j} W_k \right| > \nu_1 \right)\leq \Pr(\bar{V}>\tau) + \underbrace{2 \exp\left( -cK_j \min\left( \frac{\nu_1^2}{\tau}, \frac{\nu_1}{b} \right) \right)}_{\text{Term II}}
\]

where 
\[
\bar{V} := \frac{1}{K_j} \sum_{k=1}^{K_j} V(\mu_k),\quad \nu_1 = \frac{r^2}{8}, \quad b = \frac{4t_j^2}{c_0}\leq\frac{r^2}{c_0}, \quad \tau := 8 \left( \frac{d^2}{s_j^2} + \frac{3dr^2}{2s_j}\right) \leq C'r^4
\]
for some absolute constant $C'$.

Recall that $X_k = d \mu_k^2 - r^2$ is sub-exponential with parameters $(Cr^4, Cr^2)$. As $K_j = c_1 \log \left( \frac{1}{\delta_j} \right)$, we have the following for some absolute constant $c$ and for a large universal constant $c_1$:
\[
\Pr\left( \left| \frac{1}{K_j} \sum_{k=1}^{K_j} d \mu_k^2 - r^2 \right| > \frac{r^2}{2} \right). 
\leq 2 \exp(-cK_j) \leq \frac{\delta_j}{4}
\]
 Hence, we have $\Pr(\bar{V}>\tau)\leq\frac{\delta_j}{4}$.

As $ \frac{\nu_1}{b}\geq\frac{c_0}{8}, 
\frac{\nu_1^2}{\tau} \geq\frac{1}{64C'}, $ and $ K_j = c_1 \log \left( \frac{1}{\delta_j} \right)$, then for a large universal constant $c_1$, we have the following:
\[
\Pr\left( \left| \frac{1}{K_j} \sum_{k=1}^{K_j} W_k \right| > \nu_1 \right)\leq \Pr(\bar{V}>\tau) +\text{Term II}\leq \frac{\delta_j}{4}+\frac{\delta_j}{4}=\frac{\delta_j}{2}.
\]

Hence, we have:
\[
\Pr(\text{error under } H_1) \leq \frac{\delta_j}{2}+\frac{\delta_j}{4}= \frac{3}{4} \delta_j
\]

We now establish the guarantees of our algorithm. We divide it into $4$ cases.

\textbf{Case 1: $r\leq \varepsilon$:} In this case, for $j=0$, we have $r\leq t_j$ which implies that the hypothesis $H_0$ holds. Which implies with probability at least $1-\frac{3\delta}{16}$, we return $r_0=\varepsilon$.

\textbf{Case 2: $\varepsilon<r< \sqrt{d}$:} In this case, for all $j$ such that $2t_j\leq r$, $H_1$ always holds. Consider the index $j_*$ such that $t_{j_*}\leq r<2t_{j_*}$. If we return $r_0=t_{j_*}$, then $r/2<r_0\leq r$. If we instead proceed the index $j_*+1$, then $H_0$ holds and if we terminate and return $r_0=t_{j_*+1}=2t_{j_*}$, then $r<r_0\leq 2r$. Hence, with probability at least $1-\sum_{j=1}^\infty 3\delta_j/4=1-\frac{3\delta}{8}$, we return $r/2<r_0\leq 2r$.

\textbf{Case 3: $\sqrt{d}\leq r<4\sqrt{d}$:} In this case, for all $j$ such that $2t_j\leq r$, $H_1$ always holds. If there exists an index $j_*$ such that $t_{j_*}<2\sqrt{d}$ and $t_{j_*}\leq r<2t_{j_*}$ and if we return $r_0=t_{j_*}$, then $r/2<r_0\leq r$. If we instead proceed to the index $j_*+1$, then $H_0$ holds and if we terminate and return $r_0=t_{j_*+1}=2t_{j_*}$, then $r<r_0\leq 2r$. Hence, in this scenario, with probability at least $1-\sum_{j=1}^\infty 3\delta_j/4=1-\frac{3\delta}{8}$, we return $r/2<r_0\leq 2r$.

On the other hand, if there is no index $j_*$ such that $t_{j_*}<2\sqrt{d}$ and $t_{j_*}\leq r<2t_{j_*}$, then $H_1$ always holds for all $j$ such that $2^j\varepsilon<2\sqrt{d}$. This implies that with probability at least $1-\sum_{j=1}^\infty 3\delta_j/4=1-\frac{3\delta}{8}$, we return $r_0\geq 2\sqrt{d}$.

\textbf{Case 4: $r\geq 4\sqrt{d}$:} In this case, for all $j$ such that $2^j\varepsilon<2\sqrt{d}$, $H_1$ always holds. This implies that with probability at least $1-\sum_{j=1}^\infty 3\delta_j/4=1-\frac{3\delta}{8}$, we return $r_0\geq 2\sqrt{d}$.

Now we establish the sample complexity of our algorithm. The sample complexity is upper bounded as
\[
\sum_{j=0}^{\lceil\log_2(2\sqrt{d}/\varepsilon)\rceil}s_jt_j\leq\mathcal{O}\left(\sum_{j=0}^\infty(d/\varepsilon)^22^{-j}\log(2^{j+2}/\delta)\right)=\mathcal{O}\left(\frac{d\log(1/\delta)}{\varepsilon^2}\right).
\]


\subsection{Estimation in the Large-Norm Regime}\label{sec:ball-large-regime}
Recall that observe
\[
y_t \;=\; \langle x_t,\theta\rangle + \eta_t, \qquad t = 1,\dots,n,
\]
where $\theta \in \mathbb{R}^d$ and the noise variables are i.i.d.\ $\eta_t \sim \mathcal{N}(0,1)$. Let $r := \|\theta\|_2$. In this section we assume the ``large norm'' condition
\begin{equation}
r^2 \;\ge\; d.
\label{eq:large-norm-assumption}
\end{equation}

Our goal is to construct an estimator $\hat r$ such that, for suitable absolute
constant $C > 0$,
\[
n \;\ge\; C\,\frac{d}{\varepsilon^2}\log\frac{1}{\delta}
\quad\text{ and }\quad
\Pr\big(|\hat r - r| > \varepsilon\big) \le \delta
\]
for all $\theta$ satisfying \eqref{eq:large-norm-assumption}. We first describe out algorithm below.

\begin{algorithm2e}[H]
\caption{Algorithm for Large-Norm Regime}
\label{alg:large-regime}
\DontPrintSemicolon
\SetKwInOut{Input}{Input}
\SetKwInOut{Output}{Output}
\LinesNumbered
\Input{dimension $d$, sample size $n$}
\BlankLine
\For{$t \leftarrow 1$ \KwTo $n$}{
    Draw a Rademacher unit vector $x_t = \frac{(\varepsilon_1, \dots, \varepsilon_d)}{\sqrt{d}}\quad \text{where } \varepsilon_i \overset{iid}{\sim} \{ \pm 1 \}$.\;
    Observe $y_t$ from the model $y_t = \langle x_t,\theta\rangle + \eta_t$ with $\eta_t\sim\mathcal{N}(0,1)$ i.i.d.\;
}
Form $X \in \mathbb{R}^{n\times d}$ whose $t$-th row is $x_t^\top$, and $y \leftarrow (y_1,\dots,y_n)^\top$\;

\If{$X^\top X$ is invertible}{
    $\hat\theta \leftarrow (X^\top X)^{-1} X^\top y$\;
    $\Sigma \leftarrow (X^\top X)^{-1}$\;
}
\Else{
$\hat r\gets \sqrt{d}$\;
\KwRet{$\hat r$}\;
}

$\hat r \leftarrow \sqrt{\max\{\|\hat\theta\|_2^2 - \operatorname{tr}(\Sigma),0\}}$\;

\KwRet{$\hat r$}\;
\end{algorithm2e}

We now start analysing our algorithm. Let us consider the case when $X^\top X$ is invertible. We later show that this holds with high probability. Recall that the least-squares estimator is defined as:
\[
\hat\theta := (X^\top X)^{-1} X^\top y.
\]

Define
\[
\Sigma := (X^\top X)^{-1}, \qquad \Delta := \hat\theta - \theta = \Sigma X^\top \eta.
\]

Conditioning on $X$ and using that $\eta \sim \mathcal{N}(0,I_n)$, we have
\begin{equation}
\Delta \mid X \sim \mathcal{N}(0,\Sigma).
\label{eq:Delta-distribution}
\end{equation}

Conditioning on $X$ and using \eqref{eq:Delta-distribution}, we get
\[
\mathbb{E}[\Delta \mid X] = 0,
\qquad
\mathbb{E}[\|\Delta\|_2^2 \mid X] = \operatorname{tr}(\Sigma).
\]

We wish to estimate $r^2 = \|\theta\|_2^2$. We first have the following:
\[
\|\hat\theta\|_2^2
= \|\theta + \Delta\|_2^2
= \|\theta\|_2^2 + 2\theta^\top\Delta + \|\Delta\|_2^2
= r^2 + 2\theta^\top\Delta + \|\Delta\|_2^2.
\]

Define the debiased estimator
\begin{equation}
\widehat{R} := \|\hat\theta\|_2^2 - \operatorname{tr}(\Sigma).
\label{eq:r2-hat-def}
\end{equation}

Then, we have
\[
\mathbb{E}[\widehat{R} \mid X] = r^2,
\qquad
\mathbb{E}[\widehat{R}] = r^2.
\]

Let
\begin{equation}
Z := \widehat{R} - r^2
= 2\theta^\top\Delta + \big(\|\Delta\|_2^2 - \operatorname{tr}(\Sigma)\big).
\label{eq:Z-def}
\end{equation}

We will now derive a tail bound for $Z$ conditional on $X$. Let $g \sim \mathcal{N}(0,I_d)$, and observe that $\Delta$ is distributionally equivalent to $\Sigma^{1/2} g$. Then \eqref{eq:Z-def} becomes distributionally equivalent to
\[
\tilde Z
= 2\theta^\top\Sigma^{1/2} g
 + \big(g^\top\Sigma g - \operatorname{tr}(\Sigma)\big).
\]

Define
\[
L := 2\theta^\top\Sigma^{1/2} g,
\qquad
Q := g^\top\Sigma g - \operatorname{tr}(\Sigma),
\]

so that
\[
\tilde Z = L + Q.
\]

Conditioning on $X$, the random variable $L$ is Gaussian with mean zero and variance
\[
\operatorname{Var}(L \mid X)
= 4\|\Sigma^{1/2}\theta\|_2^2
= 4\theta^\top\Sigma\theta.
\]

Therefore for any $u>0$,
\begin{equation}
\Pr\big(|L|\ge u \,\big|\, X\big)
\le 2\exp\Big(-\frac{u^2}{8\theta^\top\Sigma\theta}\Big).
\label{eq:L-tail}
\end{equation}

We now bound the term $Q$. We use the Hanson--Wright inequality in the Gaussian case:
if $g \sim \mathcal{N}(0,I_d)$ and $A \in \mathbb{R}^{d\times d}$ is symmetric, then
there exists an absolute constant $c_0 > 0$ such that for all $u>0$,
\begin{equation}
\Pr\big(|g^\top A g - \operatorname{tr}(A)| \ge u\big)
\le 2\exp\left(
  -c_0 \min\left\{
      \frac{u^2}{\|A\|_F^2},
      \frac{u}{\|A\|_{\mathrm{op}}}
  \right\}
\right).
\label{eq:HW}
\end{equation}

Applying \eqref{eq:HW} with $A = \Sigma$, and using
$\|\Sigma\|_F^2 = \operatorname{tr}(\Sigma^2)$, we obtain
\begin{equation}
\Pr\big(|Q|\ge u \,\big|\, X\big)
\le 2\exp\left(
  -c_0 \min\left\{
      \frac{u^2}{\operatorname{tr}(\Sigma^2)},
      \frac{u}{\|\Sigma\|_{\mathrm{op}}}
  \right\}
\right).
\label{eq:Q-tail}
\end{equation}

For any $t>0$,
\[
\{|\tilde Z|\ge t\}
= \{|L+Q|\ge t\}
\subseteq \Big\{|L|\ge \frac{t}{2}\Big\} \cup \Big\{|Q|\ge \frac{t}{2}\Big\},
\]

hence, by the union bound and the distributional equivalence between $\tilde Z$ and $Z$, we get
\begin{equation}
\Pr\big(|Z|\ge t \,\big|\, X\big)
\le \Pr\Big(|L|\ge \frac{t}{2} \Big| X\Big)
 +  \Pr\Big(|Q|\ge \frac{t}{2} \Big| X\Big).
\label{eq:Z-union}
\end{equation}

Using \eqref{eq:L-tail} with $u = t/2$, we get
\[
\Pr\Big(|L|\ge \frac{t}{2} \Big| X\Big)
\le 2\exp\Big(-\frac{(t/2)^2}{8\theta^\top\Sigma\theta}\Big)
= 2\exp\Big(-\frac{t^2}{32\theta^\top\Sigma\theta}\Big).
\]

Using \eqref{eq:Q-tail} with $u = t/2$, we get
\[
\Pr\Big(|Q|\ge \frac{t}{2} \Big| X\Big)
\le 2\exp\left(
  -c_0 \min\left\{
      \frac{t^2}{4\operatorname{tr}(\Sigma^2)},
      \frac{t}{2\|\Sigma\|_{\mathrm{op}}}
  \right\}
\right).
\]

Thus \eqref{eq:Z-union} implies
\begin{equation}
\Pr\big(|Z|\ge t \,\big|\, X\big)
\le
2\exp\Big(-\frac{t^2}{32\theta^\top\Sigma\theta}\Big)
+
2\exp\left(
  -c_0 \min\left\{
      \frac{t^2}{4\operatorname{tr}(\Sigma^2)},
      \frac{t}{2\|\Sigma\|_{\mathrm{op}}}
  \right\}
\right).
\label{eq:Z-two-terms}
\end{equation}

Define
\[
A_1 := \frac{t^2}{32\theta^\top\Sigma\theta},\quad
A_2 := c_0 \frac{t^2}{4\operatorname{tr}(\Sigma^2)},\quad
A_3 := c_0 \frac{t}{2\|\Sigma\|_{\mathrm{op}}}.
\]

Then \eqref{eq:Z-two-terms} can be written as
\[
\Pr(|Z|\ge t\mid X)
\le 2e^{-A_1} + 2e^{-\min\{A_2,A_3\}}
\le 4e^{-\min\{A_1,A_2,A_3\}}.
\]

For any $a,b>0$ we have
\[
\min\left\{\frac{t^2}{a},\frac{t^2}{b}\right\}
= \frac{t^2}{\max\{a,b\}}
\ge \frac{t^2}{a+b}.
\]
Applying this with $a = 32\theta^\top\Sigma\theta$ and
$b = \frac{4}{c_0}\operatorname{tr}(\Sigma^2)$, and absorbing constants, we obtain
for some absolute $c>0$:
\begin{equation}
\Pr(|Z|\ge t\mid X)
\le 4\exp\left(
  -c \min\left\{
    \frac{t^2}{4\theta^\top\Sigma\theta + 2\operatorname{tr}(\Sigma^2)},
    \frac{t}{\|\Sigma\|_{\mathrm{op}}}
  \right\}
\right).
\label{eq:Z-HW-type}
\end{equation}

Standard results on sub-Gaussian random matrices from the Section 4.7 of \cite{vershynin2018high} gives the following.
Since $\mathbb{E}[x_t x_t^\top] = I_d/d$ and the rows are i.i.d.\ sub-Gaussian, there exists
an absolute constant $C_0>1$ such that if
\begin{equation}
n \;\ge\; C_0\big(d + \log\tfrac{2}{\delta}\big),
\label{eq:n-rm-condition}
\end{equation}

then with probability at least $1-\delta/2$,
\begin{equation}
\frac{1}{2d}I_d
\;\preceq\;
\frac{X^\top X}{n}
\;\preceq\;
\frac{2}{d}I_d.
\label{eq:XTX-bounds}
\end{equation}

Multiplying \eqref{eq:XTX-bounds} by $n$ and inverting the inequalities yields
\begin{equation}
\frac{d}{2n}I_d \;\preceq\; \Sigma \;\preceq\; \frac{2d}{n}I_d.
\label{eq:Sigma-sandwich}
\end{equation}

From \eqref{eq:Sigma-sandwich} we obtain, on this ``good event'' for $X$,
\begin{align}
\|\Sigma\|_{\mathrm{op}}
&= \lambda_{\max}(\Sigma)
 \le \frac{2d}{n},
\label{eq:Sigma-op-bound}
\\
\theta^\top\Sigma\theta
&\le \|\Sigma\|_{\mathrm{op}} \|\theta\|_2^2
 \le \frac{2d}{n} r^2,
\label{eq:theta-Sigma-theta-bound}
\\
\operatorname{tr}(\Sigma^2)
&= \sum_{i=1}^d \lambda_i^2 \le d \lambda_{\max}^2
 \le d\left(\frac{2d}{n}\right)^2
 = \frac{4d^3}{n^2}.
\label{eq:Sigma2-trace-bound}
\end{align}

We now use \eqref{eq:theta-Sigma-theta-bound} and \eqref{eq:Sigma2-trace-bound} in
\eqref{eq:Z-HW-type}.
On the good event for $X$,
\begin{align*}
4\theta^\top\Sigma\theta + 2\operatorname{tr}(\Sigma^2)
&\le 4 \cdot \frac{2d}{n}r^2 + 2\cdot\frac{4d^3}{n^2} \\
&= \frac{8d}{n}r^2 + \frac{8d^3}{n^2}.
\end{align*}
Using $r^2\ge d$ and $n\ge d$, we have
\[
\frac{d^3}{n^2} \le \frac{d^2}{n} \le \frac{d}{n}r^2,
\]
hence
\[
\frac{8d^3}{n^2} \le 8\,\frac{d}{n}r^2.
\]
Therefore
\begin{equation}
4\theta^\top\Sigma\theta + 2\operatorname{tr}(\Sigma^2)
\;\le\; \frac{8d}{n}r^2 + 8\,\frac{d}{n}r^2
\;=\; 16\,\frac{d}{n}r^2.
\label{eq:denominator-bound}
\end{equation}

Let us choose
\[
t = \varepsilon r.
\]
On the good event,
\[
\min\left\{\frac{t^2}{4\theta^\top\Sigma\theta + 2\operatorname{tr}(\Sigma^2)},\frac{t}{\|\Sigma\|_{\mathrm{op}}}\right\}
\;\ge\;
\frac{\varepsilon^2 r^2}{16\,\frac{d}{n}r^2}
= \frac{n\varepsilon^2}{16d}.
\]
Using \eqref{eq:Z-HW-type} we obtain the following for some absolute constant $c'>0$:
\begin{equation}
\Pr\big(|Z|\ge \varepsilon r \mid X\big)
\le 4\exp\left(-c'\,\frac{n\varepsilon^2}{d}\right)
\quad\text{on the good event for }X.
\label{eq:Z-large-norm-cond}
\end{equation}

If we choose
\[
n \;\ge\; C_1\,\frac{d}{\varepsilon^2}\log\frac{4}{\delta}
\]

for a sufficiently large absolute constant $C_1$, then the right-hand side of
\eqref{eq:Z-large-norm-cond} is at most $\delta/2$.
Combining with \eqref{eq:n-rm-condition} and the fact that the good event for $X$
has probability at least $1-\delta/2$, we conclude that
\begin{equation}
\Pr\big(|\widehat{R} - r^2| \ge \varepsilon r\big) \le \delta.
\label{eq:r2-hat-concentration}
\end{equation}

Recall that our estimator for $r$ is
\begin{equation}
\hat r := \sqrt{\max\{\widehat{R},0\}}.
\label{eq:r-hat-def}
\end{equation}

On the event
\[
|\widehat{R} - r^2| \le \varepsilon r
\]
and assuming $\varepsilon \le r/2$ (which trivially holds for $d\geq 4$), we have
\[
\widehat{R} \ge r^2 - \varepsilon r \ge r^2 - \tfrac{r^2}{2} = \tfrac{r^2}{2} > 0,
\]
so $\hat r = \sqrt{\widehat{R}}$, and therefore
\[
|\hat r - r|
= \big|\sqrt{\widehat{R}} - \sqrt{r^2}\big|
= \frac{|\widehat{R} - r^2|}{\sqrt{\widehat{R}} + r}
\le \frac{|\widehat{R} - r^2|}{r}
\le \varepsilon.
\]

Hence
\[
\big\{|\widehat{R} - r^2| \le \varepsilon r\big\}
\subseteq
\big\{|\hat r - r| \le \varepsilon\big\}.
\]

Combining with \eqref{eq:r2-hat-concentration}, we obtain
\[
\Pr\big(|\hat r - r| > \varepsilon\big)
\le \Pr\big(|\widehat{R} - r^2| > \varepsilon r\big)
\le \delta.
\]

Finally, we needed
\[
n \;\ge\; C_0\big(d + \log\tfrac{2}{\delta}\big)
\quad\text{and}\quad
n \;\ge\; C_1\,\frac{d}{\varepsilon^2}\log\frac{4}{\delta},
\]
so for some absolute constant $C_2>0$,
\[
n \;\ge\; C_2\,\frac{d}{\varepsilon^2}\log\frac{1}{\delta}
\quad\Longrightarrow\quad
\Pr\big(|\hat r - \|\theta\|_2| > \varepsilon\big)\le \delta
\]
for all $\theta$ satisfying $r^2 \ge d$.

\subsection{Meta Algorithm for $\ell_2$-Norm Estimation}\label{sec:ball-meta-algo}
In section, we describe a meta algorithm that uses Algorithms \ref{alg:ball-additive}, \ref{alg:multiscale-test}, and \ref{alg:large-regime}.

\begin{algorithm2e}[H]
\caption{Meta-algorithm to estimate $\|\theta\|_2^2$}
\label{alg:meta-algo}
\DontPrintSemicolon
\SetKwInOut{Input}{Input}
\SetKwInOut{Output}{Output}
\LinesNumbered
\Input{dimension $d$}
\BlankLine
Run the Algorithm \ref{alg:multiscale-test} and receive the estimate $r_0$\;
\If{$r_0=\varepsilon$}{
    $\hat r\gets r_0$\;
    \KwRet{$\hat r$}\;
}
\If{$r_0\geq2\sqrt{d}$}{
    Run the Algorithm \ref{alg:large-regime} and receive the estimate $\hat r$\;
    \KwRet{$\hat r$}\;
}
\If{$\varepsilon<r_0<2\sqrt{d}$}{
    Run the Algorithm \ref{alg:ball-additive} and receive the estimate $\hat r$\;
    \KwRet{$\hat r$}\;
}
\end{algorithm2e}

Now we claim that Algorithm \ref{alg:meta-algo} takes $\mathcal{O}\left(\frac{d\log(1/\delta)}{\varepsilon^2}\right)$ samples and returns an estimate $\hat{r}$ such that with probability at least $1-\delta$ we have $|r-\hat r|\leq \varepsilon$, where $r:=\|\theta\|_2$. 

First, consider the case where $r\leq \varepsilon$. We showed in Appendix \ref{sec:ball-mult-estimate} that with probability at least $1-\delta$, we have $r_0=\varepsilon$. Hence, probability at least $1-\delta$ we have $|r-\hat r|\leq \varepsilon$.

Next, consider the case where $\varepsilon<r< \sqrt{d}$. We showed in Appendix \ref{sec:ball-mult-estimate} that with probability at least $1-\delta/2$, we have $r/2<r_0\leq 2r$ and $r_0<2\sqrt{d}$. Conditioned on this good event, if $r_0=\varepsilon$, we have $\varepsilon<r<2\varepsilon$ and therefore $|\hat r-r|\leq \varepsilon$. On the other hand, conditioned on this good event, if $r_0>\varepsilon$ Algorithm \ref{alg:ball-additive} is run and it returns $\hat r$ such that probability at least $1-\delta/2$ we have $|r-\hat r|\leq \varepsilon$. Due to union bound, we return $\hat r$ such that with probability at least $1-\delta$ we have $|r-\hat r|\leq \varepsilon$.

Next, consider the case where $\sqrt{d}\leq r<4\sqrt{d}$. Due to the analysis in Appendix \ref{sec:ball-mult-estimate}, with probability at least $1-\delta/2$ we either run Algorithm \ref{alg:ball-additive} with an estimate $r_0$ satisfying $r/2<r_0\leq 2r$ or run the Algorithm \ref{alg:large-regime}. In either case, the algorithm being run returns $\hat r$ such that probability at least $1-\delta/2$ we have $|r-\hat r|\leq \varepsilon$. Due to union bound, we return $\hat r$ such that with probability at least $1-\delta$ we have $|r-\hat r|\leq \varepsilon$.

Finally, consider the case where $r\geq 4\sqrt{d}$. Due to the analysis in Appendix \ref{sec:ball-mult-estimate}, with probability at least $1-\delta/2$ we have $r_0\geq 2\sqrt{d}$. Hence, conditioned on this good event, Algorithm \ref{alg:large-regime} is run and it returns $\hat r$ such that probability at least $1-\delta/2$ we have $|r-\hat r|\leq \varepsilon$. Due to union bound, we return $\hat r$ such that with probability at least $1-\delta$ we have $|r-\hat r|\leq \varepsilon$.

The sample complexity guarantee follows as each algorithm run in the meta-algorithms takes $\mathcal{O}\left(\frac{d\log(1/\delta)}{\varepsilon^2}\right)$ samples.

\section{Tail to Sub-Exponential Technical Lemma}
\begin{lemma}\label{lem:tail-to-sub-exp}
    Assume $U$ is mean-zero and for all $t\ge 0$ and some absolute constant $c>0$, we have:
\[
\Pr(|U|>t)\le 2\exp\!\left(-c\min\!\left(\frac{t^2}{V},\frac{t}{b}\right)\right).
\]
Then there exists an absolute constant $C>0$ such that $U$ is sub-exponential with parameters
\[
\bigl(C(V+b^2),\, Cb\bigr),
\]
i.e.
\[
\mathbb E\bigl[e^{\lambda U}\bigr]\le \exp\!\left(\frac{\lambda^2\,C(V+b^2)}{2}\right)
\quad\text{for all }|\lambda|\le \frac{1}{Cb}.
\]
\end{lemma}
\begin{proof}
First, observe that for all $t\ge 0$, we have:
\begin{equation}\label{eq:sub-exp-1}
    \Pr(|U|>t)
\le 2e^{-c t^2/V}+2e^{-c t/b}.
\end{equation}

For any integer $k\ge 2$,
\[
\mathbb E|U|^k
= k\int_0^\infty t^{k-1}\Pr(|U|>t)\,dt,
\]

so by \eqref{eq:sub-exp-1}, we have:
\[
\mathbb E|U|^k
\le 2k\int_0^\infty t^{k-1}e^{-c t^2/V}\,dt
\;+\;2k\int_0^\infty t^{k-1}e^{-c t/b}\,dt.
\]

Using the definition of gamma function, we have the following:
\[
\int_0^\infty t^{k-1}e^{-c t^2/V}\,dt
=\frac12\left(\frac{V}{c}\right)^{k/2}\Gamma\!\left(\frac{k}{2}\right),
\]
\[
\int_0^\infty t^{k-1}e^{-c t/b}\,dt
=\left(\frac{b}{c}\right)^k\Gamma(k).
\]

Hence
\begin{equation}\label{eq:sub-exp-2}
    \mathbb E|U|^k
\le k\left(\frac{V}{c}\right)^{k/2}\Gamma\!\left(\frac{k}{2}\right)
+2k\left(\frac{b}{c}\right)^k\Gamma(k).
\end{equation}

Using the properties of gamma function from \cite{gubner2021gamma}, we have the following for all $x\geq 1$:
\begin{equation}\label{eq:sub-exp-3}
    \Gamma(x)\le 3\,x^{x-\frac12}e^{-x}.
\end{equation}

Apply \eqref{eq:sub-exp-3} to $x=k/2\ge 1$ and $x=k\ge 2$:
\[
\Gamma(k/2)\le 3\left(\frac{k}{2}\right)^{\frac{k}{2}-\frac12}e^{-k/2},
\qquad
\Gamma(k)\le 3\,k^{k-\frac12}e^{-k}.
\]
Plugging into \eqref{eq:sub-exp-2} and taking $k$-th roots gives:
\begin{equation}\label{eq:sub-exp-4}
    (\mathbb E|U|^k)^{1/k}
\le \left[k\left(\frac{V}{c}\right)^{k/2}\Gamma(k/2)\right]^{1/k}
+\left[2k\left(\frac{b}{c}\right)^k\Gamma(k)\right]^{1/k}\le 2\sqrt{\frac{Vk}{c}} \;+\; 2\frac{bk}{c}
\end{equation}

From \eqref{eq:sub-exp-4} we also get, using $(x+y)^k\le 2^{k-1}(x^k+y^k)$,
\begin{equation}\label{eq:sub-exp-5}
    \mathbb E|U|^k
\le 2^{k-1}\Bigl( (2\sqrt{Vk/c})^k + (2bk/c)^k\Bigr)
\end{equation}

Since $\mathbb E U=0$,
\[
\mathbb E e^{\lambda U}
=1+\sum_{k\ge 2}\frac{\lambda^k\mathbb E[U^k]}{k!}
\le 1+\sum_{k\ge 2}\frac{|\lambda|^k\mathbb E|U|^k}{k!}.
\]

Using \eqref{eq:sub-exp-5}, we get
\[
\mathbb E e^{\lambda U}
\le 1+\frac12\sum_{k\ge 2}\frac{(4|\lambda|)^k\Bigl( (V/c)^{k/2}k^{k/2} + (b/c)^k k^k\Bigr)}{k!}
=1+S_1+S_2,
\]
where
\[
S_1:=\frac12\sum_{k\ge 2}\frac{\bigl(4|\lambda|\sqrt{V/c}\bigr)^k\,k^{k/2}}{k!},
\qquad
S_2:=\frac12\sum_{k\ge 2}\frac{\bigl(4|\lambda|b/c\bigr)^k\,k^k}{k!}.
\]

First, we bound the term $1+S_2$. Use $k!\ge (k/e)^k$, so $k^k/k!\le e^k$. Then
\[
S_2
\le \frac12\sum_{k\ge 2}\bigl(4e|\lambda|b/c\bigr)^k.
\]
If
\begin{equation}\label{eq:sub-exp-6}
    |\lambda|\le \frac{c}{8eb},
\end{equation}
then $r:=4e|\lambda|b/c\le 1/2$, so
\[
S_2\le \frac12\sum_{k\ge 2}r^k
=\frac12\cdot \frac{r^2}{1-r}
\le r^2
\le 16e^2\frac{\lambda^2 b^2}{c^2}.
\]
Hence, using $1+x\le e^x$, we get:
\begin{equation}\label{eq:sub-exp-7}
    1+S_2\le \exp\!\left(16e^2\frac{\lambda^2 b^2}{c^2}\right).
\end{equation}
Next, we bound the term $1+S_1$.
Let
\[
a:=4|\lambda|\sqrt{V/c}.
\]
For $k\ge 2$, letting $m=\lfloor k/2\rfloor$, we have:
\[
k! \ge m!\left(\frac{k}{2}\right)^{k-m}.
\]
This implies
\begin{equation}\label{eq:sub-exp-8}
    \frac{k^{k/2}}{k!}
\le \frac{k^{k/2}}{m!(k/2)^{k-m}}
\le \frac{2^{\lceil k/2\rceil}}{m!}
\le \frac{2^{k/2+1}}{m!}.
\end{equation}
Now we consider both case on $k$, one where $k$ is odd and one where $k$ is even:
\begin{itemize}
  \item $k=2m$ with $m\ge 1$: by \eqref{eq:sub-exp-8}, we have
  \[
  \frac{a^{2m}(2m)^{m}}{(2m)!}\le \frac{a^{2m}2^{m}}{m!}.
  \]
  \item $k=2m+1$ with $m\ge 1$: by \eqref{eq:sub-exp-8}, we have
  \[
  \frac{a^{2m+1}(2m+1)^{m+1/2}}{(2m+1)!}
  \le \frac{a^{2m+1}2^{m+1}}{m!}.
  \]
\end{itemize}

Therefore
\begin{equation}\label{eq:sub-exp-9}
    S_1
\le \frac12\sum_{m\ge 1}\frac{(2a^2)^m}{m!}
\;+\; a\sum_{m\ge 1}\frac{(2a^2)^m}{m!}
= \Bigl(a+\frac12\Bigr)\bigl(e^{2a^2}-1\bigr).
\end{equation}

Using the inequalities $(x-\tfrac12)^2+\tfrac14\ge 0$ and $e^x\ge 1+x$, we have the following: 
\begin{equation}\label{eq:sub-exp-10}
    a+\frac12 \le 1+a^2 \le e^{a^2}.
\end{equation}


Combining \eqref{eq:sub-exp-9} and \eqref{eq:sub-exp-10}, we get:
\[
S_1 \le e^{a^2}(e^{2a^2}-1)=e^{3a^2}-e^{a^2}\le e^{3a^2}-1,
\]
so
\begin{equation}\label{eq:sub-exp-11}
    1+S_1 \le e^{3a^2}=\exp\!\left(48\frac{\lambda^2 V}{c}\right).
\end{equation}

Since $S_1,S_2\ge 0$,
\[
1+S_1+S_2 \le (1+S_1)(1+S_2).
\]

Using \eqref{eq:sub-exp-7} and \eqref{eq:sub-exp-11}, whenever \eqref{eq:sub-exp-6} holds, we have
\[
\mathbb E e^{\lambda U}
\le \exp\!\left(48\frac{\lambda^2 V}{c}\right)\,
\exp\!\left(16e^2\frac{\lambda^2 b^2}{c^2}\right)
= \exp\!\left(\frac{\lambda^2}{2}\Bigl(\frac{96}{c}V+\frac{32e^2}{c^2}b^2\Bigr)\right).
\]

Define
\[
b' := \frac{8e}{c}\,b,
\qquad
V' := \frac{96}{c}V+\frac{32e^2}{c^2}b^2.
\]
Then (6) is exactly $|\lambda|\le 1/b'$, and we have shown
\[
\mathbb E\bigl[e^{\lambda U}\bigr] \le \exp\!\left(\frac{\lambda^2 V'}{2}\right)
\quad\text{for all }|\lambda|\le \frac{1}{b'}.
\]
So $U$ is sub-exponential with parameters $(V',b')$.

Finally, if one desires a single constant multiplying both $(V+b^2)$ and $b$, one can take
\[
C \;:=\; \max\!\left\{\frac{96}{c},\,\frac{32e^2}{c^2},\,\frac{8e}{c}\right\},
\]
so that $V' \le C(V+b^2)$ and $b'\le Cb$. This gives exactly the stated form:
\[
\mathbb E\bigl[e^{\lambda U}\bigr] \le \exp\!\left(\frac{\lambda^2\,C(V+b^2)}{2}\right)
\quad\text{for }|\lambda|\le \frac{1}{Cb}.
\]
\end{proof}
\section{Gaussian Width Properties}\label{appendix:gaussian-width-properties}

For a design distribution $\lambda$ over the set $\mathcal{X}$, let
\[
\Sigma(\lambda):=\mathbb E_{x\sim\lambda}[xx^\top],\qquad
f(\lambda):=\mathbb E_{g\sim\mathcal N(0,I_d)}\Big[\sup_{x\in\mathcal X}\langle x,\Sigma(\lambda)^{-1/2}g\rangle\Big].
\]
Equivalently, we have
\[
f(\lambda)=\mathbb E_{g\sim\mathcal N(0,I_d)}\Big[\sup_{x\in\mathcal X}\langle \Sigma(\lambda)^{-1/2}x, g\rangle\Big].
\]
Recall that $w(\mathcal{X}):=\inf_{\lambda \in \triangle_{\mathcal{X}}}f(\lambda)$. Now we prove the following propositions. 

\begin{proposition}
    $w(\mathcal{X})\leq d$.
\end{proposition}
\begin{proof}

Let $\lambda_G$ be a G-optimal design, that is a distribution that minimizes $\sup_{x\in\mathcal X} x^\top \Sigma(\lambda)^{-1}x$.
\[
\lambda_G \in \arg\min_\lambda\ \sup_{x\in\mathcal X} x^\top \Sigma(\lambda)^{-1}x.
\]

For compact $\mathcal X$ with $\mathrm{span}(\mathcal X)=\mathbb R^d$, we have the following:
\[
\sup_{x\in\mathcal X} x^\top \Sigma(\lambda_G)^{-1}x \ \le\ d.
\]

Equivalently,
\[
\sup_{x\in\mathcal X}\|\Sigma(\lambda_G)^{-1/2}x\|_2 \ \le\ \sqrt d.
\]

 For any fixed $g$, we have the following:
\[
\sup_{x\in\mathcal X}\langle x,\Sigma(\lambda_G)^{-1/2}g\rangle
=\sup_{x\in\mathcal X}\langle \Sigma(\lambda_G)^{-1/2}x, g\rangle
\le \Big(\sup_{x\in\mathcal X}\|\Sigma(\lambda_G)^{-1/2}x\|_2\Big)\,\|g\|_2
\le \sqrt d\,\|g\|_2.
\]

Take expectation:
\[
f(\lambda_G)\le \sqrt d\ \mathbb E\|g\|_2\leq \sqrt d\ \sqrt{\mathbb E\|g\|_2^2}=d.
\]

Therefore, we have the following:
\[
 w(\mathcal{X})=\inf_{\lambda\in\triangle_{\mathcal{X}}} f(\lambda) \le f(\lambda_G) \le d.\
\]
\end{proof}

\begin{proposition}
    For a finite set $\mathcal{X}$ with $|\mathcal X|=m$, we have $w(\mathcal{X})\leq O(\sqrt{d\log m})$.
\end{proposition}
\begin{proof}
Define $u_x:=\Sigma(\lambda_G)^{-1/2}x$. Then $\|u_x\|_2^2\le d$ for all $x\in\mathcal{X}$.

If $g\sim \mathcal{N}(0,I_d)$, then for each $x\in\mathcal{X}$, $Z_x:=\langle u_x,g\rangle$ is Gaussian with variance $\|u_x\|_2^2\le d$, so for $t\ge 0$,
\[
\Pr(Z_x\ge t)\le \exp\!\left(-\frac{t^2}{2d}\right).
\]

Now by applying the union bound, we get the following:
\[
\Pr\Big(\max_{x\in\mathcal X} Z_x \ge t\Big)\le m\exp\!\left(-\frac{t^2}{2d}\right).
\]

Using the fact $\mathbb E[\max_{x\in\mathcal{X}} Z_x]=\int_0^\infty \Pr(\max_{x\in\mathcal{X}} Z_x\ge t)\,dt$ and choosing $t_0=\sqrt{2d\log m}$, we get the following:
\[
\mathbb E[\max_{x\in\mathcal{X}} Z_x]\le t_0+\int_{t_0}^\infty m e^{-t^2/(2d)}\,dt\leq t_0+m\cdot \frac{d}{t_0} e^{-t_0^2/(2d)}
= t_0+ \frac{d}{t_0}\leq O(\sqrt{d\log m}).
\]

Therefore, we have the following:
\[
 w(\mathcal{X})=\inf_{\lambda\in\triangle_{\mathcal{X}}} f(\lambda) \le f(\lambda_G) \le O(\sqrt{d\log |\mathcal{X}|}).\
\]
\end{proof}

Next, recall the definition of $\lambda_1$, $\mathcal{R}$, $w(\lambda_1,\mathcal{R},\mathcal{X})$ for the finite set $\mathcal{X}$ in Appendix \ref{appendix:adaptive-version}.
\begin{proposition}\label{prop:partition-gaussian-width}
    For a finite set $\mathcal{X}$ with $|\mathcal X|=m\geq d$, we have $w(\lambda_1,\mathcal{R},\mathcal{X})\leq O(\sqrt{d\log(m/d)})$.
\end{proposition}
\begin{proof}
For simplicity of presentation, assume that $m$ is a multiple of $d$. Fix an index $i\in [d]$. Define $u_x:=\Sigma(\lambda_G)^{-1/2}x$. Then $\|u_x\|_2^2\le d$ for all $x\in\mathcal{X}$.

If $g\sim \mathcal{N}(0,I_d)$, then for each $x\in\mathcal{X}$, $Z_x:=\langle u_x,g\rangle$ is Gaussian with variance $\|u_x\|_2^2\le d$, so for $t\ge 0$,
\[
\Pr(Z_x\ge t)\le \exp\!\left(-\frac{t^2}{2d}\right).
\]

Now by applying the union bound, we get the following:
\[
\Pr\Big(\max_{x\in R_i} Z_x \ge t\Big)\le (m/d)\exp\!\left(-\frac{t^2}{2d}\right).
\]

Using the fact $\mathbb E[\max_{x\in R_i} Z_x]=\int_0^\infty \Pr(\max_{x\in R_i} Z_x\ge t)\,dt$ and choosing $t_0=\sqrt{2d\log (m/d)}$, we get the following:
\[
\mathbb E[\max_{x\in R_i} Z_x]\le t_0+\int_{t_0}^\infty (m/d) e^{-t^2/(2d)}\,dt\leq t_0+(m/d)\cdot \frac{d}{t_0} e^{-t_0^2/(2d)}
= t_0+ \frac{d}{t_0}\leq O(\sqrt{d\log (m/d)}).
\]

Therefore, we have the following:
\[
w(\lambda_1,\mathcal{R},\mathcal{X}) \le \max_{i\in[d]}\mathbb E[\max_{x\in R_i} Z_x] \le O(\sqrt{d\log(m/d)}).\
\]
\end{proof}

\begin{proposition}
    $w(\mathcal{X})\geq \Omega(\sqrt{d\log d})$.
\end{proposition}
\begin{proof}
Fix any $\lambda$ with $\Sigma:=\Sigma(\lambda)\succ 0$.
Define the set
\[
T := \{\Sigma^{-1/2}x:\ x\in\mathcal X\}\subset\mathbb R^d,
\]
and define the Gaussian process $Z_t:=\langle g,t\rangle$ for $g\sim\mathcal N(0,I_d)$ over $T$.
Then we have:
\[
f(\lambda)=\mathbb E\Big[\sup_{t\in T} Z_t\Big].
\]
Let $U:=\Sigma^{-1/2}X$ where $X\sim\lambda$. Since $\mathbb E[XX^\top]=\Sigma$,
\[
\mathbb E[UU^\top]=\Sigma^{-1/2}\,\mathbb E[XX^\top]\,\Sigma^{-1/2}=I_d.
\]

Let $m:=\lfloor d/2\rfloor$ and $r:=\sqrt{d/2}$.
We now claim that there exists $t_1,\dots,t_m\in T$ such that $\|t_i-t_j\|_2\ge r$ for all $i\ne j$.

We construct them greedily. First pick $t_1$ arbitrarily. Now assume that $t_1,\dots,t_k$ have been chosen for some $1\le k<m$.
Let $V_k:=\mathrm{span}\{t_1,\dots,t_k\}$, let $P_k$ be the orthogonal projection onto $V_k$,
and let $Q_k:=I-P_k$ be the orthogonal projection onto $V_k^\perp$.
Then $Q_k$ is symmetric and idempotent ($Q_k^2=Q_k$). Moreover, for an orthogonal projection, the trace equals the dimension of the subspace it projects onto. Note that $Q_k$ projects onto $V_k^\perp$, which has dimension $d-\dim(V_k)$. Therefore
\[
\mathrm{tr}(Q_k)=\dim(V_k^\perp)=d-\dim(V_k)\ge d-k,
\]
since $\dim(V_k)\le k$.

Since $Q_k$ is an orthogonal projection,
\[
\|Q_k U\|_2^2 = U^\top Q_k^\top Q_k U=U^\top Q_k^2 U= U^\top Q_k U.
\]
Using $\mathrm{tr}(ab)=ab$ for scalars and cyclicity of trace, we have the following:
\[
\mathbb E\|Q_k U\|_2^2
= \mathbb E[U^\top Q_k U]
= \mathbb E[\mathrm{tr}(Q_k U U^\top)]
= \mathrm{tr}\big(Q_k \mathbb E[UU^\top]\big)
= \mathrm{tr}(Q_k I_d)
= \mathrm{tr}(Q_k)
\ge d-k.
\]
Therefore there exists a point $t_{k+1}\in T$ with
\[
\|Q_k t_{k+1}\|_2^2 \ge d-k.
\]
For any $i\le k$, we have $t_i\in V_k$, which implies $Q_k t_i=0$. As orthogonal projections cannot increase the norm, we have the following:
\[
\|t_{k+1}-t_i\|_2 \ge \|Q_k(t_{k+1}-t_i)\|_2
= \|Q_k t_{k+1}\|_2
\ge \sqrt{d-k}
\ge \sqrt{d/2}=r.
\]
This completes the proof of our claim.

Let $S:=\{t_1,\dots,t_m\}$. Recall that $\|t_i-t_j\|_2\ge r$ for all $i\neq j$.
Now define $Z_i:=\langle g,t_i\rangle$ where $g\sim \mathcal{N}(0,1)$.
Then $(Z_i)_{i=1}^m$ is a centered Gaussian vector and for $i\ne j$, we have:
\[
\mathbb E[(Z_i-Z_j)^2] = \|t_i-t_j\|_2^2 \ge r^2.
\]
Let $\xi_1,\dots,\xi_m$ be i.i.d.\ $\mathcal N(0,1)$ and set $Y_i:=(r/\sqrt{2})\xi_i$.
Then for $i\ne j$, we have:
\[
\mathbb E[(Y_i-Y_j)^2] = \frac{r^2}{2}\mathbb E[(\xi_i-\xi_j)^2]=\frac{r^2}{2}\cdot 2=r^2.
\]
Due to Sudakov–Fernique inequality, we have the following:
\[
\mathbb E\Big[\sup_{t\in T} Z_t\Big]\geq\mathbb E\Big[\max_{1\le i\le m} Z_i\Big]
\ge
\mathbb E\Big[\max_{1\le i\le m} Y_i\Big]
=
\frac{r}{\sqrt{2}}\,\mathbb E\Big[\max_{1\le i\le m}\xi_i\Big]\geq \Omega(\sqrt{d\log d}).
\]

This implies that $w(\mathcal{X})\geq \Omega(\sqrt{d\log d})$.
\end{proof}

\begin{proposition}
    When $\mathcal{X}$ is $\{-1,+1\}^d$, $\{0,+1\}^d$ or a unit $\ell_2$-ball, we have $w(\mathcal{X})\geq\Omega(d)$. For $m\leq d/21$, we have $w(\mathcal{X})\geq \Omega(\sqrt{md})$ when $\mathcal{X}$ is an $m$-set.
\end{proposition}
\begin{proof}
    Recall that in Section \ref{sec:upper}, we described an algorithm with an upper bound of $\mathcal{O}\left(\frac{d\log(1/\delta)}{\varepsilon^2}+\frac{w(\mathcal{X})^2}{\varepsilon^2}\right).$
    When $\mathcal{X}$ is $\{-1,+1\}^d$, $\{0,+1\}^d$ or a unit $\ell_2$-ball, our adaptive lower bound results in Table \ref{table:1} imply that $w(\mathcal{X})\geq \Omega(d)$. Similarly when $\mathcal{X}$ is an $m$-set, our adaptive lower bound results in Table \ref{table:1} imply that $w(\mathcal{X})\geq \Omega(\sqrt{md})$ for $m\leq d/21$.
\end{proof}

\section{Gaussian Width Calculations for Multi-Task MAB}\label{appendix:gaussian-width-multi-task-MAB}
Recall that \(d=\sum_{j=1}^m d_j\). Consider the blocks \(B_j:=\{d_{1:j-1}+1,\dots,d_{1:j}\}\).  Recall that the action set is
\begin{equation}\label{eq:action-set}
\mathcal{X}:=\Big\{x\in\{0,1\}^d:\ \forall j\in[m],\ \sum_{i\in B_j}x_i=1\Big\}.\nonumber
\end{equation}

Let $r := d-m+1$. Since $\mathcal X$ is $r$-dimensional, it is contained in an $r$-dimensional linear subspace
$ \operatorname{span}(\mathcal X)\subseteq \mathbb R^d$.
Let $U\in\mathbb R^{d\times r}$ be a matrix whose columns form an orthonormal basis of $\operatorname{span}(\mathcal X)$, that is $U^\top U = I_r$ and $\operatorname{range}(U)=\operatorname{span}(\mathcal X)$.
Define the coordinate representation of $\mathcal X$ in $\mathbb R^r$ by
\[
\mathcal X_r := \{\,U^\top x : x\in \mathcal X\,\}\subseteq \mathbb R^r.
\]
Then the maps
\[
x \mapsto U^\top x \quad\text{and}\quad z \mapsto Uz
\]
are inverse to each other on $\mathcal X$ and $\mathcal X_r$, respectively. In particular, for every $x\in \mathcal X$,
\[
x = U(U^\top x),
\]
so $U^\top x$ is the unique coordinate vector of $x$ w.r.t the orthonormal basis $\{U_1,\ldots,U_r\}$ where $U_j$ is the $j$-th column of $U$.

We now construct one such matrix $U$. Define the vector $\mu\in\mathbb R^d$ by setting, for each $i\in\{1,\ldots,d\}$,
\[
\mu_i := \frac{1}{d_j}\quad \text{where $j$ is the unique index such that } i\in B_j.
\]
Then, we have:
\[
S := \|\mu\|_2^2 = \sum_{i=1}^d \mu_i^2 = \sum_{j=1}^m d_j\left(\frac{1}{d_j}\right)^2
= \sum_{j=1}^m \frac{1}{d_j},
\]
and we define the unit vector:
\[
u_0 := \frac{\mu}{\|\mu\|_2} = \frac{\mu}{\sqrt S}.
\]
We now define a matrix \(H_n\in\mathbb R^{n\times(n-1)}\). For \(n\ge 2\), define \((H_n)_{r,k}\) for \(\ell\in[n]\), \(k\in[n-1]\) by
\begin{equation}\label{eq:helmert-def}
(H_n)_{\ell,k}=
\begin{cases}
\frac{1}{\sqrt{k(k+1)}} & 1\le \ell\le k,\\[4pt]
-\frac{k}{\sqrt{k(k+1)}} & \ell=k+1,\\[4pt]
0 & \ell\ge k+2.\nonumber
\end{cases}
\end{equation}
Then, we have:
\begin{equation}\label{eq:helmert-props}
H_n^\top H_n=I_{n-1},\qquad H_n^\top \mathbf 1_n=\mathbf{0}_{n-1},
\qquad
H_n H_n^\top = I_n-\frac1n\mathbf 1_n\mathbf 1_n^\top.
\end{equation}
Let \(s_j:=d_{1:j-1}\). We define the entries of the matrix \(Q_j\in\mathbb R^{d\times(d_j-1)}\) as:
\begin{equation}\label{eq:Qj-embed}
(Q_j)_{i,k}=
\begin{cases}
(H_{d_j})_{\,i-s_j,\ k} & i\in B_j,\\
0 & i\notin B_j,
\end{cases}
\qquad k\in[d_j-1].
\end{equation}
Now we define:
\begin{equation}\label{eq:U-def}
U:=\big[u_0\ \ Q_1\ \cdots\ Q_m\big]\in\mathbb R^{d\times r},\qquad
r:=1+\sum_{j=1}^m(d_j-1)=d-m+1.\nonumber
\end{equation}
As \(Q_j,Q_\ell\) have disjoint supports  for \(j\neq \ell\) and due to Eq. \eqref{eq:helmert-props}, we get:
\begin{equation}\label{eq:Q-orth}
Q_j^\top Q_j = H_{d_j}^\top H_{d_j}=I_{d_j-1},\qquad
Q_j^\top Q_\ell=0\ (j\neq \ell).\nonumber
\end{equation}
For any matrix $A$, let $A_{i,:}$ denote the $i$-th row of the matrix $A$. Since \(u_0\) is constant on block \(B_j\), we have:
\begin{align*}\label{eq:u0-Qj-orth}
u_0^\top Q_j
&=\sum_{i\in B_j}(u_0)_i (Q_j)_{i,:}
=\frac{1/d_j}{\sqrt S}\sum_{r=1}^{d_j}(H_{d_j})_{r,:}
=\frac{1/d_j}{\sqrt S}\,\mathbf 1_{d_j}^\top H_{d_j}
=\mathbf{0}_{n-1}.
\end{align*}
Hence
\begin{equation}\label{eq:U-orth}
U^\top U = I_r.\nonumber
\end{equation}
Next, we prove that \(\operatorname{range}(U)=\operatorname{span}(\mathcal X)\). Towards that, define
\begin{equation}\label{eq:Uj-def}
R_j:=\Big\{v\in\mathbb R^d:\operatorname{supp}(v)\subseteq B_j,\ \mathbf 1_{B_j}^\top v=0\Big\}.\nonumber
\end{equation}
where $\mathbf 1_{B_j}$ is the indicator vector of the block $B_j$. From \eqref{eq:helmert-props}–\eqref{eq:Qj-embed}, we get:
\begin{equation}\label{eq:range-Qj}
\operatorname{range}(Q_j)=R_j.\nonumber
\end{equation}
Also \(\mu\in\operatorname{span}(\mathcal{X})\) and $\forall j,\ \forall k\in[d_j],
e_{s_j+k}-e_{s_j+1} \in \operatorname{span}(\mathcal{X})$, so \(R_j\subseteq \operatorname{span}(\mathcal X)\) (since \(\{e_{s_j+k}-e_{s_j+1}\}_{k=2}^{d_j}\) spans \(R_j\)). Thus
\begin{equation}\label{eq:span-inclusion-1}
\operatorname{span}(\mu)\oplus\Big(\bigoplus_{j=1}^m R_j\Big) \subseteq \operatorname{span}(\mathcal X).\nonumber
\end{equation}
where $\oplus$ means direct sum of vector spaces.

Conversely, for any \(x\in \mathcal X\), we have:
\begin{equation}\label{eq:span-inclusion-2}
\forall j\in[m]:\ \mathbf 1_{B_j}^\top(x-\mu)=1-1=0
\quad\Rightarrow\quad x-\mu\in \bigoplus_{j=1}^m R_j,
\quad\Rightarrow\quad x\in \operatorname{span}(\mu)\oplus\Big(\bigoplus_{j=1}^m R_j\Big).\nonumber
\end{equation}
Hence
\begin{equation}\label{eq:spanX-rangeU}
\operatorname{span}(\mathcal X)=\operatorname{span}(\mu)\oplus\Big(\bigoplus_{j=1}^m R_j\Big)
=\operatorname{range}(U),\text{ and indeed } \dim(\operatorname{span}(\mathcal X))=r=d-m+1.\nonumber
\end{equation}
Now the gaussian width term $w(\mathcal{X})$ is defined using the coordinate representation of $\mathcal X$ in $\mathbb R^r$ as follows:
\[
w(\mathcal{X}):=\inf_{\lambda\in\triangle_{\mathcal{X}_r}}\mathbb E_{g_r\sim \mathcal{N}(0,I_r)}\Big[\sup_{x\in \mathcal{X}}\ \langle U^\top x,\ A(\lambda;\mathcal{X}_r)^{-1/2}g_r\rangle\Big]
\]
Fix a probability distribution $\lambda$ on $\mathcal X_r$ such that, if $Z\sim\lambda$ and we map to $\mathcal X$ via
\[
X := UZ \in \mathcal X \subseteq \mathbb R^d,
\]
then the following holds: for each block $B_k$, exactly one coordinate index $j_k \in B_k$ is selected uniformly at random, and the selections $\{j_k\}_{k=1}^m$ are independent across blocks. In other words, viewed in $\mathbb R^d$, $X$ is distributed as a vector obtained by choosing one coordinate uniformly within each block, independently across blocks.  Now define:
\begin{equation}\label{eq:sigma-def}
\Sigma_r:=\mathbb E_{x\sim\lambda}[xx^\top],\nonumber
\end{equation}
where $\Sigma_r\in \mathbb{R}^{r\times r}$. Now the Gaussian width of $\mathcal{X}$ is upper bounded as follows:
\begin{equation}\label{eq:gaussian-width}
w(\mathcal{X})\leq w(\mathcal{X}_r;\Sigma_r):=\mathbb E_{g_r\sim \mathcal{N}(0,I_r)}\Big[\sup_{x\in \mathcal{X}}\ \langle U^\top x,\ \Sigma_r^{-1/2}g_r\rangle\Big]\nonumber
\end{equation}
Observe that \(\Sigma_r := U^\top \Sigma U\) where $\Sigma:=\mathbb{E}_{x\sim\lambda}[Ux(Ux)^\top]$. Due to the definition of $\lambda$, for \(i\in B_j\), \(k\in B_\ell\), we have:
\begin{equation}\label{eq:Sigma-entries}
\Sigma_{ik}=
\begin{cases}
\frac1{d_j}\mathbf 1\{i=k\} & j=\ell,\\[3pt]
\frac1{d_j d_\ell} & j\neq \ell.
\end{cases}\nonumber
\end{equation}
Let \(v\in R_j\) (\(\mathbf 1_{B_j}^\top v=0\), support in \(B_j\)). For \(i\in B_j\),
\begin{align*}\label{eq:Sigma-on-Uj}
(\Sigma v)_i
&=\sum_{k\in B_j}\frac1{d_j}\mathbf 1\{i=k\}v_k
+\sum_{\ell\neq j}\sum_{k\in B_\ell}\frac1{d_j d_\ell}v_k
=\frac1{d_j}v_i+0,
\end{align*}
so \(\Sigma v=\frac1{d_j}v\). Therefore
\begin{equation}\label{eq:Q-blocks}
Q_j^\top \Sigma Q_j=\frac1{d_j}Q_j^\top Q_j=\frac1{d_j}I_{d_j-1},
\qquad
Q_j^\top \Sigma Q_\ell=0\ (j\neq \ell).\nonumber
\end{equation}
For \(i\in B_j\),
\begin{align*}\label{eq:Sigma-mu}
(\Sigma\mu)_i
&=\sum_{k\in B_j}\frac1{d_j}\mathbf 1\{i=k\}\frac1{d_j}
+\sum_{\ell\neq j}\sum_{k\in B_\ell}\frac1{d_j d_\ell}\frac1{d_\ell}\\
&=\frac1{d_j^2}+\sum_{\ell\neq j}\frac1{d_j d_\ell}
=\frac1{d_j}\sum_{\ell=1}^m\frac1{d_\ell}
=S\,\mu_i.
\end{align*}
Thus \(\Sigma\mu=S\mu\) and, since \(\|u_0\|=1\),
\begin{equation}\label{eq:u0-eigs}
u_0^\top\Sigma u_0 = S,\qquad u_0^\top\Sigma Q_j = 0.\nonumber
\end{equation}
Hence, in the explicit basis \(U\),
\begin{equation}\label{eq:Sigma_r}
\Sigma_r
:=U^\top\Sigma U
=
\operatorname{diag}\Big(S,\ \frac1{d_1}I_{d_1-1},\ \dots,\ \frac1{d_m}I_{d_m-1}\Big),
\quad
\Sigma_r^{-1/2}
=
\operatorname{diag}\Big(S^{-1/2},\ \sqrt{d_1}\,I_{d_1-1},\ \dots,\ \sqrt{d_m}\,I_{d_m-1}\Big).
\end{equation}

\bigskip

\subsection*{D) Gaussian Width in \(r=d-m+1\) Dimensions and Evaluation}

Let \(g_r\sim\mathcal N(0,I_r)\), write \(g_r=(g_0,g^{(1)},\dots,g^{(m)})\) with
\(g^{(j)}\sim\mathcal N(0,I_{d_j-1})\) independent. Recall that
\begin{equation}\label{eq:width-intrinsic}
w(\mathcal X)\leq\mathbb E_{g_r\sim\mathcal{N}(0,I_r)}\Big[\sup_{x\in \mathcal X}\ \langle U^\top x,\ \Sigma_r^{-1/2}g_r\rangle\Big].\nonumber
\end{equation}
Also, we have
\begin{equation}\label{eq:u0-x}
\langle u_0,x\rangle=\frac{\langle \mu,x\rangle}{\sqrt S}
=\frac{\sum_{j=1}^m \frac1{d_j}}{\sqrt S}
=\sqrt S,\qquad \forall x\in \mathcal X,\nonumber
\end{equation}
so, from \eqref{eq:Sigma_r}, we have
\begin{align*}
\langle U^\top x,\Sigma_r^{-1/2}g_r\rangle
&=
S^{-1/2}g_0\langle u_0,x\rangle
+\sum_{j=1}^m \sqrt{d_j}\,\langle g^{(j)},Q_j^\top x\rangle\\
&=
g_0+\sum_{j=1}^m \sqrt{d_j}\,\langle g^{(j)},Q_j^\top x\rangle.
\end{align*}
Therefore
\begin{align}\label{eq:width-reduce}
w(\mathcal X)\leq
\mathbb E[g_0]
+\sum_{j=1}^m \sqrt{d_j}\ \mathbb E\Big[\max_{k\in[d_j]}\ \langle g^{(j)},H_{d_j}^\top e_k\rangle\Big]
&=
\sum_{j=1}^m \sqrt{d_j}\ \mathbb E\Big[\max_{k\in[d_j]}\ \langle g^{(j)},H_{d_j}^\top e_k\rangle\Big].
\end{align}
Let \(h^{(j)}:=H_{d_j}g^{(j)}\in\mathbb R^{d_j}\). Then we have
\begin{equation}\label{eq:h-coords}
\langle g^{(j)},H_{d_j}^\top e_k\rangle = e_k^\top H_{d_j}g^{(j)} = (h^{(j)})_k,\nonumber
\end{equation}
and by \eqref{eq:helmert-props},
\begin{equation}\label{eq:h-law}
h^{(j)}\sim\mathcal N\Big(0,\ H_{d_j}H_{d_j}^\top\Big)
=\mathcal N\Big(0,\ I_{d_j}-\frac1{d_j}\mathbf 1\mathbf 1^\top\Big).\nonumber
\end{equation}
If \(Z^{(j)}\sim\mathcal N(0,I_{d_j})\) and \(\bar Z^{(j)}:=\frac1{d_j}\mathbf 1^\top Z^{(j)}\), then
\begin{equation}\label{eq:center-proj}
\Big(I_{d_j}-\frac1{d_j}\mathbf 1\mathbf 1^\top\Big)Z^{(j)}
=Z^{(j)}-\bar Z^{(j)}\mathbf 1
\sim \mathcal N\Big(0,\ I_{d_j}-\frac1{d_j}\mathbf 1\mathbf 1^\top\Big),\nonumber
\end{equation}
so
\begin{equation}\label{eq:max-shift}
\max_{k\le d_j} (h^{(j)})_k\ \stackrel{d}{=}\ \max_{k\le d_j}\big((Z^{(j)})_k-\bar Z^{(j)}\big)
=\max_{k\le d_j} (Z^{(j)})_k-\bar Z^{(j)}.\nonumber
\end{equation}
Taking expectations and using \(\mathbb E[\bar Z^{(j)}]=0\),
\begin{equation}\label{eq:max-equality}
\mathbb E\Big[\max_{k\le d_j} (h^{(j)})_k\Big]
=
\mathbb E\Big[\max_{k\le d_j} Z_k\Big]\leq \mathcal{O}(\sqrt{\log d_i}),
\qquad Z_k\stackrel{iid}{\sim}\mathcal N(0,1).\nonumber
\end{equation}
Thus, due to Eq. \eqref{eq:width-reduce} and the calculations above, we have
\begin{equation}\label{eq:width-final}
w(\mathcal{X})\leq O\left(\sum_{j=1}^m \sqrt{d_j\log d_j}\right).\nonumber
\end{equation}

\section{Lower Bounds for Structured Sets}\label{appendix-structured-sets-lower-bound}
First, in section \ref{appendix-multi-task-lower bound}, we present the lower bound for the multi-task MAB problem. Next, in section \ref{appendix-multi-task-lower bound}, we present the lower bound for hypercubes. Finally, in section \ref{appendix-m-sets-lower-bound}, we present the lower bound for $m$-sets.
\subsection{Lower Bound for Multi-task MAB}\label{appendix-multi-task-lower bound}


Recall that in the Multi-task MAB, the set of arms $\mathcal{X}$ is defined as follows: 
\begin{equation*}
    \mathcal{X}=\left\{\mathbf{x}\in \{0,1\}^d: \forall j\in [m] \sum_{i=d_{1:j-1}+1}^{d_{1:j}} \mathbf{x}[i]=1\right\}
\end{equation*}
where $d_i\geq 2$, $d=\sum_{i=1}^md_i$, $d_{1:j}=\sum_{i=1}^jd_i$ and $d_{1:0}=0$ for all $j\in [m]$.

We first focus on the case when $d_j\geq 100$ for all $j\in[m]$.

Let us fix one regret minimizing algorithm, say $\mathcal{A}$ and assume that $\mathcal{A}$ is deterministic given the reward realizations. Let us assume that $\mathcal{A}$ terminates after taking $T:=\frac{(\sum_{i=1}^m\sqrt{d_i})^2}{20000\varepsilon^2}$ samples (if the algorithm terminates before this, we can take additional dummy samples). Let $\widehat{x}\in \mathcal{X}$ be the arm recommended by $\mathcal{A}$ after sampling the arms $\mathbf{x}_1,\ldots,\mathbf{x}_T$ and terminating. 

Later in this section, we construct a set of input instances $\mathcal{I}$. Each instance $I\in \mathcal{I}$ is associated with a reward vector $\theta_I\in \mathbb{R}^d$ such that the following hold:
\begin{itemize}
    \item If $\mathcal{A}$ samples an arm $x\in \mathcal{X}$, it observes $\langle x,\theta_I\rangle+\eta$ where $\eta\sim \mathcal{N}(0,1)$.
    \item For all $x\in \mathcal{X}$, we have $\langle x,\theta_I\rangle\geq 0$.
    \item We have $\max_{x\in \mathcal{X}}\langle x,\theta_I\rangle=10\varepsilon$.
\end{itemize}

Suppose we show that $\mathbb{E}_{I\sim Unif(\mathcal{I})}[\langle \widehat{x}, \theta_I\rangle]\leq 7\varepsilon$, where $\mathbb E_{I}[\cdot]$ is the expectation under the instance $I$. Then by Yao's lemma, we have for any randomized algorithm, the recommended arm $\widehat{x}$ satisfies $\min_{I\in \mathcal{I}}\mathbb E_{I}[\langle \widehat{x},\theta_I\rangle]\leq 7\varepsilon$. Then due to markov's inequality, we have that with probability at least $0.1$, there exists an instance on which the randomized algorithm does not recommend an $\varepsilon$-best arm.

Now we focus on showing that $\mathbb{E}_{I\sim Unif(\mathcal{I})}[\langle \widehat{x}, \theta_I\rangle]\leq 7\varepsilon$. We begin with construction instances $I$ with their corresponding reward vector $\theta_I$. These instances include both the set of all input instances and a set of alternate instances that will be used to argue the performance of the algorithm $\mathcal{A}$.


Let $\tilde{\mathcal{X}}=\left\{\mathbf{x}\in \{0,1\}^d: \forall j\in [m] \sum_{i=d_{1:j-1}+1}^{d_{1:j}} \mathbf{x}[i]\leq 1\right\}$.
First we describe an instance $I_{\tilde{\mathbf{x}}}$, where $\tilde{\mathbf{x}}\in \tilde{\mathcal{X}}$. In this instance, we define the associated reward vector $\theta_{I_{\tilde{\mathbf{x}}}}$ as follows. For all $j\in [m]$ and $i\in \{d_{1:j-1}+1,\ldots,d_{1:j}\}$, we have $\theta_{I_{\tilde{\mathbf{x}}}}[i]=10\varepsilon_j\cdot \tilde{\mathbf{x}}[i]$ where $\varepsilon_j=\frac{\varepsilon\cdot\sqrt{d_j}}{\sum_{s\in[m]}\sqrt{d_s}}$. Observe that for all $x\in\mathcal{X}$, we have $\langle x,\theta_{I_{\tilde x}}\rangle\geq 0$. Also observe that $\max_{x\in\mathcal{X}}\langle x,\theta_{I_{\tilde x}}\rangle=10\varepsilon$ if $\tilde{x}\in \mathcal{X}$.

Let $r^{(j)}_{\tilde x}:=\mathbb{E}_{I_{\tilde x}}[\sum_{i=d_{1:j-1}+1}^{d_{1:j}}\widehat{x}[i]\cdot \theta_{I_{\tilde x}}[i]]$. Observe that $\mathbb{E}_{I_{\tilde x}}[\langle \widehat{x},\theta_{I_{\tilde x}}\rangle]=\sum_{j=1}^m r^{(j)}_{\tilde x}$.


    Let $\mathcal{I}=\bigcup_{\mathbf{x}\in\mathcal{X}}I_{\mathbf{x}}$. Fix an index $j\in[m]$. We now show that $\mathbb{E}_{I_{\mathbf{x}'}\sim \mathrm{Unif}(\mathcal{I})}[\sum_{i=d_{1:j-1}+1}^{d_{1:j}}\widehat{x}[i]\cdot \theta_{I_{\mathbf{x}'}}[i]]\leq 7\varepsilon_j$. Let $$\mathcal{X}^{(j)}:=\Bigg\{\mathbf{x}\in \{0,1\}^d: \forall i\in[m]\setminus\{j\}\; \sum_{s=d_{1:i-1}+1}^{d_{1:i}}\mathbf{x}[s]=1,\; \sum_{s=d_{1:j-1}+1}^{d_{1:j}}\mathbf{x}[s]=0\Bigg\}.$$ For any $\mathbf{x}\in \mathcal{X}^{(j)}$, let $\mathbf{x}^{(i)}$ be the vector in $\mathcal{X}$ such that $\mathbf{x}^{(i)}[s]=\mathbf{x}[s]$ for all $s\notin \{d_{1:j-1}+1,\ldots,d_{1:j}\}$ and $\mathbf{x}^{(i)}[d_{1:j-1}+i]=1$.

    Let us fix $\mathbf{x}\in \mathcal{X}^{(j)}$. Now we claim that there is a set $\mathcal{S}_{\mathbf{x}}\subseteq[d_j]$ with at least $d_j/3$ indices such that for each $i\in \mathcal{S}_\mathbf{x}$, we have $r^{(j)}_{\mathbf{x}^{(i)}} \leq \varepsilon_j$.
    
    Before we prove our claim, we first show that if our claim holds true, then we have that $$\mathbb{E}_{I_{\mathbf{x}'}\sim \mathrm{Unif}(\mathcal{I})}[\sum_{i=d_{1:j-1}+1}^{d_{1:j}}\widehat{x}[i]\cdot \theta_{I_{\mathbf{x}'}}[i]]\leq 7\varepsilon_j.$$ Now we have the following:
    \begin{align*}
        \mathbb{E}_{I_{\mathbf{x}'}\sim \mathrm{Unif}(\mathcal{I})}[\sum_{i=d_{1:j-1}+1}^{d_{1:j}}\widehat{x}[i]\cdot \theta_{I_{\mathbf{x}'}}[i]]&=\frac{1}{\prod_{s=1}^md_s}\sum_{\mathbf{x}\in\mathcal{X}^{(j)}}\sum_{i=1}^{d_j} r^{(j)}_{\mathbf{x}^{(i)}} \\
        &= \frac{1}{\prod_{s=1}^md_s}\sum_{\mathbf{x}\in\mathcal{X}^{(j)}}\left(\sum_{i\in \mathcal{S}_{\mathbf{x}}}r^{(j)}_{\mathbf{x}^{(i)}}+\sum_{i\in [d_j]\setminus\mathcal{S}_{\mathbf{x}}}r^{(j)}_{\mathbf{x}^{(i)}}\right)\\
        &\leq \frac{1}{\prod_{s=1}^md_s}\sum_{\mathbf{x}\in\mathcal{X}^{(j)}}\left(\frac{d_j}{3}\cdot \varepsilon_j+\frac{2d_j}{3}\cdot 10\varepsilon_j\right)\\
        &\leq \frac{1}{\prod_{s=1}^md_s}\sum_{\mathbf{x}\in\mathcal{X}^{(j)}}7d_j\varepsilon_j  \\
        & = \frac{7\varepsilon_j}{\prod_{s=1}^md_s} \cdot \prod_{s\neq j}d_s\cdot d_j \\
        & = 7\varepsilon_j
    \end{align*}

    Now we prove our claim. We use the adaptive KL chain rule (Lemma \ref{kl-chain-rule}) in our analysis.

    Let $\mathbb{P}_I$ denote the probability law instances under an instance $I$. For an instance $I_{\mathbf{x}^{(i)}}$, let $f_i(\ell_1,\ldots,\ell_T)$ denote the joint PDF for the tuple of reward values observed by $\mathcal{A}$ in each round under the probability law $\mathbb{P}_{I_{\mathbf{x}^{(i)}}}$. Observe that our sample space is $\Omega=\mathbb{R}^T$. This is a valid sample space as $\mathcal{A}$ is deterministic and the probability density function of the reward values in round $t$ only depends on the reward values it observed in the previous rounds. Similarly for the alternate instance  $I_{\mathbf{x}}$, let $f_0(\ell_1,\ldots,\ell_T)$ denote the joint PDF for the tuple of loss values observed by $\mathcal{A}$ in each round under the probability law $\mathbb{P}_{I_{\mathbf{x}}}$.

    First observe that the instances $I_{\mathbf{x}^{(i)}}$ and $I_{\mathbf{x}}$ only differ at index $d_{1:j-1}+i$. For each $\omega\in \Omega$, let $\mathbf{x}_{1,\omega},\mathbf{x}_{2,\omega},\ldots, \mathbf{x}_{T,\omega}$ be the sequence of arms chosen by $\mathcal{A}$ on $\omega$.
    Conditioning on a set of outcomes $X_1=\omega_1,X_2=\omega_2,\ldots,X_{t-1}=\omega_{t-1}$, we have $X_t\sim \mathcal{N}(\mu_i,1)$ for the instance $I_{\mathbf{x}^{(i)}}$ and $X_t\sim \mathcal{N}(\mu_0,1)$ for the instance $I_{\mathbf{x}}$ where $\mu_i-\mu_0=10\varepsilon_j\cdot \mathbf{x}_{t,\omega}[d_{1:j-1}+i]$. Let $T_i=\sum_{t=1}^T\mathbf{x}_t[d_{1:j-1}+i]$.  For each $\omega\in \Omega$, let $T_{i,\omega}=\sum_{t=1}^T\mathbf{x}_{t,\omega}[d_{1:j-1}+i]$. Note that $T_i$ is a random variable and $T_{i,\omega}$ is a fixed value. Let $X_{-t}$ denote $X_1,\ldots,X_{t-1}$ and $\omega_{-t}$ denote $\omega_1,\ldots,\omega_{t-1}$. Now we have the following:
    \begin{align*}
         KL(f_0,f_i)&=\int\limits_{\omega\in \Omega}f_0(\omega)\left(KL(f_0(X_1),f_i(X_1))+\sum_{t=2}^T KL(f_0(X_t|X_{-t}=\omega_{-t}),f_i(X_t|X_{-t}=\omega_{-t}))\right)d\omega\\
         &=50\varepsilon_j^2\int\limits_{\omega\in \Omega}f_0(\omega)\sum_{t=1}^T\mathbf{x}_{t,\omega}[d_{1:j-1}+i]\;d\omega\\
         &=50\varepsilon_j^2 \int\limits_{\omega\in \Omega}f_0(\omega)T_{i,\omega}\;d\omega\\
         &= 50\varepsilon_j^2\cdot \mathbb{E}_{I_{\mathbf{x}}}[T_{i}]
    \end{align*}

    Now observe that $\sum_{i=1}^{d_j}\mathbb{E}_{I_{\mathbf{x}}}[T_i]=T$. Hence, there exists a set $\mathcal{S}_\mathbf{x}^{(1)}\subseteq[d_j]$ with at least $2d_j/3$ indices such that for each $i\in \mathcal{S}_x^{(1)}$, we have $\mathbb{E}_{I_{\mathbf{x}}}[T_i]\leq \frac{3T}{d_j}$. Similarly, there exists a set $\mathcal{S}_\mathbf{x}^{(2)}\subseteq[d_j]$ with at least $2d_j/3$ indices such that for each $i\in \mathcal{S}_x^{(2)}$, we have $\mathbb{P}_{I_{\mathbf{x}}}[\widehat{x}[d_{1:j-1}+i]=1]\leq \frac{3}{d_j}$. Now we define $\mathcal{S}_\mathbf{x}:=\mathcal{S}_x^{(1)}\cap \mathcal{S}_x^{(2)}$. Observe that $|\mathcal{S}_\mathbf{x}|\geq d_j/3$. Next for each $i\in \mathcal{S}_x$, we have $\mathbb{E}_{I_{\mathbf{x}}}[T_i]\leq \frac{3T}{d_j}$ and $\mathbb{P}_{I_{\mathbf{x}}}[\widehat{x}[d_{1:j-1}+i]=1]\leq \frac{3}{d_j}$. Now for each $i\in \mathcal{S}_\mathbf{x}$, we have $KL(f_0,f_i)\leq\frac{150\varepsilon_j^2T}{d_j}=\frac{3}{400}$. 

    Fix $i\in \mathcal{S}_x$. Let $A_i$ be the event that $\widehat{x}[d_{1:j-1}+i]=1$. Now due to Pinsker's inequality we have the following:
    \begin{align*}
        \mathbb{P}_{I_{\mathbf{x}^{(i)}}}(A_i)&\leq  \mathbb{P}_{I_{\mathbf{x}}}(A_i)+\sqrt{\frac{KL(f_0,f_i)}{2}}\\
        &\leq \frac{3}{d_j}+\sqrt{\frac{3}{800}}\\
        &\leq\frac{3}{100}+\sqrt{\frac{3}{800}}\tag{as $d_j\geq 100$}\\
        &< \frac{1}{10}
    \end{align*}

    Hence, we have $r^{(j)}_{\mathbf{x}^{(i)}} < \frac{1}{10}\cdot 10\varepsilon_j= \varepsilon_j$.

    Now, we conclude the proof by showing that $\mathbb{E}_{I\sim Unif(\mathcal{I})}[\langle \widehat{x}, \theta_I\rangle]\leq 7\varepsilon$:
    \begin{align*}
        \mathbb{E}_{I\sim Unif(\mathcal{I})}[\langle \widehat{x}, \theta_I\rangle]&=\sum_{j=1}^m \mathbb{E}_{I_{\mathbf{x}'}\sim \mathrm{Unif}(\mathcal{I})}[\sum_{i=d_{1:j-1}+1}^{d_{1:j}}\widehat{x}[i]\cdot \theta_{I_{\mathbf{x}'}}[i]]\\
        &=7\sum_{j=1}^m \varepsilon_j\\
        &=7\varepsilon\cdot\frac{\sum_{j=1}^m \sqrt{d_j}}{\sum_{s=1}^m \sqrt{d_s}}\tag{as $\varepsilon_j=\frac{\varepsilon\cdot\sqrt{d_j}}{\sum_{s=1}^m \sqrt{d_s}}$}\\
        &=7\varepsilon
    \end{align*}

\subsubsection{$d_j<100$ Case}

For the simplicity of presentation let us focus on the case $d_j=2$. This analysis can be easily extended to the other cases easily.
For the case $d_j=2$, the construction of the hard instances remain exactly the same as that of the $d_j\geq 100$ case. The analysis for this case only differs at the end. Consider an index $i\in\{1,2\}$. Observe that $KL(f_0,f_i)\leq50\varepsilon_j^2T=\frac{1}{200}$. Recall that $A_i$ is the event that $\widehat{x}[d_{1:j-1}+i]=1$. Now due to Pinsker's inequality we have the following:
    \begin{align*}
        \mathbb{P}_{I_{\mathbf{x}^{(i)}}}(A_i)\leq  \mathbb{P}_{I_{\mathbf{x}}}(A_i)+\sqrt{\frac{KL(f_0,f_i)}{2}}\leq \mathbb{P}_{I_{\mathbf{x}}}(A_i)+\frac{1}{20}\\
    \end{align*}
Recall that $r^{(j)}_{\tilde x}:=\mathbb{E}_{I_{\tilde x}}[\sum_{i=d_{1:j-1}+1}^{d_{1:j}}\widehat{x}[i]\cdot \theta_{I_{\tilde x}}[i]]$. Now we have the following:
    \begin{align*}
        \mathbb{E}_{I_{\mathbf{x}'}\sim \mathrm{Unif}(\mathcal{I})}[\sum_{i=d_{1:j-1}+1}^{d_{1:j}}\widehat{x}[i]\cdot \theta_{I_{\mathbf{x}'}}[i]]&=\frac{1}{\prod_{s=1}^md_s}\sum_{\mathbf{x}\in\mathcal{X}^{(j)}}\sum_{i=1}^{2} r^{(j)}_{\mathbf{x}^{(i)}} \\
        &\leq \frac{10\varepsilon_j}{\prod_{s=1}^md_s}\sum_{\mathbf{x}\in\mathcal{X}^{(j)}}\left(\mathbb{P}_{I_{\mathbf{x}}}(A_1)+\mathbb{P}_{I_{\mathbf{x}}}(A_2)+0.1\right)\\
        & = \frac{10\varepsilon_j}{\prod_{s=1}^md_s} \cdot\frac{1.1}{2}\cdot \prod_{s\neq j}d_s\cdot d_j \\
        & = 5.5\varepsilon_j
    \end{align*}

\subsection{Hypercubes}\label{appendix-hypercubes-lower-bound}
For the $\{0,1\}^d$ hypercube, we consider the family of hard instances $\{I_{\mathbf{x}}\}_{\mathbf{x}\in\{0,1\}^d}$ with corresponding reward vectors $\{\theta_{\mathbf{x}}\}_{\mathbf{x}\in\{0,1\}^d}$. For each $\mathbf{x}\in\{0,1\}^d$, define $\theta_{\mathbf{x}}\in\mathbb{R}^d$ coordinate-wise by
$(\theta_{\mathbf{x}})_i=\mathbbm{1}\{\mathbf{x}_i=1\}\cdot \frac{10\varepsilon}{d}-\mathbbm{1}\{\mathbf{x}_i=0\}\cdot \frac{10\varepsilon}{d}$.
An analysis analogous to our $d_j=2$ argument for multi-task MAB yields a lower bound of $\Omega(d^2/\varepsilon^2)$. The only minor difference is that, in the multi-task setting, we reasoned about the expected value of $\langle x,\theta\rangle$, whereas here one must instead work with the expected gap $\langle x,\theta\rangle-\min_{x\in\{0,1\}^d}\langle x,\theta\rangle$.

For the $\{-1,+1\}^d$ hypercube, we similarly consider the family of hard instances $\{I_{\mathbf{x}}\}_{\mathbf{x}\in\{-1,+1\}^d}$ with corresponding reward vectors $\{\theta_{\mathbf{x}}\}_{\mathbf{x}\in\{-1,+1\}^d}$. For each $\mathbf{x}\in\{-1,+1\}^d$, define $\theta_{\mathbf{x}}\in\mathbb{R}^d$ coordinate-wise by
$(\theta_{\mathbf{x}})_i=\mathbf{x}_i\cdot \frac{5\varepsilon}{d}$.
An analysis analogous to our $d_j=2$ argument for multi-task MAB again yields a lower bound of $\Omega(d^2/\varepsilon^2)$. As above, the only difference is that the argument proceeds via the expected gap $\langle x,\theta\rangle-\min_{x\in\{-1,+1\}^d}\langle x,\theta\rangle$ rather than expectation of $\langle x,\theta\rangle$ itself.

\subsection{$m$-Sets Lower Bound}\label{appendix-m-sets-lower-bound}
\begin{theorem}\label{appendix:thm-m-sets}
Let us denote the $m$-sets by
\[
\mathcal X=\{x\in\{0,1\}^d:\|x\|_0=m\}.
\]
In each round $t=1,\dots,T$, the learner chooses $x_t\in\mathcal X$ and observes
\[
y_t=\langle x_t,\theta\rangle+\eta_t,\qquad \eta_t\sim \mathcal N(0,1)\ \text{i.i.d.}
\]
Suppose $d-m+1\ge 20m$. Then there exists a universal constant $c>0$ such that for any (possibly adaptive and randomized) algorithm,
if
\[
T \le c\,\frac{m(d-m+1)}{\varepsilon^2},
\]
then there exists an instance $\theta$ for which the algorithm fails to output an $\varepsilon$-optimal arm with probability
at least $2/9$.
\end{theorem}

\begin{proof}
We first describe the family of hard instances. For each $S\subset[d]$ with $|S|=m$, define
\[
\theta^{(S)}_i=
\begin{cases}
\Delta & i\in S,\\
0 & i\notin S,
\end{cases}
\qquad\text{where}\qquad \Delta:=\frac{10\varepsilon}{m}.
\]
Let $S$ be uniform over all $m$-sized subsets of $[d]$ and the environment parameter be $\theta=\theta^{(S)}$.
For any output $\hat x\in\mathcal X$, write $\hat S:=\mathrm{supp}(\hat x)$ so $|\hat S|=m$. Then
\[
\langle \hat x,\theta^{(S)}\rangle=\Delta\,|\hat S\cap S|.
\]
The optimal value is $\max_{x\in\mathcal X}\langle x,\theta^{(S)}\rangle=\Delta|S|=m\Delta=10\varepsilon$.
Thus on any instance, $\varepsilon$-success implies $\langle \hat x,\theta^{(S)}\rangle\ge 9\varepsilon$.

By Yao's lemma it suffices to fix an arbitrary deterministic algorithm and analyze its error under the random draw of $S$.

Given $S$, let $I$ be uniform on $S$ (an auxiliary random variable), and define $B:=S\setminus\{I\}$ so $|B|=m-1$.
Conditioning on $(S,\hat S)$,
\[
\mathbb P(I\in \hat S\mid S,\hat S)=\frac{|\hat S\cap S|}{|S|}=\frac{|\hat S\cap S|}{m},
\]
so by the tower rule
\begin{equation}\label{eq:overlap-identity}
\mathbb E[|\hat S\cap S|]=m\,\mathbb P(I\in \hat S).
\end{equation}
Therefore,
\begin{equation}\label{eq:value-vs-I}
\mathbb E[\langle \hat x,\theta^{(S)}\rangle]
=\Delta\,\mathbb E[|\hat S\cap S|]
=m\Delta\,\mathbb P(I\in \hat S)
=10\varepsilon\cdot \mathbb P(I\in \hat S).
\end{equation}
So it is enough to show $\mathbb P(I\in \hat S)\le 0.7$, which would imply
$\mathbb E[\langle \hat x,\theta^{(S)}\rangle]\le 7\varepsilon$.


We now prove the following technical lemma.
\begin{lemma}[Equivalent sampling]\label{appendix-lem-equivalent-sampling}
The joint law of $(B,I,S)$ defined above is the same as:
(i) draw $B$ uniformly among all $(m-1)$-subsets of $[d]$,
(ii) draw $I$ uniformly from $[d]\setminus B$,
(iii) set $S=B\cup\{I\}$.
In particular, conditional on $B$, $I\sim \mathrm{Unif}([d]\setminus B)$.
\end{lemma}
\begin{proof}
Fix $b\subset[d]$ with $|b|=m-1$ and $i\notin b$.
Under the original procedure,
\[
\mathbb P(B=b,I=i)
=\mathbb P(S=b\cup\{i\})\cdot \mathbb P(I=i\mid S=b\cup\{i\})
=\frac{1}{\binom{d}{m}}\cdot \frac{1}{m}.
\]
Under the alternative procedure,
\[
\mathbb P(B=b,I=i)
=\mathbb P(B=b)\cdot \mathbb P(I=i|B=b)=\frac{1}{\binom{d}{m-1}}\cdot \frac{1}{d-m+1}=\frac{1}{\binom{d}{m}}\cdot \frac{1}{m}.
\]
Hence, both sampling procedures have the same joint law.
\end{proof}

Fix any $B$ and define the alternative instance $\theta^{(B)}$ by
$\theta^{(B)}_j=\Delta$ if $j\in B$ and $0$ otherwise.
For each $i\notin B$, define
\[
N_i:=\sum_{t=1}^T x_t(i)\qquad\text{and}\qquad A_i:=\{i\in \hat S\}.
\]
Because every played action has exactly $m$ ones,
\[
\sum_{i=1}^d N_i=mT \quad\Rightarrow\quad \sum_{i\notin B}\mathbb E_{\theta^{(B)}}[N_i]\le mT.
\]
Because $|\hat S|=m$ always,
\[
\sum_{i\notin B}\mathbb P_{\theta^{(B)}}(A_i)\le m.
\]
Let $n:=d-m+1=|[d]\setminus B|$. Due to the above two inequalities, there exists a set
$G_B\subseteq [d]\setminus B$ with $|G_B|\ge n/2$ such that for all $i\in G_B$,
\begin{equation}\label{eq:good-bounds}
\mathbb E_{\theta^{(B)}}[N_i]\le \frac{4mT}{n},
\qquad
\mathbb P_{\theta^{(B)}}(A_i)\le \frac{4m}{n}.
\end{equation}

Fix $i\in G_B$. Compare $\theta^{(B)}$ with $\theta^{(B\cup\{i\})}$.
They differ only in coordinate $i$ by $\Delta$.

Let $P_B$ and $P_{B\cup\{i\}}$ be the probability law under
$\theta^{(B)}$ and $\theta^{(B\cup\{i\})}$, respectively.
For Gaussian noise with variance $1$, the one-step KL between $\mathcal N(\mu,1)$ and $\mathcal N(\mu+\Delta,1)$
equals $\Delta^2/2$.
Using the adaptive KL chain rule (Lemma \ref{kl-chain-rule}) in the same way as we did the multi-task MAB, we get
\begin{equation}\label{eq:kl}
\mathrm{KL}(P_B\|P_{B\cup\{i\}})
=\frac{\Delta^2}{2}\,\mathbb E_{\theta^{(B)}}[N_i]
\le \frac{\Delta^2}{2}\cdot \frac{4mT}{n}
=2\Delta^2\frac{mT}{n},
\end{equation}
where we used \eqref{eq:good-bounds}.
Pinsker's inequality gives, for the event $A_i$,
\[
\mathbb P_{\theta^{(B\cup\{i\})}}(A_i)
\le \mathbb P_{\theta^{(B)}}(A_i) + \sqrt{\tfrac12\,\mathrm{KL}(P_B\|P_{B\cup\{i\}})}.
\]
Combining with \eqref{eq:good-bounds} and \eqref{eq:kl},
\begin{equation}\label{eq:pinsker-bound}
\mathbb P_{\theta^{(B\cup\{i\})}}(i\in \hat S)
\le \frac{4m}{n} + \Delta\sqrt{\frac{mT}{n}}.
\end{equation}
Assume $n\ge 20m$ so that $4m/n\le 0.2$.
Also assume  $T \le \frac{1}{2500}\cdot \frac{mn}{\varepsilon^2}.$ Since $\Delta=10\varepsilon/m$, we have
\[
\Delta\sqrt{\frac{mT}{n}}
=\frac{10\varepsilon}{m}\sqrt{\frac{mT}{n}}
\le \frac{10\varepsilon}{m}\sqrt{\frac{m}{n}\cdot \frac{1}{2500}\frac{mn}{\varepsilon^2}}
=0.2.
\]
Plugging into \eqref{eq:pinsker-bound} yields, for every $i\in G_B$,
\begin{equation}\label{eq:good-i-prob}
\mathbb P_{\theta^{(B\cup\{i\})}}(i\in \hat S)\le 0.4.
\end{equation}

Due to Lemma \ref{appendix-lem-equivalent-sampling}, conditional on $B$, $I$ is uniform on $[d]\setminus B$, and if $I=i$ then the instance is
$\theta^{(B\cup\{i\})}$. Therefore, we have:
\begin{equation}\label{eq:avg}
\mathbb P(I\in \hat S\mid B)
=\frac{1}{n}\sum_{i\in[d]\setminus B}\mathbb P_{\theta^{(B\cup\{i\})}}(i\in \hat S).
\end{equation}
Split the sum into $G_B$ and its complement. As $|G_B|\ge n/2$,
applying \eqref{eq:good-i-prob} for each $i\in G_B$, and using the trivial upper bound of $1$ for $i\notin G_B$, we get:
\[
\mathbb P(I\in \hat S\mid B)
\le \frac{1}{n}\Big(\frac{n}{2}\cdot 0.4 + \frac{n}{2}\cdot 1\Big)=0.7.
\]
Averaging over $B$ gives
\begin{equation}\label{eq:PI}
\mathbb P(I\in \hat S)\le 0.7.
\end{equation}

From \eqref{eq:value-vs-I} and \eqref{eq:PI}, we get:
\[
\mathbb E[\langle \hat x,\theta^{(S)}\rangle]\le 10\varepsilon\cdot 0.7=7\varepsilon.
\]
On the other hand, success implies $\langle \hat x,\theta^{(S)}\rangle\ge 9\varepsilon$, and the value is always
nonnegative, hence
\[
\mathbb E[\langle \hat x,\theta^{(S)}\rangle]\ge 9\varepsilon\cdot \mathbb P(\text{success}).
\]
Therefore $\mathbb P(\text{success})\le 7/9$, so $\mathbb P(\text{fail})\ge 2/9$ under uniform distribution over the hard instances.
By Yao's principle, there exists a fixed instance $\theta$ on which the algorithm fails with probability at least $2/9$.
This proves the theorem with a suitable universal constant $c$ (for instance $c=1/2500$).
\end{proof}
\section{Unit Ball Lower Bound}\label{appendix-ball-lower-bound}
Fix $\varepsilon\in (0,1]$, $\delta\in(0,1)$ and an $(\varepsilon,\delta)$-PAC algorithm. Let the algorithm run for $T:=\frac{H_1\log(1/\delta)+H_2}{\varepsilon^2}$ rounds and output $\hat x\in \mathcal{X}:=\mathbb{B}_d$.
Define the simple regret for any $\theta$ as
\[
R_T(\theta)\;:=\;\max_{x\in \mathcal{X}}\langle x,\theta\rangle-\langle \hat x,\theta\rangle \;\ge 0.
\]
Observe that $R_T(\theta)\leq 2\|\theta\|$.

As the algorithm is $(\varepsilon,\delta)$-PAC, we have $\mathbb P(R_T(\theta)>\varepsilon)\le \delta$. Hence, taking expectation, we have
\begin{equation}\label{eq:ball-simple-regret-upper}
    \mathbb{E}[R_T(\theta)]\leq \varepsilon+2\|\theta\|\cdot\delta 
\end{equation}

Consider a small enough $\varepsilon$ such that $T\geq d^2$.
Now due to \cite{chen2024assouad}, there exists a $\theta\in\mathbb{R}^d$ such that $\|\theta\|_2\leq c_0d/\sqrt{T}$ and $\mathbb E[R_T(\theta)]\geq c_1d/\sqrt{T}$ for some absolute constants $c_0,c_1\in (0,1)$. Let us consider now consider such a $\theta$. If we choose $\delta=c_1/4$, we have $\varepsilon\geq \frac{c_1d}{2\sqrt{T}}$ due to \eqref{eq:ball-simple-regret-upper}. As $H_1\leq H_2$ and $T=\frac{H_1\log(1/\delta)+H_2}{\varepsilon^2}$, we have $H_2\geq c_2 d^2$ for some absolute constant $c_2>0$.

Let us also consider the case where stopping time $\tau$ is randomized as it will be used in the next section. Due to \cite{chen2024assouad}, there exists a $\theta\in\mathbb{R}^d$ such that $\|\theta\|_2\leq c_0d/\sqrt{T}$ and $\mathbb E[R_T(\theta)]\geq c_1d/\sqrt{T}$ for some absolute constants $c_0,c_1\in (0,1)$. Consider one such $\theta$. Let assume that $\mathbb{E}[\tau]=T_0:=\frac{H_1\log(1/\delta)+H_2}{\varepsilon^2}$ with $0<H_1\leq H_2$. Let us consider $T=\frac{1}{\delta}\cdot T_0 $. If the algorithm terminates before $T$ rounds, we pad the remaining rounds using dummy samples. Now we have the following due to Markov's inequality:
\[
   \mathbb E[R_T(\theta)]\leq \varepsilon +2 \|\theta\|( \mathbb{P}(\tau>T)+\mathbb{P}(R_T(\theta)>\varepsilon,\tau\leq T))\leq \varepsilon +4\delta\cdot \|\theta\|
\]
Hence, for small enough absolute constant $\delta$, we have $H_2\geq c_3\cdot d^2$ for an absolute constant $c_3$.



\section{Polynomial Separation Instance's Non-Adaptive Lower Bound}\label{appendix-poly-sep-lower-bound}
Recall $\mathcal{X}=\bigcup_{i=1}^k \mathcal{X}_i\subseteq \mathbb{R}^{kd}$, where for each $i\in[k]$ we define the \emph{$i$-th block of coordinates} as $B_i:=\{(i-1)d+1,\ldots,id\}$ and
\[
\mathcal{X}_i:=\Bigl\{x\in\mathbb{R}^{kd}:\ \mathrm{supp}(x)\subseteq B_i,\ \|x\|_2\le 1\Bigr\}.
\]
Consider any (possibly randomized) non-adaptive $(\varepsilon,\delta)$-PAC algorithm that uses $T=\frac{H_1\log(1/\delta)+H_2}{\varepsilon^2}$ samples with $0<H_1\le H_2$, and let $\tau_i:=\sum_{t=1}^T \mathbbm{1}\{X_t\in\mathcal{X}_i\}$ denote the (random) number of samples allocated to block $i$, so that $\sum_{i=1}^k \tau_i=T$ and hence there exists $i^\star\in[k]$ with $T_{i^\star}:=\mathbb{E}[\tau_{i^\star}]\le T/k$. Fix such an $i^\star$ and define a hard instance $\theta\in\mathbb{R}^{kd}$ supported only on block $B_{i^\star}$ as follows: set $\theta_j=0$ for all $j\notin B_{i^\star}$, and on the coordinates in $B_{i^\star}$ set $(\theta_j)_{j\in B_{i^\star}}=\vartheta$ where $\vartheta\in\mathbb{R}^d$ is chosen according to the unit-ball lower bound from the previous section (\cite{chen2024assouad} together with the padding and Markov's inequality argument), so that for some absolute constants $c_0,c_1\in(0,1)$ one has $\|\vartheta\|_2\le c_0\cdot \sqrt{\delta}\cdot d/\sqrt{T_{i^\star}}$ and any algorithm that receives $T_{i^\star}/\delta$ informative samples on this block satisfies 
\[
c_1\cdot \sqrt{\delta}\cdot d/\sqrt{T_{i^\star}} \le\mathbb{E}[R_{T_{i^\star}/\delta}(\vartheta)]\le \varepsilon +4\delta\cdot\|\vartheta\| 
\]
For our chosen $\theta$, any action $x\in\mathcal{X}_j$ with $j\neq i^\star$ satisfies $\langle x,\theta\rangle=0$, so samples taken outside block $i^\star$ are uninformative; moreover, if the algorithm outputs $\hat x\in\mathcal{X}_j$ with $j\neq i^\star$, then $\langle \hat x,\theta\rangle=0$, which is the same value as outputting the zero vector on block $i^\star$. Therefore the $(\varepsilon,\delta)$-PAC requirement for sufficiently small absolute constant $\delta$ forces $H_2\ge c_3\cdot kd^2$ for an absolute constant $c_3>0$.

$H_1\geq \Omega(kd)$ follows directly from Theorem~\ref{thm:adaptive-lower-bound}.


\end{document}